%% file: main.tex
\documentclass{elsarticle} 
\usepackage[outline]{contour} 
\contourlength{0.4pt}
\makeatletter
	\def\ps@pprintTitle{%
 	\let\@oddhead\@empty
	\let\@evenhead\@empty
	\def\@oddfoot{\centerline{\thepage}}%
	\let\@evenfoot\@oddfoot}
\makeatother
\usepackage[dvipsnames]{xcolor}
\usepackage[english]{babel}
\usepackage{amsfonts}
\usepackage[T1]{fontenc}
\usepackage{algorithm}
\usepackage{algpseudocode} 

\usepackage{amsthm}
\usepackage{tikz}
\usetikzlibrary{calc}
\usepackage{amsmath}
\usepackage[margin=1in]{geometry}
\usetikzlibrary{arrows.meta, positioning, shapes.geometric}
\usepackage[utf8]{inputenc}
\usepackage{pgfplots} 
\usepackage{pgfgantt}
\usepackage{pdflscape}
\usepackage{array,makecell}
\pgfplotsset{compat=newest} 
\pgfplotsset{plot coordinates/math parser=false} 
\newlength\fwidth
\newlength\fheight
\usepackage{amsmath}
\usepackage{amsthm}
\usepackage{graphicx}
\usepackage{subcaption}
\usepackage{overpic}
\usepackage{siunitx}
\usepackage{mathtools}

\newtheorem{proposition}{Proposition}
\usepackage[textsize=tiny]{todonotes}

\usepackage[colorlinks=true, allcolors=blue]{hyperref}

\theoremstyle{definition}
\newtheorem{remark}{Remark}
\usepackage{booktabs}

\DeclareMathOperator{\sech}{sech}
\newcommand{\real}{\mathbb{R}}

\newcommand{\bvarphi}{ \boldsymbol \varphi}
\newcommand{\bvarphihat}{\widehat{ \bvarphi}}
\newcommand{\bphihat}{\widehat{ \bphi}}

\newcommand{\dt}{\Delta t}
\newcommand{\y}{\mathbf y}
\newcommand{\dd}{\rm d}
\newcommand{\f}{\mathbf{f}}

\newcommand{\q}{\mathbf q}
\newcommand{\w}{\mathbf w}
\newcommand{\qdot}{\dot{\q}}

\newcommand{\pdot}{\dot{\p}}

\newcommand{\ydot}{\dot{\y}}

\newcommand{\ybar}{\bar{\y}}
\newcommand{\nbar}{\bar{n}}
\newcommand{\what}{\widehat{\w}}
\newcommand{\What}{\widehat{\W}}
\newcommand{\phat}{\widehat{\p}}

\newcommand{\p}{\mathbf p}
\newcommand{\x}{\mathbf x}

\newcommand{\bphi}{\boldsymbol \phi}

\newcommand{\bPhi}{\mathbf{\Phi}}

\newcommand{\V}{\mathbf V}
\newcommand{\F}{\mathbf F}
\newcommand{\Fhat}{\widehat \F}
\newcommand{\fn}{\f_{\text{non}}}
\newcommand{\Y}{\mathbf Y}

\newcommand{\Hhat}{\widehat{\mathbf H}}
\newcommand{\Yhat}{\widehat{\mathbf Y}}

\newcommand{\A}{\mathbf A}

\newcommand{\D}{\mathbf D}
\newcommand{\bH}{\mathbf H}
\newcommand{\Dhat}{\widehat{\D}}

\newcommand{\B}{\mathbf B}
\newcommand{\Q}{\mathbf Q}
\newcommand{\W}{\mathbf W}
\newcommand{\qhat}{\widehat{\q}}

\newcommand{\qhatdot}{\mathbf{\dot{\widehat q}}}
\newcommand{\phatdot}{\mathbf{\dot{\widehat p}}}

\newcommand{\Qhat}{\mathbf{\widehat Q}}
\newcommand{\Phat}{\mathbf{\widehat P}}
\newcommand{\Qhatdot}{\mathbf{\dot{\widehat Q}}}

\newcommand{\Phatdot}{\mathbf{\dot{\widehat P}}}

\newcommand{\Yhatdot}{\mathbf{\dot{\widehat Y}}}

\newcommand{\Pp}{\mathbf P}

\newcommand{\In}{\mathbf I_n}

\begin{document}
\begin{frontmatter}
\title{Structure-preserving Lift \& Learn: Scientific machine learning for nonlinear conservative partial differential equations }

 		\author[affil1]{Harsh Sharma\corref{cor1}}
		 		\cortext[cor1]{Corresponding author}
		\ead{hasharma@ucsd.edu}
 		\author[affil1]{Juan Diego Draxl Giannoni}
 		\author[affil1]{Boris Kramer}

			\address[affil1]{Department of Mechanical and Aerospace Engineering, University of California San Diego, CA, United States}

\begin{abstract}
This work presents \textit{structure-preserving Lift \& Learn}, a scientific machine learning method that employs lifting variable transformations to learn structure-preserving reduced-order models for nonlinear partial differential equations (PDEs) with conservation laws. We propose a hybrid learning approach based on a recently developed energy-quadratization strategy that uses knowledge of the nonlinearity at the PDE level to derive an equivalent quadratic lifted system with quadratic system energy. The lifted dynamics obtained via energy quadratization are linear in the old variables, making model learning very effective in the lifted setting. Based on the lifted quadratic PDE model form, the proposed method derives quadratic reduced terms analytically and then uses those derived terms to formulate a constrained optimization problem to learn the remaining linear reduced operators in a structure-preserving way. The proposed hybrid learning approach yields computationally efficient quadratic reduced-order models that respect the underlying physics of the high-dimensional problem. We demonstrate the generalizability of quadratic models learned via the proposed structure-preserving Lift \& Learn method through three numerical examples: the one-dimensional wave equation with exponential nonlinearity, the two-dimensional sine-Gordon equation,  and the two-dimensional Klein-Gordon-Zakharov equations. The numerical results show that the proposed learning approach is competitive with the state-of-the-art structure-preserving data-driven model reduction method in terms of both accuracy and computational efficiency.
 \end{abstract}
\end{frontmatter}
\section{Introduction}
\label{sec:introduction}
Model reduction methods seek to derive reduced-order models (ROMs) that provide accurate approximations of the  full-order model (FOM) solutions. The low-dimensional nature of ROMs  enables many-query tasks such as optimization, design, uncertainty quantification, and control, where models need to be evaluated many times for different settings. For applications where the high-dimensional problem possesses additional qualitative features like conservations laws, it is desirable to construct a ROM that respects the underlying physics of the problem. The field of structure-preserving model reduction has developed such physics-preserving ROMs for a wide class of applications ranging from computational physics~\cite{klein2024energy,tyranowski2023symplectic,hesthaven2024adaptive} to structural mechanics~\cite{carlberg2015preserving,sharma2024lagrangian} to soft robotics~\cite{adibnazari2023full,lepri2023neural,sharma2024data}. 

In this work, we focus on structure-preserving nonintrusive model reduction for conservative PDEs of the form
\begin{equation}
\frac{\partial^2 \phi (\x,t)}{\partial t^2}=\mathbf \Delta \phi (\x,t) - f_{\text{non}}(\phi (\x,t)), 
\label{eq:pde}
\end{equation}
where $\x=(x_1,x_2,\cdots,x_d) \in \Omega$ is the spatial variable, $t$ is time, $\mathbf \Delta$ is the Laplacian operator in $\real^d$, $\phi (\x,t)$ is the scalar state field,  and the nonlinear component of the vector field, $f_{\text{non}}(\phi (\x,t)):=\nabla_{\phi}(g(\phi(\x,t)))$, is derived from a smooth nonlinear function $g(\phi(\x,t))$. A key property of nonlinear wave equations of the form~\eqref{eq:pde} is that the total energy
\begin{equation}
\mathcal E[\phi(\x,t)]:=\int_{\Omega} \left( \frac{1}{2}\left( \frac{\partial \phi(\x,t)}{\partial t} \right)^2 + \frac{1}{2}\left(\nabla \phi(\x,t) \right)^2  + g(\phi(\x,t)) \right) \dd \x,
\label{eq:epde}
\end{equation}
is a conserved quantity provided the boundary conditions enforce zero normal energy flux through the boundary, i.e., $\partial_n \phi(\x,t)\,\frac{\partial \phi(\x,t)}{\partial t} = 0$ on $\partial\Omega$. This includes common conservative settings such as periodic boundary conditions, homogeneous Neumann boundary conditions, and unbounded domains where $\phi(\x,t)$ and its gradients decay sufficiently fast at infinity.
The integrand in~\eqref{eq:epde} typically has the physical interpretation of energy density with the first term in the integrand representing the kinetic energy and the other two terms representing the potential energy. 

Structure-preserving model reduction for conservative PDEs was first explored from the Hamiltonian viewpoint in~\cite{peng2016symplectic} where the authors derived Hamiltonian ROMs by projecting the Hamiltonian FOM operators onto a symplectic subspace obtained via proper symplectic decomposition (PSD). In a similar direction, the work in~\cite{afkham2017structure} developed reduced basis methods for structure-preserving model reduction of parametric Hamiltonian systems. The authors in~\cite{gong2017structure} presented a modified projection technique based on the proper orthogonal decomposition (POD) basis such that the Hamiltonian structure is preserved after the Galerkin projection step. Building on these methods,  a variety of structure-preserving model reduction methods have been developed recently~\cite{hesthaven2021structureb, buchfink2023symplectic, sharma2023symplectic, gruber2024variationally}. All of these methods, however, are \textit{intrusive}, which means that the projection step requires access to the high-dimensional FOM operators. Moreoever, all of the aforementioned structure-preserving approaches suffer from computational efficiency issues as the evaluation of the nonlinear components of the structure-preserving ROM vector field still scales with the FOM dimension. To tackle this computational bottleneck, the authors in~\cite{wang2021structure} developed a structure-preserving variant of the discrete empirical interpolation method (DEIM) for Hamiltonian systems. The authors in~\cite{pagliantini2023gradient} improved on this earlier work and presented a gradient-preserving hyper-reduction approach with theoretical guarantees for the preservation of the FOM Hamiltonian. This approach, however, requires a computationally demanding offline phase that constructs the DEIM basis from the Jacobian snapshot data matrix.

In another research direction, projection-based model reduction via lifting~\cite{gu2011qlmor, benner2015two, kramer2019nonlinear,benner2018mathcal,kramer2019balanced} has emerged as a promising approach for nonlinear dynamical systems. By quadratizing the nonlinear dynamics before projection, these methods eliminate the need for the additional hyper-reduction step.  In general, however, quadratic ROMs derived via lifting are not guaranteed to respect the underlying physics of the high-dimensional problem. The authors in~\cite{sharma2025nonlinear} presented an energy-quadratization strategy for structure-preserving model reduction for conservative PDEs of the form~\eqref{eq:pde}. Even though the energy-quadratization strategy yields a computationally efficient structure-preserving ROM without any additional hyper-reduction, the Galerkin projection step is intrusive in the sense that it requires access to the linear and quadratic FOM operators in the lifted setting. For many practical applications, the lifted FOM operators are not available as this would require implementing a discretization for a new PDE, which motivates the need for a \textit{nonintrusive} approach for learning structure-preserving quadratic ROMs directly from data. 

The Operator Inference method for nonintrusive model reduction of FOMs with linear and/or low-order polynomial terms was introduced in~\cite{peherstorfer2016data,kramer2024learning}. This nonintrusive approach has been extended to nonlinear FOMs with nonpolynomial systems in~\cite{benner2020operator} where knowledge of nonpolynomial nonlinear terms at the PDE level is leveraged to learn the remaining reduced operators via Operator Inference. Building on this gray-box approach, the authors in~\cite{sharma2022hamiltonian} developed Hamiltonian Operator Inference for learning physics-preserving ROMs of canonical Hamiltonian systems. Since then a variety of structure-preserving Operator Inference methods have been developed for different classes of conservative FOMs~\cite{gruber2023canonical,sharma2024preserving,filanova2023operator,vijaywargiya2025tensor}. However, all of these structure-preserving Operator Inference approaches, similarly to their intrusive counterparts, suffer from computational efficiency issues for nonlinear systems and need an additional structure-preserving hyper-reduction step to reduce the computational cost in the online stage.

The authors in~\cite{qian2020lift} integrated lifting transformations with the Operator Inference framework to develop a method called Lift \& Learn for nonintrusively learning quadratic ROMs of general nonlinear FOMs. This method circumvents the need for hyper-reduction by learning quadratic ROMs from projections of the FOM snapshot data in the lifted setting. Even though this machine learning approach learns computationally efficient quadratic ROMs from data, the learned ROMs are not guaranteed to respect the underlying physics of the problem (see Section~\ref{sec:background}). The main goal of this work is to develop a structure-preserving Lift \& Learn approach to derive structure-preserving quadratic ROMs of nonlinear conservative PDEs. The main contributions of this work are:
\begin{enumerate}
\item We present a nonintrusive model reduction method--a special class of scientific machine learning--to derive structure-preserving quadratic ROMs from data generated by the high-dimensional nonlinear models of conservative PDEs. The proposed approach (see Figure~\ref{fig:schematic} for a schematic overview) leverages knowledge of the non-polynomial nonlinearity at the conservative PDE level to ensure that the learned quadratic ROMs respect the underlying physics of the problem. 
\item We present a theoretical result that shows that the structure-preserving Lift \& Learn ROMs conserve a perturbed lifted FOM energy exactly.
\item We learn structure-preserving quadratic ROMs for three nonlinear conservative PDEs. The numerical results demonstrate the learned ROMs' ability to provide accurate and stable predictions outside the training data regime while also achieving computational efficiency similar to the state-of-the-art nonintrusive Hamiltonian ROMs with structure-preserving hyper-reduction.
\end{enumerate}
The paper is structured as follows. Section~\ref{sec:background} summarizes the standard Lift \& Learn approach for learning quadratic ROMs of general nonlinear systems via lifting transformations.  Section~\ref{sec:method} introduces structure-preserving Lift \& Learn, a nonintrusive model reduction approach for learning quadratic ROMs that respect the underlying physics of the problem. Section~\ref{sec:numerical} demonstrates the proposed nonintrusive approach on three nonlinear conservative PDEs with increasing complexity: one-dimensional nonlinear wave equation with exponential nonlinearity, two-dimensional sine-Gordon equation, and  two-dimensional Klein-Gordon-Zakharov equation. Finally, Section~\ref{sec:conclusion} provides concluding remarks and suggests future research directions.

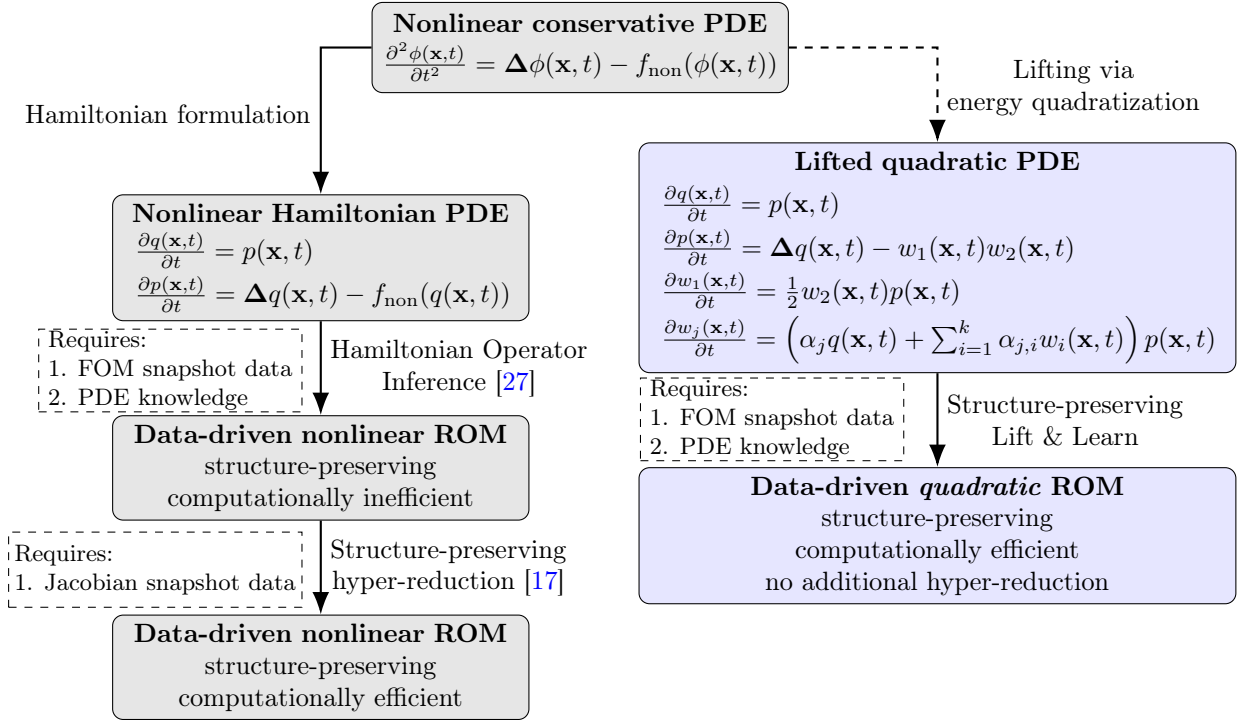
\begin{figure}
\centering
\begin{tikzpicture}[
  node distance=1.25cm and 3.2cm,
  box/.style={rectangle, draw=black!70!black, fill=gray!20!white, text width=5.3cm, align=center, rounded corners, minimum height=1.2cm},
  boxRight/.style={rectangle, draw=black!70!black, fill=blue!10!white, text width=7.7cm, align=center, rounded corners, minimum height=1.2cm},
  arrow/.style={thick, -{Latex[length=3mm]}},
  label/.style={font=\small\itshape}
]

\node[box] (nonlinearFOM) {
    \textbf{Nonlinear Hamiltonian PDE}\\ \vspace{0.1cm}
    \(
\begin{array}{l}\vspace{0.1cm}
\frac{\partial q (\x,t)}{\partial t}=p(\x,t) \\
\frac{\partial p (\x,t)}{\partial t}=\mathbf \Delta q (\x,t) - f_{\text{non}}(q (\x,t))
\end{array}
\)
    
};

\node[boxRight, right=of nonlinearFOM, xshift=-1.75cm] (quadraticFOM) {
    \textbf{Lifted quadratic PDE}\\ \vspace{0.1cm}
    \(
\begin{array}{l}\vspace{0.1cm}
 \frac{\partial q (\x,t)}{\partial t}=p(\x,t) \\ \vspace{0.1cm}
 \frac{\partial p (\x,t)}{\partial t}=\mathbf \Delta q (\x,t) - w_1(\x,t) w_2(\x,t) \\\vspace{0.1cm}
\frac{\partial w_1 (\x,t)}{\partial t}=\frac{1}{2}w_2(\x,t) p(\x,t) \\
  \frac{\partial w_j (\x,t)}{\partial t}=\left(\alpha_jq(\x,t) +\sum_{i=1}^{k}\alpha_{j,i}w_i(\x,t) \right) p(\x,t)
\end{array}
\)
};

\path (nonlinearFOM.north west) -- (quadraticFOM.north east) coordinate[pos=0.5] (midtop);
\node[box, anchor=south] (pde) at ($(midtop)+(-1.25,1)$) {
    \textbf{Nonlinear conservative PDE}\\ \vspace{0.1cm}
   
   $\frac{\partial^2 \phi (\x,t)}{\partial t^2}=\mathbf \Delta \phi (\x,t) - f_{\text{non}}(\phi (\x,t))$
    
};
\node[box, below=of nonlinearFOM] (nonlinearROM1) {
    \textbf{Data-driven nonlinear ROM}\\
    structure-preserving\\
    computationally inefficient
};

\node[boxRight, below=of quadraticFOM] (quadraticROM) {
    \textbf{Data-driven \textit{quadratic} ROM}\\
    structure-preserving\\
    computationally efficient\\
    no additional hyper-reduction
};

\node[box, below=of nonlinearROM1] (nonlinearROM2) {
    \textbf{Data-driven nonlinear ROM}\\
    structure-preserving\\
    computationally efficient
};
\draw[arrow] 
  (pde.west) -| ([yshift=0.2cm]nonlinearFOM.north) node[midway, left, yshift=-25pt, align=center] {Hamiltonian formulation} -- (nonlinearFOM.north);
  \draw[arrow,dashed] 
  (pde.east) -| ([yshift=0.2cm]quadraticFOM.north) node[midway, right, yshift=-15pt, align=center] {Lifting via \\
  energy quadratization} -- (quadraticFOM.north)   ;
             \node[draw,dashed, minimum width=3.5cm, minimum height=1.1cm, align=center] at (-2, -1.5) {};
         \node  [align=left,font=\small] at (-2,-1.5) {Requires: \\ 1. FOM snapshot data \\ 2. PDE knowledge};
                  \node[draw,dashed, minimum width=3.5cm, minimum height=1.1cm, align=center] at (6,-2.15) {};
              \node  [align=left,font=\small] at (6,-2.15) {Requires: \\ 1. FOM snapshot data \\ 2. PDE knowledge};
                                \node[draw,dashed, minimum width=3.9cm, minimum height=1cm, align=center] at (-2.2,-4.15) {};
            \node  [align=left,font=\small] at (-2.2,-4.15) {Requires: \\ 1. Jacobian snapshot data };
    \draw[arrow] (nonlinearFOM.south) -- (nonlinearROM1.north)
    node[midway,right, align=center]{Hamiltonian Operator \\ Inference~\cite{sharma2022hamiltonian} }; 
    
        \draw[arrow] (nonlinearROM1.south) -- (nonlinearROM2.north)
    node[midway, right, align=center]{Structure-preserving \\ hyper-reduction~\cite{pagliantini2023gradient}}; 
        \draw[arrow] (quadraticFOM.south) -- (quadraticROM.north)
    node[midway, right, align=center] {Structure-preserving \\Lift \& Learn};

\end{tikzpicture}

\caption{The proposed structure-preserving Lift \& Learn approach (right path) learns computationally efficient quadratic ROMs without any additional hyper-reduction step whereas the Hamiltonian Operator Inference approach (left path) requires additional Jacobian snapshot data for structure-preserving hyper-reduction. The dashed line style for the arrow from the nonlinear conservative PDE to the lifted quadratic FOM indicates that we only derive the symbolic form of the lifted quadratic FOM, and constructing the lifted FOM operators is \textit{not} required.}
\label{fig:schematic}
\end{figure}
\section{Background}
\label{sec:background}
In Section~\ref{sec:std_ll} we review Lift \& Learn~\cite{qian2020lift}, a model reduction method for deriving quadratic ROMs of general nonlinear FOMs nonintrusively from data. In Section~\ref{sec:motivation} we motivate the need for structure-preserving Lift \& Learn by demonstrating how the standard Lift \& Learn approach violates the conservative nature of the original high-dimensional problem. 
\subsection{Standard Lift \& Learn}
\label{sec:std_ll}
Consider a nonlinear PDE
\begin{equation}
\frac{\partial y(\x,t)}{\partial t}=f(y(\x,t)), \qquad y(\x,0)=y_0(\x),
\label{eq:gen_non_pde}
\end{equation}
where $f(y(\x,t))$ is a differentiable nonlinear function that maps the $d_s$-dimensional vector consisting of field variables $y(\x,t)$ to its time derivative and $y_0(\x)$ is the initial condition. The corresponding space-discretized model is a high-dimensional system of nonlinear ODEs
\begin{equation}
\dot{\y}(t)=\f(\y(t)), \qquad \y(0)=\y_0,
\label{eq:gen_non}
\end{equation}
where $\y(t) \in \real^{n\cdot d_s}$ is the spatially disretized state vector, $\y_0$ is the space-discretized initial condition, and $\f(\y(t))$ is an $n\cdot d_s$-dimensional vector of real-valued nonlinear functions. We summarize the three steps of the standard Lift \& Learn~\cite{qian2020lift} next, and summarize our notation in Table~\ref{tablex}.

\begin{enumerate}
\item \textbf{Expose quadratic structure via lifting transformations:}  Given a nonlinear PDE of the form~\eqref{eq:gen_non_pde}, exploit knowledge of the functional form of the nonlinearity in $f(y(\x,t))$ to identify a lifting transformation which yields an augmented vector of field variables in which the system dynamics have quadratic structure. The key idea of lifting is to find a quadratization $w \in \real^{d_w}$ of the nonlinear PDE~\eqref{eq:gen_non_pde} in terms of $d_w$ additional variables $w_i=\tau_i(y)$ for $i=1,\cdots, d_w$ and  then employ a transformation $\tau_{\text{lift}}: \real^{d_s} \to \real^{d_a}$ with $d_a=d_s+d_w$ that transforms the original nonlinear PDE in field variables $y \in \real^{d_s}$ into an equivalent lifted PDE in field variables $\bar{y}=[y^\top,w^\top]^\top \in  \real^{d_a}$ with quadratic dynamics, i.e., 
\begin{equation} 
\frac{\partial y(\x,t)}{\partial t}=f(y(\x,t)) \qquad \xrightarrow{\bar{y}=\tau_{\text{lift}}(y)} \qquad \frac{\partial \bar{y}(\x,t)}{\partial t}=\bar{a}(\bar{y}(\x,t)) + \bar{b}(\bar{y}(\x,t)),
\label{eq:lifted_pde}
\end{equation}
where $\bar{a}$ and $\bar{b}$ consist of linear functions $\bar{a}_j$ and quadratic functions $\bar{b}_j$, respectively, for $j=1,\cdots, d_a$. 
\item \textbf{Construct lifted training data:} At the space-discretized level, the lifting transformation transforms the nonlinear FOM~\eqref{eq:gen_non} into an equivalent lifted quadratic FOM, i.e.,
\begin{equation} 
\ydot(t)=\f(\y(t)) \qquad \xrightarrow{\ybar=\tau_{\text{lift}}(\y)} \qquad \dot{\ybar}(t)=\bar{\A}\ybar(t) + \bar{\B}(\ybar(t) \otimes \ybar(t)),
\label{eq:lifted_FOM}
\end{equation}
where $\ybar(t)\in \real^{\nbar}$ with $\nbar=n \cdot d_a$ is the  space-discretized state vector in the lifted setting, the notation $\otimes$ denotes the Kronecker product of vectors, and $\bar{\A} \in \real^{\nbar \times\nbar}$ and $\bar{\B}\in \real^{\nbar \times\nbar^2}$ are the linear and quadratic FOM operators in the lifted setting, respectively. To obtain the data in the lifted variables, the first step is to simulate the original non-lifted FOM~\eqref{eq:gen_non} for $K$ time steps to construct  the high-dimensional snapshot data $\Y\in \real^{n\cdot d_s \times K}$. Then, the lifting map is applied to this snapshot data matrix to obtain the lifted snapshot data matrix $\bar{\Y} \in \real^{\bar{n}\times K}$.  For projection-based model reduction, the lifted state is approximated as $\ybar(t) \approx \bar{\V}\ybar_r(t)$ with a basis matrix $\bar{\V} \in \real^{\nbar \times \bar{r}}$ and then the reduced snapshot data matrix $\Yhat_r \in \real^{\bar {r} \times K}$ is obtained by projecting $\bar{\Y}$ onto the subspace spanned by $\bar{\V}$. Finally, the reduced time-derivative data matrix $ \Yhatdot_r \in \real^{\bar{r} \times K}$ is constructed from $\Yhat_r$ using a finite difference scheme.
\item \textbf{Solve the least-squares problem via Operator Inference:}  The postulated quadratic form for the nonintrusive ROM based on the lifted FOM in~\eqref{eq:lifted_FOM} is given by
\begin{equation}
\dot{\ybar}_r(t)=\bar{\A}_r\ybar_r(t) + \bar{\B}_r(\ybar_r(t) \otimes \ybar_r(t)),
\label{eq:std_lifting}
\end{equation}
where the reduced operators $\bar{\A}_r \in \real^{\bar{r} \times\bar{r}}$ and $\bar{\B}_r\in \real^{\bar{r} \times\bar{r}^2}$  in the lifted setting are learned from data. Based on~\eqref{eq:std_lifting}, Lift \& Learn formulates the following minimization problem:
\begin{equation}
\min_{ \bar{\A}_r\in \real^{\bar{r} \times\bar{r}}, \bar{\B}_r\in \real^{\bar{r} \times\bar{r}^2}} \lVert \Yhatdot_r - \bar{\A}_r\Yhat_r - \bar{\B}_r(\Yhat_r\otimes \Yhat_r) \rVert_F.
\end{equation}
\end{enumerate}
The ability to learn the reduced operators in~\eqref{eq:std_lifting} without assuming access to the lifted FOM operators in~\eqref{eq:lifted_FOM} is a key advantage of the Lift \& Learn approach. However, this learning approach based on exploiting the quadratic structure in the lifted setting is not guaranteed to be structure-preserving for nonlinear conservative PDEs. This is best understood with an example. 
\begin{table}[h!]

\centering
\caption{Summary of notation and dimensions for the original (non-lifted) and lifted systems introduced in Section~\ref{sec:std_ll}.}
\label{tablex}
\setlength{\tabcolsep}{6pt}
\renewcommand{\arraystretch}{1.4}

\begin{tabular}{c|c|c}
 & Original variables (non-lifted)  & Lifted variables  \\ \hline

PDE level  &
 $\displaystyle \partial y(\x,t)/\partial t = f(y(\x,t))$   &
$\displaystyle \partial \bar{y}(\x,t)/ \partial t
   = \bar{a}(\bar{y}(\x,t)) + \bar{b}(\bar{y}(\x,t))$     \\ 
(space-time continuous)   &with $d_s$-dimensional state vector $y(\x,t)$    &  with $d_a$-dimensional state vector $\bar{y}(\x,t)$  
  \\ \hline

FOM level  &
$\displaystyle \dot{\y}(t) = \f(\y(t))$  &
$\displaystyle \dot{\ybar}(t)
= \bar{\A}\ybar(t) + \bar{\B}(\ybar(t) \otimes \ybar(t))$  \\ 
 (space-discretized)  &with $\y \in \real^{n \cdot d_s}$   & with $\bar{\y} \in \real^{n \cdot d_a}$  
  \\ \hline

ROM level &
$\displaystyle \dot{\y}_r(t)= \f_r(\y_r(t))$  &
 $\displaystyle \dot{\ybar}_r(t)
= \bar{\A}_r\ybar_r(t)
+ \bar{\B}_r(\ybar_r(t) \otimes \ybar_r(t))$  \\ 
(reduced-order)  &  with $\y_r \in \real^{r}$  & with $\bar{\y}_r \in \real^{\bar{r}}$  
  \\ 

\end{tabular}
\end{table}
\subsection{Motivation: Lift \& Learn does not conserve additional properties}
\label{sec:motivation}
Consider the one-dimensional sine-Gordon equation
\begin{equation}
\frac{\partial^2 \phi(x,t)}{\partial t^2}=\frac{\partial^2 \phi(x,t)}{\partial x^2}-\sin(\phi(x,t)).
 \label{eq:sg_pde}
\end{equation}
The boundary conditions are periodic and the initial conditions are $(\phi(x,0),\frac{\partial \phi}{\partial t}(x,0))=(0, 4/\cosh (x))$.
We define $q(x,t)=\phi(x,t)$ and $p(x,t)=\frac{\partial \phi(x,t)}{\partial t}$ to rewrite the nonlinear conservative PDE~\eqref{eq:sg_pde} in first-order form. A structure-preserving spatial discretization of these first-order PDEs would lead to a $2n$-dimensional nonlinear conservative FOM of the form
\begin{equation} 
    \dot{\q}(t)=\p(t), \quad
    \dot{\p}(t)
    =\D\q(t) - \sin(\q(t)),
 \label{eq:sg_cons}
\end{equation}
which conserves the space-discretized nonlinear FOM energy $E(\q,\p,t)=\frac{1}{2} \p(t)^\top \p(t) - \frac{1}{2}\q(t)^\top\D \q(t) + \sum_{i=1}^n \left( 1-\cos(q_i(t)) \right)$. Note that the symmetric discretization matrix $\D=\D^\top \in \real^{n \times n} $, also known as the symmetric discrete Laplacian, contains all the information about the spatial discretization scheme. From hereon, we simplify the notation by omitting explicit dependence on time: the space-discretized state vectors $\q(t)$ and $\p(t)$ at time $t$ are therefore denoted as $\q$ and $\p$, respectively.  

Since the lifting map $\tau_{\text{lift}}$ in~\eqref{eq:lifted_pde} is generally non-unique, there are multiple ways of deriving a quadratization for~\eqref{eq:sg_pde}. We presented an energy-quadratization strategy in~\cite{sharma2025nonlinear} that, in combination with intrusive Galerkin projection, leads to a structure-preserving quadratic ROM. Based on this energy-quadratization strategy, we introduce two auxiliary variables $w_1=\sin(q/2)$ and $w_2=\cos(q/2)$ that quadratize the nonlinear energy  and transform~\eqref{eq:sg_cons} into a $4n$-dimensional quadratic lifted FOM
\begin{align}
\label{eq:sg_lifted}
    \dot{\q}&=\p, \nonumber\\
    \dot{\p}&=\D\q - 2\w_1\odot \w_2, \\
    \dot{\w}_1&=\frac{1}{2}\w_2 \odot \p,\nonumber\\
    \dot{\w}_2&= -\frac{1}{2} \w_1\odot \p, \nonumber
\end{align}
where the notation $\odot$ denotes the (Hadamard) component-wise product of two vectors. This quadratic lifted FOM possesses a quadratic invariant in the lifted variables $E_{\text{lift}}(\q,\p,\w_1,\w_2)=\frac{1}{2}\p^\top\p - \frac{1}{2}\q^\top \D\q + 2\w_1^\top\w_1$. We use a block-diagonal basis matrix $\bar{\V}  \in \real^{4n \times 4r}$ that preserves the coupling structure
\begin{equation*}
\bar{\V} =\text{blkdiag}(\bPhi,\bPhi,\V_{1},\V_{2})\in \real^{4n\times 4r},
\end{equation*}
where $\bPhi$ is the proper symplectic decomposition (PSD) basis matrix computed using the cotangent lift algorithm in~\cite{peng2016symplectic}, and $ \V_{1}$ and $\V_{2}$ are the POD basis matrices for $\w_1$ and $\w_2$, respectively.\footnote{We showed in~\cite{sharma2025nonlinear} that choosing the PSD basis matrix for $\q$ and $\p$ ensures that the quadratic ROM derived via intrusive projection of the lifted FOM~\eqref{eq:sg_lifted} conserves the lifted FOM energy.}
 
We build a training dataset by simulating the nonlinear FOM of dimension $2n=400$ from $t=0$ to $t=10$ with fixed time step $\Delta t=0.005$. We then compute a block-diagonal basis matrix for the lifted FOM state and then learn quadratic ROMs from projected data. We numerically integrate  the quadratic ROMs using the implicit midpoint method which is a fully implicit energy-conserving integrator for quadratic vector fields. Figure~\ref{fig:comp_state} shows the relative state error in the training data regime  for learned quadratic ROMs of different dimensions. We observe that the state approximation error over the training data decreases monotonically from $2r=2$ to $2r=20$. However, the energy error plots in Figure~\ref{fig:comp_energy} show that standard Lift \& Learn ROMs demonstrate unbounded energy error growth despite using an energy-conserving numerical integrator. This illustrates that preserving the quadratic structure alone in Lift \& Learn is not enough for nonlinear FOMs with conservation laws.
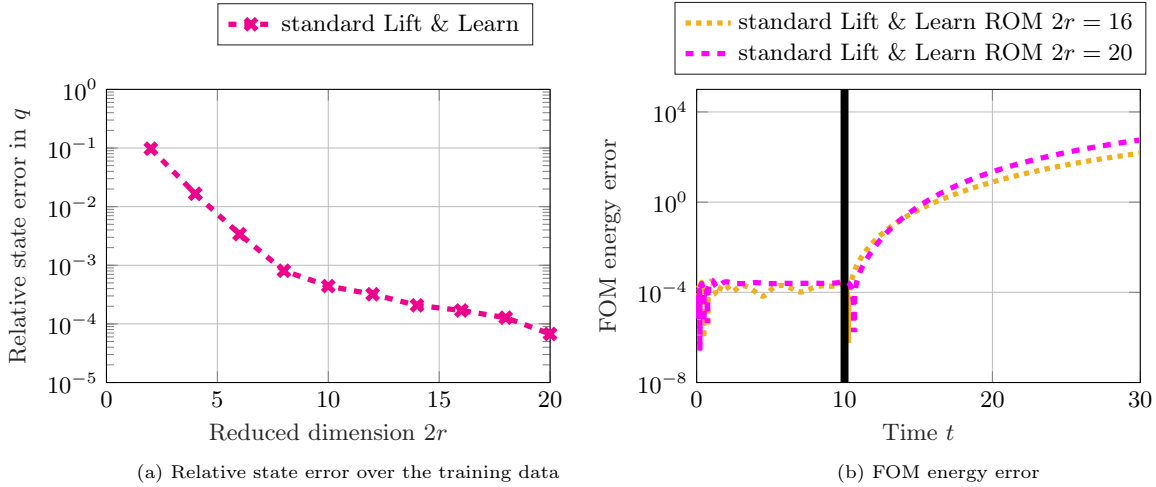
\begin{figure}[tbp]
\small
\captionsetup[subfigure]{oneside,margin={1.8cm,0 cm}}
\begin{subfigure}{.45\textwidth}
       \setlength\fheight{6 cm}
        \setlength\fwidth{\textwidth}
\input{figures/motivation/q_train.tex}
\caption{Relative state error over the training data }
\label{fig:comp_state}
    \end{subfigure}
    \hspace{0.2cm}
    \begin{subfigure}{.45\textwidth}
           \setlength\fheight{6 cm}
           \setlength\fwidth{\textwidth}
\raisebox{-5mm}{\input{figures/motivation/fom_energy.tex}}
\caption{FOM energy error}
\label{fig:comp_energy}
    \end{subfigure}
\caption{One-dimensional sine-Gordon PDE. Plot (a) shows that the standard Lift \& Learn approach yields quadratic ROMs that achieve relative state error below $10^{-2}$ for ROMs of size $2r>6$ in the training data regime. 
However, the energy error comparison in plot (b) demonstrates that the standard learning approach yields unstable ROMs that exhibit unbounded energy error growth outside the training data regime. The solid black line in plot (b) indicates end of the training time interval.
}
 \label{fig:comp}
\end{figure} 
\section{Structure-preserving Lift \& Learn}
\label{sec:method} 
We introduce \textit{structure-preserving Lift \& Learn} (sp Lift \& Learn) as a nonintrusive model reduction method to learn structure-preserving ROMs for nonlinear conservative PDEs. We leverage energy-quadratizing lifting transformations that ensure the ROM is both energy-conserving \textit{and} its lifted dynamics have only linear and quadratic terms. This has the additional benefit that we explicitly know the quadratic terms in the ROM (they arise from the lifting) and we only need to learn the linear terms that contain the reduced spatial-derivative operators.

In Section~\ref{sec:conservative_model} we derive the nonlinear conservative FOM for nonlinear conservative PDEs of the form~\eqref{eq:pde}. In Section~\ref{sec:splifting} we describe energy-quadratizing lifting transformations at the PDE level. In Section~\ref{sec:splifting_data} we use these lifting transformations to obtain the lifted reduced data from a nonlinear conservative FOM simulation. In Section~\ref{sec:splifting_opinf} we introduce the proposed hybrid approach to learn structure-preserving quadratic ROMs from the lifted reduced data. In Section~\ref{sec:comp} we compare the proposed method with the Hamiltonian Operator Inference~\cite{sharma2022hamiltonian} method for nonlinear wave equations.  In Section~\ref{sec:comp_numerical} we revisit the motivational example from Section~\ref{sec:motivation} to demonstrate the advantages of using the proposed structure-preserving Lift \& Learn approach over standard Lift \& Learn.
\subsection{Model formulation}
\label{sec:conservative_model}
We rewrite~\eqref{eq:pde} in first-order form by defining $q(\x,t):=\phi(\x,t)$ and $p(\x,t):=\frac{\partial}{\partial t}(\phi(\x,t))$ to obtain
\begin{equation}
\frac{\partial q (\x,t)}{\partial t}=p(\x,t), \qquad \frac{\partial p (\x,t)}{\partial t}=\mathbf \Delta q (\x,t) - f_{\text{non}}(q (\x,t)), 
\label{eq:pde_first}
\end{equation}
which conserves the total energy
\begin{equation}
\mathcal E[q(\x,t),p(\x,t)]:=\int_{\Omega} \left( \frac{1}{2} p(\x,t)^2 + \frac{1}{2}\left(\nabla q(\x,t) \right)^2  + g(q(\x,t)) \right) \dd \x.
\label{eq:energy}
\end{equation}
We assume that the nonlinear potential energy component $g(q(\x,t))$ in~\eqref{eq:energy} is nonnegative which is typically the case for conservative PDEs. A structure-preserving spatial discretization of the first-order PDEs in~\eqref{eq:pde_first} would have the form
\begin{equation}
    \dot{\q}=\p, \qquad
    \dot{\p}=\D \q - \fn(\q),
\label{eq:cons_fom}
\end{equation}
where $\qdot \in \real^{n}$ and $\pdot\in \real^{n}$ denote the time derivatives of the space-discretized state vectors $\q\in \real^{n}$ and $\p\in \real^{n}$, respectively, $\D=\D^\top \in \real^{n \times n}$ is the symmetric discrete Laplacian, and $ \fn(\q) \in \real^{n}$ is the nonlinear component of the space-discretized vector field. The FOM~\eqref{eq:cons_fom} conserves the space-discretized nonlinear FOM energy $ E(\q,\p)=\frac{1}{2} \p^\top \p - \frac{1}{2}\q^\top\D \q + \sum_{i=1}^n  g(q_i) .$ We assume that we know the symbolic form of the PDE~\eqref{eq:pde_first}, but we do not know how it is discretized in space and time, and do not have access to the discretized operators in~\eqref{eq:cons_fom}.
\subsection{Exposing structure-preserving variables that quadratize energy}
\label{sec:splifting}
Given a nonlinear PDE of the form~\eqref{eq:pde_first} with nonlinear potential energy component $g(q)$ in~\eqref{eq:energy} satisfying the nonnegativity condition, we use the energy-quadratization strategy we presented in~\cite{sharma2025nonlinear} to introduce the first auxiliary variable $w_1=\tau_1(q)$ defined by $w_1^2=\kappa^2g(q)$, where $\kappa \in \real$ is a free scalar parameter.We then define the second auxiliary variable as $w_2=\tau_2(q)=\frac{f_{\text{non}}(q)}{\bar{\kappa}w_1}$ with another free scalar parameter $\bar{\kappa}\in \real$. The auxiliary state dynamics for $w_1$ can be written as $ \frac{\partial w_1 (\x,t)}{\partial t}=\frac{\kappa^2\bar{\kappa}}{2}w_2(\x,t) p(\x,t)$. Thus, these first two auxiliary variables ensure that the time evolution equations for $\{q,p,w_1\}$ are quadratic in terms of the lifted variables $\{q, p, w_1, w_2\}$, independent of the form of the nonlinearity in $g(q)$. We then compute auxiliary state dynamics for $w_2$ and if the dynamics are not quadratic in terms of the lifted variables $\{q, p, w_1, w_2\}$ then we introduce another auxiliary variable $w_3$ that yields quadratic dynamics for $w_2$ in terms of the lifted variables $\{q, p, w_1, w_2,w_3\}$. We continue this process of introducing auxiliary variables until we find a lifted PDE with quadratic dynamics in the lifted variables. 

Assuming the energy-quadratization strategy yields a quadratization with $k$ auxiliary variables, the resulting quadratic lifted PDE can be symbolically written as
\begin{align}
\label{eq:gen_lift_pde}
    \frac{\partial q (\x,t)}{\partial t}&=p(\x,t), \nonumber\\
    \frac{\partial p (\x,t)}{\partial t}&=\mathbf \Delta q (\x,t) -\bar{\kappa} w_1(\x,t) w_2(\x,t),\nonumber \\
    \frac{\partial w_1 (\x,t)}{\partial t}&=\frac{\kappa^2\bar{\kappa}}{2}w_2(\x,t) p(\x,t), \\
     \frac{\partial w_j (\x,t)}{\partial t}&=\left(\alpha_jq(\x,t) +\sum_{i=1}^{k}\alpha_{j,i}w_i(\x,t) \right) p(\x,t),  \qquad \text{for} \qquad j=2, \ldots,k, \nonumber 
    \end{align}
where $\alpha_{2},\ldots,\alpha_k$ and $\alpha_{i,1}, \ldots, \alpha_{i,k}$ for $i=2,\ldots,k$ are real-valued constant coefficients such that the constants in the set $\pmb \alpha_{i}:=\{\alpha_{i},\alpha_{i,1}, \ldots, \alpha_{i,k}\}$ can not be all zero for $i=2,\ldots,k$. This quadratic lifted PDE possesses a quadratic invariant in terms of the lifted variables
\begin{equation}
\mathcal E_{\text{lift}}[q(\x,t),p(\x,t),w_1(\x,t),\cdots,w_k(\x,t)]:=\int_{\Omega} \left( \frac{1}{2} p(\x,t)^2 + \frac{1}{2}\left(\nabla q(\x,t) \right)^2  +\frac{1}{\kappa^2}w_1(\x,t)^2 \right) \dd \x.
\label{eq:energy_quad}
\end{equation}
\subsection{Constructing reduced snapshot data in the lifted setting}
\label{sec:splifting_data}
If we were to discretize~\eqref{eq:gen_lift_pde}, the model would have the form
\begin{align}
\label{eq:gen_lift_fom}
    \dot{\q}&=\p, \nonumber\\
    \dot{\p}&=\D\q -\bar{\kappa} \w_1\odot \w_2,\nonumber \\
    \dot{\w}_1&=\frac{\kappa^2\bar{\kappa}}{2}\w_2 \odot \p, \\
    \dot{\w}_j&=\left(\alpha_j\q +\sum_{i=1}^{k}\alpha_{j,i}\w_i \right) \odot \p,  \qquad \text{for} \qquad j=2, \ldots,k. \nonumber 
    \end{align}
This lifted FOM can be rewritten in a standard quadratic form by expressing each Hadamard product as a Kronecker-product term:
\begin{align}
\label{eq:gen_lift_fom_kron}
    \dot{\q}&=\p, \nonumber\\
    \dot{\p}&=\D\q + \bH_{\p} (\w_1\otimes \w_2),\nonumber \\
    \dot{\w}_1&=\bH_{\w_1}(\w_2 \otimes \p), \\
    \dot{\w}_j&=\bH_{\w_j}(\q\otimes \p) +\sum_{i=1}^{k}\bH_{\w_{j,i}} (\w_i  \otimes \p),  \qquad \text{for} \qquad j=2, \ldots,k, \nonumber 
    \end{align}
where the matrices $\bH_{\p} \in \real^{n \times n^2}$, $\bH_{\w_1}\in \real^{n \times n^2}$, and  $\bH_{\w_{j}}\in \real^{n \times n^2}, \bH_{\w_{j,1}}\in \real^{n \times n^2},\cdots,\bH_{\w_{j,k}}\in \real^{n \times n^2}$ for $j=2, \ldots,k$ are sparse matrices of the coefficients of the quadratic terms in~\eqref{eq:gen_lift_fom}. 
In this lifted setting, the quadratic FOM~\eqref{eq:gen_lift_fom} conserves the lifted FOM energy
\begin{equation}
\label{eq: lift_energy}
E_{\text{lift}}(\q,\p,\w_1, \cdots,\w_{k})=\frac{1}{2}\p^\top\p - \frac{1}{2}\q^\top \D\q +\frac{1}{\kappa^2} \w_1^\top\w_1. 
\end{equation}
We emphasize that since we derived the algebraic relations of the lifted variables to the original variables, i.e., $w_i=\tau_i(q)$ for $i=1, \cdots, k$, we can directly insert those relationships into the discretized system, c.f., the quadratic terms in~\eqref{eq:gen_lift_fom}. In other words, we do not need to know the specific discretization scheme used for deriving~\eqref{eq:cons_fom} to determine those quadratic terms.

To construct the basis matrix for the augmented state vector in the lifted quadratic FOM~\eqref{eq:gen_lift_fom} without ever solving that system, we first build position and momentum snapshot data matrices by simulating the nonlinear conservative FOM~\eqref{eq:cons_fom}. Let $(\q_1,\p_1),\cdots,(\q_K,\p_K)$ be the solutions of~\eqref{eq:cons_fom} at time $t_1,\cdots,t_K$ computed with a structure-preserving time integrator~\cite{sharma2020review}. Then, the position and momentum snapshot data matrices are defined as
\begin{equation}
\Q=[\q_1, \cdots, \q_K] \in \real^{n \times K}, \qquad \Pp=[\p_1,\cdots, \p_K] \in \real^{n \times K}.
\end{equation}
We compute the lifted states as $\w_{i,j}=\tau_{i}(\q_j)$ and construct the lifted snapshot data matrix for each lifted variable 
\begin{equation}
\W_i=[\w_{i,1},\cdots,\w_{i,K}] \in \real^{n\times K}, \qquad i=1, \ldots,k.
\end{equation}
To obtain low-dimensional data via projection, we use a block-diagonal basis matrix $\bar{\V} \in \real^{\nbar\times \bar{r}}$ that preserves the coupling structure
\begin{equation}
\bar{\V} =\text{blkdiag}(\bPhi,\bPhi,\V_{1}, \cdots,\V_{k})\in \real^{\nbar\times \bar{r}},
\label{eq:basis}
\end{equation}
where $\bPhi$ is the PSD basis matrix for $\q$ and $\p$ computed using the cotangent lift algorithm, and $ \V_{i}$ is the POD basis matrix that contains as columns the POD basis vectors for $\w_i$ for $i=1, \ldots, k$.\footnote{The PSD basis matrix $\bPhi$ in the cotangent lift algorithm is computed via the singular value decomposition of the extended snapshot data matrix $\Y_e:=[\Q,\Pp] \in \real^{n \times 2K}$ whereas the POD basis matrix $ \V_{i}$ for $\w_i$ is computed via the singular value decomposition of the corresponding snapshot data matrix $\W_i\in \real^{n \times K}$.} We obtain projections of the snapshot data matrices as
\begin{equation}
\Qhat=\bPhi^\top\Q \in \real^{r \times K}, \qquad \Phat=\bPhi^\top\Pp \in \real^{r \times K}, \qquad \What_i=\V_{i}^\top\W_i \in \real^{r \times K}, \qquad i=1,\ldots,k.
\end{equation}
Using the reduced momentum trajectories, we obtain reduced time-derivative data $\Phatdot \in \real^{r \times K}$ via an eighth-order central finite difference scheme.  Although lower-order schemes (such as second- or third-order finite difference schemes) may suffice, we use an eighth-order scheme to make negligible derivative approximation errors that could affect the learned ROM operators. 
\subsection{Learning structure-preserving quadratic ROMs via constrained Operator Inference}
\label{sec:splifting_opinf}
Based on the block-diagonal basis matrix $\bar{\V}$ in~\eqref{eq:basis}, the postulated model form--which would arise if we were to project the lifted FOM~\eqref{eq:gen_lift_fom} with $\bar{\V}$--for the structure-preserving ROM is
\begin{align} 
\label{eq:learned_lift_rom_odot}
    \dot{\qhat}&=\phat,\nonumber \\
    \dot{\phat}&=\Dhat\qhat +  \bPhi^\top\bH_{\p}\left(\V_{1}\what_1 \otimes \V_{2}\what_2\right) ,\nonumber \\
    \dot{\what}_1&=\V_{1}^\top \bH_{\w_1} \left( \V_{2}\what_2 \otimes\bPhi\phat\right),\\
        \dot{\what}_j&=\V_{j}^\top\bH_{\w_j} (\bPhi\qhat \otimes \bPhi \phat) + \sum_{i=1}^{k}\V_{j}^\top \bH_{\w_{j,i}} \left(  \V_{i}\what_i \otimes \bPhi \phat \right),  \qquad \text{for} \qquad j=2, \ldots,k, \nonumber 
    \end{align}
where $\Dhat=\Dhat^\top \in \real^{r \times r}$ is the nonintrusive analogue of the symmetric ROM operator that would arise if we were to project intrusively.
Since none of the quadratic terms in the lifted FOM~\eqref{eq:gen_lift_fom} require approximation of spatial derivatives, the corresponding
reduced quadratic terms in~\eqref{eq:learned_lift_rom_odot} after projection can be computed analytically.  Hence, $\Dhat=\Dhat^\top\in \real^{r\times r}$ is the only matrix that we need to learn from data. We simplify the notation by re-writing the quadratic ROM form in~\eqref{eq:learned_lift_rom_odot} as
\begin{align} 
\label{eq:learned_lift_rom}
    \dot{\qhat}&=\phat,\nonumber \\
    \dot{\phat}&=\Dhat\qhat + \Hhat_{\p}(\what_1\otimes\what_2) ,\nonumber \\
    \dot{\what}_1&=\Hhat_{\w_1}(\what_2\otimes\phat),\\
        \dot{\what}_j&=\Hhat_{\w_j}(\qhat\otimes\phat) + \sum_{i=1}^k\Hhat_{\w_{j,i}}(\what_i\otimes \phat),  \qquad \text{for} \qquad j=2, \ldots,k. \nonumber 
    \end{align}
    where  $\Hhat_{\p}:=\bPhi^\top\bH_{\p}\left(\V_{1} \otimes \V_{2}\right)\in \real^{r \times r^2}$, $\Hhat_{\w_1}:=\V_{1}^\top \bH_{\w_1} \left( \V_{2} \otimes\bPhi \right)\in \real^{r \times r^2}$, $\Hhat_{\w_j}:=\V_{j}^\top\bH_{\w_j} (\bPhi \otimes \bPhi) \in \real^{r \times r^2}$, and $\Hhat_{\w_{j,i}}:=\V_{j}^\top\bH_{\w_{j,i}}\left(  \V_{i} \otimes \bPhi \right)\in \real^{r \times r^2}$ for $i=1, \ldots,k$ and $j=2, \ldots,k$. 
As mentioned above, the matrices $\Hhat_{\p}$, $\Hhat_{\w_1}$, and  $\Hhat_{\w_{j}}, \Hhat_{\w_{j,1}},\cdots,\Hhat_{\w_{j,k}}$ for $j=2, \ldots,k$ are merely representations of the quadratic products that we already derived analytically in~\eqref{eq:learned_lift_rom_odot}. Thus, there is no need to learn them.

Next, we use the analytically constructed quadratic terms in~\eqref{eq:learned_lift_rom} to infer $\Dhat$ from projections of the FOM snapshot data. 
Based on the postulated model form for the quadratic ROM in~\eqref{eq:learned_lift_rom}, we formulate the following constrained optimization problem for learning $\Dhat$:
\begin{equation}
\min_{\Dhat=\Dhat^\top} \lVert \Phatdot - \Hhat_{\p}(\What_1\otimes\What_2) - \Dhat\Qhat \rVert_F.
\end{equation}
Although this paper focuses on nonlinear conservative PDEs of the form~\eqref{eq:pde}, we note that the proposed approach can be flexibly adapted to learn structure-preserving ROMs for a more general class of coupled conservative PDEs. In Section~\ref{sec:kgz} we provide evidence for the wide scope of the proposed approach through the numerical example of two-dimensional Klein-Gordon-Zakharov equations, a system of coupled conservative PDEs that does not have a canonical Hamiltonian formulation.
\begin{proposition} The structure-preserving Lift \& Learn ROM~\eqref{eq:learned_lift_rom} conserves the perturbed lifted FOM energy 
\begin{equation}
\label{eq:pertrubed_energy}
\widehat{E}_{\mathrm{lift}}(\qhat,\phat,\what_1, \cdots,  \what_{k}):=E_{\mathrm{lift}}(\bPhi\qhat,\bPhi\phat,\V_{1}\what_1, \cdots,\V_{k}\what_{k}) + \Delta E_{\mathrm{lift}}(\qhat)
\end{equation}
with a perturbation of the form $\Delta E_{\mathrm{lift}}(\qhat)=\frac{1}{2}\qhat^\top (\bPhi^\top\D\bPhi-\Dhat)\qhat$ where $\D=\D^\top \in \real^{n \times n}$ is from~\eqref{eq:gen_lift_fom}.
\end{proposition}
\begin{proof}
Given a lifted quadratic FOM of the form~\eqref{eq:gen_lift_fom}, we evaluate the lifted FOM energy function $E_{\text{lift}}$~\eqref{eq: lift_energy} at the lifted state approximation to obtain
\begin{align}
\label{eq:lift_energy}
E_{\text{lift}}(\bPhi\qhat,\bPhi\phat,\V_{1}\what_1, \cdots,\V_{k}\what_{k})&=\frac{1}{2}(\bPhi\phat)^\top(\bPhi\phat) - \frac{1}{2}(\bPhi\qhat)^\top \D(\bPhi\qhat) + \frac{1}{\kappa^2}(\V_{1}\what_1)^\top(\V_{1}\what_1) \nonumber\\
&=\frac{1}{2}\phat^\top\phat - \frac{1}{2}\qhat^\top (\bPhi^\top\D\bPhi )\qhat + \frac{1}{\kappa^2}\what_1^\top\what_1.
\end{align}
Defining $\Delta E_{\mathrm{lift}}(\qhat):=\frac{1}{2}\qhat^\top (\bPhi^\top\D\bPhi-\Dhat)\qhat$ we see that the energy associated with the lifted ROM~\eqref{eq:learned_lift_rom} is--analogous to~\eqref{eq: lift_energy}--given by
\begin{equation*}
\widehat{E}_{\text{lift}}(\qhat,\phat,\what_1, \cdots,  \what_{k})=\frac{1}{2}\phat^\top\phat - \frac{1}{2}\qhat^\top (\bPhi^\top\D\bPhi )\qhat + \frac{1}{\kappa^2}\what_1^\top\what_1 + \frac{1}{2}\qhat^\top (\bPhi^\top\D\bPhi-\Dhat)\qhat=\frac{1}{2}\phat^\top\phat - \frac{1}{2}\qhat^\top \Dhat\qhat +\frac{1}{\kappa^2} \what_1^\top\what_1. 
\end{equation*}
We now compute the time-derivative of $\widehat{E}_{\text{lift}}(\qhat,\phat,\what_1, \cdots,  \what_{k})$
\begin{align*} 
    \frac{\dd}{\dd t}\widehat{E}_{\text{lift}}(\qhat,\phat,\what_1, \cdots,\what_{k})&=\phat^\top\phatdot - (\Dhat\qhat)^\top\qhatdot +\frac{2}{\kappa^2} \what_1^\top\dot{\what}_1 \\
    &=-\bar{\kappa}\phat^\top\bPhi^\top\left( \V_{1}\what_1 \odot  \V_{2}\what_2\right) +  \bar{\kappa} \what_1^\top\V_{1}^\top \left( \V_{2}\what_2 \odot \bPhi\phat\right)\\
    &= 0. 
        \end{align*}
Thus, the quadratic ROM learned via structure-preserving Lift \& Learn conserves $\widehat{E}_{\text{lift}}(\qhat,\phat,\what_1, \cdots,\what_{k})$.
\end{proof}
We showed in~\cite{sharma2025nonlinear} that the structure-preserving quadratic ROM obtained via intrusive lifting conserves the lifted FOM energy exactly, which means that it possesses a quadratic invariant of the form $\widetilde{E}_{\text{lift}}(\qhat,\phat,\what_1, \cdots,\what_{k}):=E_{\text{lift}}(\bPhi\qhat,\bPhi\phat,\V_{1}\what_1, \cdots,\V_{k}\what_{k})$. Building on this result, the quadratic invariant $\widehat{E}_{\text{lift}}(\qhat,\phat,\what_1, \cdots,\what_{k})$ in~\eqref{eq:pertrubed_energy} can be interpreted as a perturbation of $\widetilde{E}_{\text{lift}}(\qhat,\phat,\what_1, \cdots,\what_{k})$. Moreover, this perturbation term is bounded as follows
\begin{equation}
|\Delta E_{\text{lift}}(\qhat)|=\frac{1}{2}\qhat^\top (\widetilde{\D}-\Dhat)\qhat \leq \frac{1}{2} \lVert  \widetilde{\D}-\Dhat \rVert \lVert \qhat \rVert^2,
\end{equation}
where $\widetilde{\D}:=\bPhi^\top\D\bPhi \in \real^{r \times r}$ is the linear ROM operator obtained via intrusive projection. Thus, the structure-preserving Lift \& Learn ROM trajectories conserve a perturbed lifted FOM energy where the magnitude of the perturbation depends on the difference between the learned linear ROM operator $\Dhat$ and the intrusively projected linear ROM operator $\widetilde{\D}$.
\begin{remark}
The nonlinear FOM for conservative PDEs of the form~\eqref{eq:pde} can also be written in the Lagrangian form where the governing equations are a set of $n$ coupled second-order ordinary differential equations. The Lagrangian Operator Inference method~\cite{sharma2024preserving} for nonlinear wave equations also uses knowledge of the nonlinear potential energy at the PDE level to learn low-dimensional nonlinear Lagrangian ROMs from high-dimensional FOM data. However, these learned Lagrangian ROMs suffer from computational efficiency issues as the evaluation of the nonlinear ROM vector field scales
with the FOM dimension. To tackle this challenge and have efficient ROMs, the Lagrangian
Operator Inference approach would need an additional structure-preserving hyper-reduction step,  which to the best of the authors' knowledge remains an open problem. In contrast, the
proposed structure-preserving Lift \& Learn approach yields computationally efficient quadratic ROMs that do not require a second level of approximation.
\end{remark}
\begin{remark}
While our knowledge of the symbolic form of the quadratic lifted PDE~\eqref{eq:gen_lift_pde} enables us to derive the quadratic ROM form~\eqref{eq:learned_lift_rom_odot}, we do not formally define, or numerically solve, the lifted PDE~\eqref{eq:gen_lift_pde}. 
\end{remark}
\subsection{Comparison of structure-preserving Lift \& Learn and Hamiltonian Operator Inference}
\label{sec:comp}
In this section and in the numerical results, we compare the proposed structure-preserving Lift \& Learn approach with Hamiltonian Operator Inference (HOpInf)~\cite{sharma2022hamiltonian}, a structure-preserving learning approach for canonical Hamiltonian PDEs. The HOpInf method also uses knowledge of the functional form of the nonlinear potential energy at the PDE level to postulate a Hamiltonian model of the form
\begin{equation}
\dot{\qhat}=\Dhat_{\q}\phat,\qquad \dot{\phat}=\Dhat_{\p}\qhat+\bPhi^\top\f_{\text{non}}(\bPhi\qhat),
\label{eq:lopinf_rom}
\end{equation}
where the symmetric reduced matrices $\Dhat_{\q}=\Dhat_{\q}^\top \in \real^{r \times r}$ and $\Dhat_{\p}=\Dhat_{\p}^\top \in \real^{r \times r}$ are inferred from data. The symmetric constraints on  $\Dhat_{\q}$ and $\Dhat_{\p}$ ensure that the linear components of the learned Hamiltonian ROM retains the symmetric property of the linear FOM operator introduced during the structure-preserving discretization.

Given $K$ reduced state snapshots and the reduced time derivative data, HOpInf formulates the following separate symmetric linear least-squares problems for learning $\Dhat_{\q}$ and $\Dhat_{\p}$:
\begin{equation}
\min_{\Dhat_{\q}=\Dhat_{\q}^\top} \lVert \Qhatdot - \Dhat_{\q}\Phat \rVert_F, \qquad
\min_{\Dhat_{\p}=\Dhat_{\p}^\top} \lVert \Phatdot - \Fhat_{\text{non}}(\Qhat) - \Dhat_{\p}\Qhat \rVert_F.
\end{equation}
The HOpInf approach, similar to Lagrangian Operator Inference,  still requires evaluating the nonlinear vector field~\eqref{eq:lopinf_rom} that scales with the FOM dimension. To reduce the online computational cost of simulating~\eqref{eq:lopinf_rom}, the HOpInf approach needs an additional structure-preserving hyper-reduction step. The authors in~\cite{pagliantini2023gradient} presented a gradient-preserving DEIM strategy that yields nonlinear Hamiltonian ROMs that conserve the nonlinear FOM energy in an asymptotic sense at a computational cost independent of the FOM dimension. We therefore consider the HOpInf approach combined with structure-preserving DEIM (spDEIM) as the state of the art and use it to compare our structure-preserving Lift \& Learn method with. 
%
%
\begin{figure}[tbp]
\small
\captionsetup[subfigure]{oneside,margin={1.8cm,0 cm}}
\begin{subfigure}{.45\textwidth}
       \setlength\fheight{6 cm}
        \setlength\fwidth{\textwidth}
\input{figures/motivation/state_train_final_hopinf.tex}
\caption{Relative state error over the training data }
\label{fig:hopinf_state_train}
    \end{subfigure}
    \hspace{0.2cm}
    \begin{subfigure}{.45\textwidth}
           \setlength\fheight{6 cm}
           \setlength\fwidth{\textwidth}
\raisebox{-15mm}{\input{figures/motivation/fom_energy_final_hopinf_sparse.tex}}
\caption{FOM energy error}
\label{fig:hopinf_energy}
    \end{subfigure}
\caption{Comparison of structure-preserving Lift \& Learn and HOpInf with spDEIM. Plots (a) and (b) show that both approaches learn accurate and stable ROMs with bounded FOM energy error. 
}
 \label{fig:lopinf}
\end{figure}
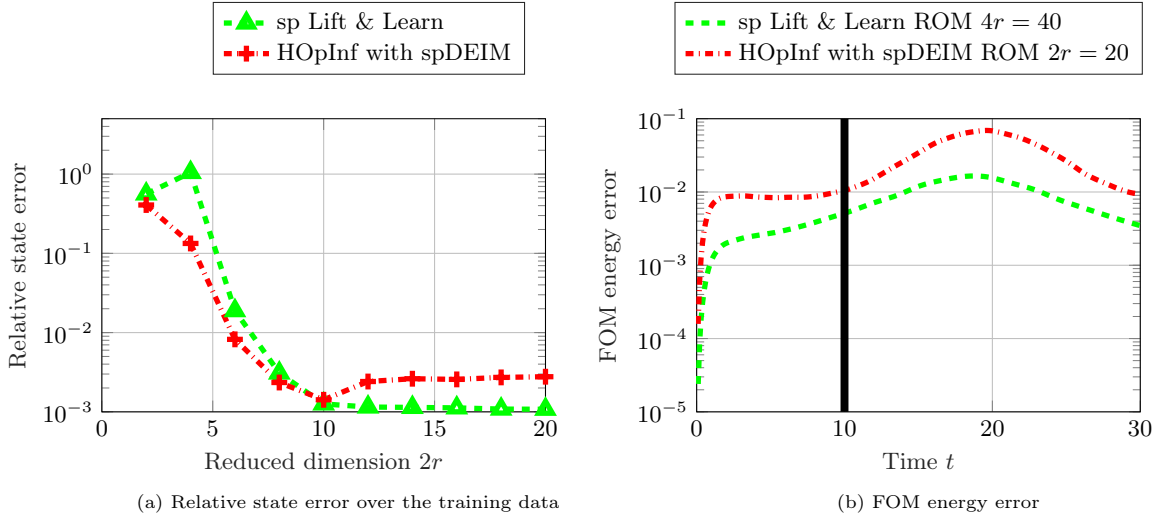
\subsection{Revisiting the one-dimensional sine-Gordon equation}
\label{sec:comp_numerical}
For a numerical illustration, we revisit the sine-Gordon equation example from Section~\ref{sec:background}. The proposed structure-preserving Lift \& Learn method learns $4r$-dimensional quadratic ROMs of the form 
\begin{align} 
\label{eq:ll_sg}
    \dot{\qhat}&=\phat, \nonumber \\
    \dot{\phat}&=\Dhat\qhat +\Hhat_{\p}\left(\what_1 \otimes  \what_2\right) , \\
    \dot{\what}_1&=\Hhat_{\w_{1}} \left( \what_2 \otimes \phat\right), \nonumber\\
    \dot{\what}_2&=\Hhat_{\w_{2,1}} \left( \what_1 \otimes \phat\right),\nonumber
\end{align} 
where the quadratic ROM operators $\Hhat_{\p}\in \real^{r \times r^2}$, $\Hhat_{\w_1}\in \real^{r \times r^2}$, and $\Hhat_{\w_{2,1}}\in \real^{r \times r^2}$ are computed analytically, and the symmetric linear ROM operator $\Dhat=\Dhat^\top \in \real^{r \times r}$ is learned from data. 

In Figure~\ref{fig:lopinf}, we compare the accuracy and energy error behavior for structure-preserving Lift  \& Learn ROMs with $2r$-dimensional nonlinear Hamiltonian ROMs obtained via HOpInf with $r_{\text{spDEIM}}=2r$. The relative state error comparison in Figure~\ref{fig:hopinf_state_train} shows that both approaches achieve similar accuracy in the training data regime up to $2r=10$. For $2r>10$, we observe that the accuracy in the training data does not decrease as favorably with an increase in the reduced dimension.  The energy error comparison in Figure~\ref{fig:hopinf_energy} shows that both approaches achieve bounded FOM energy error that stays below $10^{-1}$ for all times. Compared to Figure~\ref{fig:comp_energy}, the bounded energy error $200\%$ past the training time interval for the structure-preserving Lift \& Learn ROM shows that the proposed approach learns ROMs that respect the underlying physics of the problem.
%

\section{Numerical results}
\label{sec:numerical}
In this section, we investigate the numerical performance of the proposed structure-preserving Lift \& Learn method for three nonlinear conservative PDEs. In Section~\ref{sec:exp}, we study the one-dimensional nonlinear wave equation with exponential nonlinearity. In Section~\ref{sec:sg_2d} we consider the two-dimensional sine-Gordon equation.  Finally, in Section~\ref{sec:kgz}, we consider a nonlinear conservative FOM with $960{,}000$ degrees of freedom to derive structure-preserving ROMs for the two-dimensional Klein-Gordon-Zakharov equations, a system of coupled conservative PDEs that does not have a canonical Hamiltonian formulation. Below, we list practical details regarding the numerical implementation and the error measures used in the numerical experiments.

\begin{itemize}
\item The numerical experiments in Section~\ref{sec:exp} and Section~\ref{sec:sg_2d} are implemented in MATLAB 2022a on a quad-core Intel i7 processor with 2.3 GHz and 32 GB memory. The numerical experiments for the two-dimensional Klein-Gordon-Zakharov equations in Section~\ref{sec:kgz} are implemented in MATLAB 2022b using compute nodes of the Triton Shared Computing Cluster~\cite{san2022triton} equipped with $8$ processing cores of Intel Xeon Platinum 64-core CPU at 2.9 GHz and 1 TB memory.
\item The constrained optimization problems to learn the linear reduced operators in this work are solved using the CVX 2.2~\cite{cvx}, a MATLAB-based modeling system for convex optimization. The symmetric constraints and the optimization objectives for a given constrained optimization problem are specified using standard MATLAB expression syntax in the CVX package.
\item For nonlinear Hamiltonian ROMs obtained via HOpInf~\cite{sharma2022hamiltonian}, we use the implicit midpoint rule for numerical time integration. The implicit midpoint rule is a second-order symplectic integrator that exhibits bounded energy error for nonlinear Hamiltonian systems~\cite{hairer2006geometric}.
\item For quadratic ROMs obtained via the proposed structure-preserving Lift \& Learn approach, we use Kahan's method for numerical time integration. Kahan's method is a second-order structure-preserving numerical integrator for quadratic vector fields, see~\cite{celledoni2012geometric} for more details about  Kahan's method. 
\item The \emph{relative state error} in $q$ is computed in the entire training or testing intervals as  
\begin{equation}\label{eq:err_state}
 \text{Relative state error in } q =\frac{\lVert \Q- \bPhi\widehat{\Q} \rVert^2_F}{\lVert \Q \rVert^2_F},
\end{equation}
where $\widehat{\Q}=[\qhat_1, \cdots, \qhat_K] \in \real^{r \times K}$ is the ROM snapshot data obtained from the ROM simulations and $\bPhi\widehat{\Q}\in \real^{n \times K}$ is the approximation of the FOM snapshot data $\Q$. 
\item We measure the common cost/accuracy tradeoff for ROMs using the \emph{efficacy} metric which is computed as
\begin{equation}\label{eq:eff}
{\text{Efficacy}}= \frac{1}{\text{relative state error in training data regime} \times \text{wall-clock time in seconds}} \ ,
\end{equation}
where the MATLAB wall-clock time is obtained by calculating the average over 20 runs. In comparisons based on this metric, the model reduction approach with higher efficacy is considered advantageous. The reported models are sufficiently accurate such that reporting efficacy is sensible. 
\item The \emph{FOM energy error} is computed as follows:
\begin{equation}\label{eq:err_fom}
  \text{FOM energy error}=  \left| E(\bPhi\qhat(t),\bPhi \phat(t)) -E(\bPhi\qhat(0),\bPhi \phat(0)) \right|,
\end{equation}
where $E(\bPhi\qhat(t),\bPhi \phat(t))$ is the FOM energy approximation obtained by evaluating the space-discretized nonlinear FOM energy $E$ for the FOM state approximation at time $t$. 
\end{itemize}
\subsection{Nonlinear wave equation with exponential nonlinearity}
\label{sec:exp}
The one-dimensional nonlinear wave equation with exponential nonlinearity considered here arises from the Johnson–Mehl–Avrami–Kolmogorov theory~\cite{avrami1939kinetics, johnson1939reaction, kolmogorov1937statistical} of nucleation and growth reactions for modeling the kinetics of phase change. 
\subsubsection{PDE formulation and the corresponding nonlinear conservative FOM}
We consider the FOM setup from~\cite{li2020linearly}. Let $\Omega=(0,\pi) \subset \real$ be the spatial domain and consider the one-dimensional nonlinear wave equation
\begin{equation}
\frac{\partial^2 \phi(x,t)}{\partial t^2}=\frac{\partial^2 \phi(x,t)}{\partial x^2}+\exp(-\phi(x,t)),
\label{eq:exp_wave}
\end{equation}
with scalar field variable $\phi(x,t)$ at spatial location $x \in \Omega$ and time $t \in (0,T]$. We consider homogenous Dirichelet boundary conditions $
\phi(0,t)=\phi(\pi,t)=0$
for $t \in (0,T]$. The corresponding initial conditions  are $(\phi(x,0), \frac{\partial \phi}{\partial t} (x,0))=(0.5x(\pi-x), 0)$.
We define $q(x,t):=\phi(x,t)$ and $p(x,t):=\frac{\partial \phi(x,t)}{\partial t}$ to
 rewrite~\eqref{eq:exp_wave} in first-order form and then discretize the system of first-order PDEs using $n=200$ equally spaced grid points to derive the nonlinear conservative FOM
\begin{equation*} 
    \dot{\q}=\p, \qquad
    \dot{\p} =\D\q + \exp (-\q),
    \label{eq:exp_fom}
\end{equation*}
where $\D=\D^\top$ denotes the symmetric discrete Laplacian. The nonlinear conservative FOM conserves the space-discretized energy 
\begin{equation*}
E(\q,\p)=\frac{1}{2} \p^\top \p -\frac{1}{2}\q^\top\D \q + \sum_{i=1}^n \left( \exp(-q_i) \right).
\end{equation*}
\subsubsection{Model form for structure-preserving Lift \& Learn based on energy quadratization} 
Based on the exponential term in the nonlinear FOM energy expression, we solve $ w^2=\kappa^2 g(q)$ with $g(q)=\exp(-q)$ and choose $\kappa=1$ to obtain the first auxiliary variable $w=\exp(-q/2)$, which quadratizes the nonlinear energy in the lifted variables. Since the time evolution equations for $\{\q,\p,\w\}$ are quadratic in terms of $\{\q,\p,\w \}$ in this example, we do not need to introduce additional auxiliary variables. The resulting lifted FOM  
\begin{align*} 
    \dot{\q}&=\p, \\
    \dot{\p}&=\D \q + \w \odot \w, \\
    \dot{\w}&=-\frac{1}{2}\w \odot \p ,
  \end{align*}
conserves the lifted FOM energy $E_{\text{lift}}(\q,\p,\w)=\frac{1}{2}\p^\top\p - \frac{1}{2}\q^\top \D\q + \w^\top\w $. We consider a block-diagonal basis matrix of the form
\begin{equation*}
\bar{\V} =\text{blkdiag}(\bPhi,\bPhi,\V)\in \real^{3n\times 3r},
\end{equation*}
where the PSD basis matrix $\bPhi \in \real^{n\times r}$ for $\q$ and $\p$ is computed using the cotangent lift algorithm and the POD basis matrix $\V \in \real^{n \times r}$ for the auxiliary variable $\w$ is computed via the singular value decomposition of the lifted snapshot data matrix $\textbf W \in \real^{n \times r}$. The model form for the nonintrusive quadratic ROM is 
\begin{align*} 
    \dot{\qhat}&=\phat, \\
    \dot{\phat}&=\Dhat\qhat +\Hhat_{\p}\left(\what \otimes  \what\right) , \\
    \dot{\what}&=\Hhat_{\w} \left( \what \otimes \phat\right),
    \end{align*} 
  where the reduced quadratic operators $\Hhat_{\p} \in \real^{r \times r^2}$ and $\Hhat_{\w}\in \real^{r \times r^2}$ are computed analytically.
\subsubsection{Numerical results}
 The training dataset is built by integrating the nonlinear FOM until time $t=10$ using the symplectic midpoint rule with $\dt=0.005$. We consider FOM snapshots from $t=10$ to $t=100$ as the test dataset to study the predictive capability of quadratic ROMs learned via structure-preserving Lift \& Learn. For comparison, we consider quadratic ROMs obtained via lifting and intrusive projection, and nonlinear ROMs obtained via HOpInf with $r_{\text{spDEIM}}=r$. 

In Figure~\ref{fig:exp_1d_state}, we evaluate the numerical performance of the aforementioned approaches by comparing the relative state error~\eqref{eq:err_state} for different reduced dimension.  The state error plots in Figure~\ref{fig:exp_state_train} and Figure~\ref{fig:exp_state_test} show that all three approaches yields ROMs with similar accuracy in both training and test data regimes. 

We compare the computational efficiency of structure-preserving Lift \& Learn against the HOpInf with spDEIM approach with $r_{\text{spDEIM}}=r$ through efficacy plots in Figure~\ref{fig:exp_1d_eff}. The $3r$-dimensional quadratic ROMs learned via structure-preserving Lift \& Learn achieve higher efficacy than the $2r$-dimensional nonlinear Hamiltonian ROMs obtained via HOpInf with spDEIM. Thus, compared to HOpInf with spDEIM, the proposed approach learns ROMs that achieve similar state error at a substantially lower computational cost for this example. The FOM energy error plots in Figure~\ref{fig:exp_1d_ener} show that the all three approaches achieve bounded FOM energy error that stays below $10^{-3}$ for all times. In this example, the FOM energy is approximately $2.80$, while the perturbation term $\Delta E_{\mathrm{lift}}$ discussed in~\eqref{eq:pertrubed_energy} remains on the order of $10^{-4}$–$10^{-5}$ throughout the simulation. These results demonstrate the ability of the learned quadratic ROM to provide accurate and stable predictions far outside the training data regime.
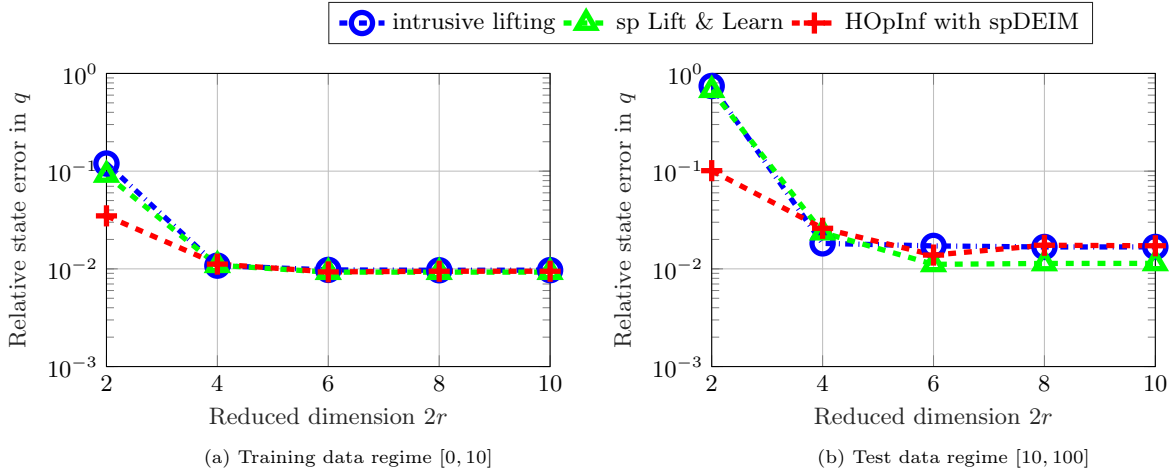
\begin{figure}[tbp]
\small
\captionsetup[subfigure]{oneside,margin={1.8cm,0 cm}}
\begin{subfigure}{.45\textwidth}
       \setlength\fheight{6 cm}
        \setlength\fwidth{\textwidth}
\input{figures/exp_1d/state_train.tex}
\caption{Training data regime $[0,10]$}
\label{fig:exp_state_train}
    \end{subfigure}
    \hspace{0.4cm}
    \begin{subfigure}{.45\textwidth}
           \setlength\fheight{6 cm}
           \setlength\fwidth{\textwidth}
\raisebox{-57mm}{\input{figures/exp_1d/state_test.tex}}
\caption{Test data regime $[10,100]$}
\label{fig:exp_state_test}
    \end{subfigure}
\caption{Nonlinear wave equation with exponential nonlinearity. Quadratic ROMs obtained via structure-preserving Lift \& Learn achieve similar state error to nonlinear ROMs obtained via intrusive lifting and HOpInf with spDEIM in both training and test data regimes.}
 \label{fig:exp_1d_state}
\end{figure}
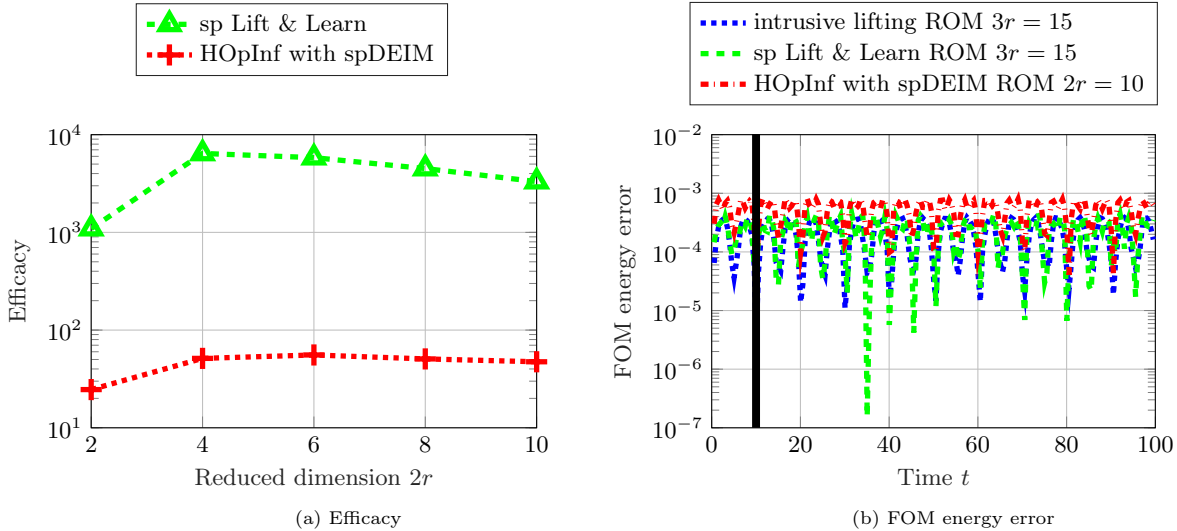
\begin{figure}[tbp]
\small
\captionsetup[subfigure]{oneside,margin={1.8cm,0 cm}}
\begin{subfigure}{.45\textwidth}
       \setlength\fheight{6 cm}
        \setlength\fwidth{\textwidth}
\input{figures/exp_1d/efficacy.tex}
\caption{Efficacy}
\label{fig:exp_1d_eff}
    \end{subfigure}
    \hspace{0.4cm}
    \begin{subfigure}{.45\textwidth}
           \setlength\fheight{6 cm}
           \setlength\fwidth{\textwidth}
\input{figures/exp_1d/fom_energy_sparse.tex}
\caption{FOM energy error}
\label{fig:exp_1d_ener}
    \end{subfigure}
\caption{Nonlinear wave equation with exponential nonlinearity. The efficacy comparison in plot (a) shows that the structure-preserving Lift \& Learn approach achieves higher efficacy than the HOpInf with spDEIM approach. Plot (b) shows that all ROMs demonstrate bounded energy error. The solid black line in plot (b) indicates the end of the training data regime.}
 \label{fig:exp_1d}
\end{figure}
\subsection{Two-dimensional sine-Gordon wave equation}
\label{sec:sg_2d}
We now consider the two-dimensional analogue of the sine-Gordon equation from Section~\ref{sec:background}. 
\subsubsection{PDE formulation and the corresponding nonlinear conservative FOM}
We consider the FOM setup from~\cite{bo2022arbitrary}. Let $\Omega=(-7,7) \times (-7,7) \subset \real^2 $ be the spatial domain and consider the two-dimensional sine-Gordon equation
\begin{equation}\label{eq:sg_2d_pde}
\frac{\partial^2 \phi}{\partial t^2}(x,y,t)=\frac{\partial^2 \phi}{\partial x^2} (x,y,t)+ \frac{\partial^2 \phi}{\partial y^2}(x,y,t) - \sin(\phi(x,y,t)),
\end{equation}
for $t \in (0,T]$. The boundary conditions are periodic and the corresponding initial conditions are
\begin{equation*}
\phi(x,y,0)=4\tan^{-1}\left(\exp\left((3-\sqrt{x^2+y^2}\right)\right), \qquad  \frac{\partial \phi}{\partial t}(x,y,0)=0.
\end{equation*}
We define $q(x,y,t):=\phi(x,y,t)$ and $p(x,y,t):=\partial \phi(x,y,t)/\partial t$ and rewrite the second-order conservative PDE~\eqref{eq:sg_2d_pde} in the first-order form.
We then discretize the computational domain $\Omega$ with $n_x=n_y=100$ equally spaced grid points in both spatial directions leading to a nonlinear FOM of dimension $2n=20,000$
\begin{equation*} 
    \dot{\q}=\p, \qquad
    \dot{\p} =\D\q -\sin( \q),
    \label{eq:kg_fom}
\end{equation*}
where $\D=\D^\top$ denotes the symmetric discrete Laplacian in the two-dimensional setting. The nonlinear FOM conserves the space-discretized energy 
\begin{equation*}
E(\q,\p)=\frac{1}{2} \p^\top \p -\frac{1}{2}\q^\top\D \q + \sum_{i=1}^n (1-\cos(q_i)).
\end{equation*}
\subsubsection{Model form for structure-preserving Lift \& Learn based on energy quadratization}
Since the two-dimensional sine-Gordon equation~\eqref{eq:sg_2d_pde} has the same form of nonlinearity as its one-dimensional counterpart in equation~\eqref{eq:sg_pde}, the nonintrusive ROM form has the same structure as the nonintrusive ROM for the one-dimensional sine-Gordon example in equation~\eqref{eq:ll_sg}.  
\subsubsection{Numerical results}
We build the training dataset by integrating the nonlinear FOM for the two-dimensional sine-Gordon equation from $t=0$ to $t=10$ and we consider FOM snapshots from $t=10$ to $t=12.5$ (predicting 25\% outside training data) as the test dataset. We compare the relative state error of the structure-preserving Lift \& Learn ROMs against intrusive lifting ROMs and HOpInf ROMs with $r_{\text{spDEIM}}=6r$ in Figure~\ref{fig:sg_2d_state}. The relative state error plots in Figure~\ref{fig:sg_state_train} show that the quadratic ROMs learned via the proposed structure-preserving Lift \& Learn approach achieve accuracy similar to the intrusive lifting ROMs in the training data regime. Notably, the relative state error for HOpInf ROMs with $r_{\text{spDEIM}}=6r$ does not decrease as favorably with an increase for $2r>20$. This could be remedied by further increasing $r$, which, however, would make the HOpInf ROM increasingly more costly to simulate. The comparison plots in Figure~\ref{fig:sg_state_test} show that all three approaches achieve similar accuracy in the test data regime. 

In Figure~\ref{fig:sg_2d_eff}, we compare the efficacy of quadratic ROMs obtained via structure-preserving Lift \& Learn and nonlinear Hamiltonian ROMs obtained via HOpInf with spDEIM. We observe that both approaches exhibit a similar trend where the efficacy increases with an increase with the reduced dimension up to $2r=20$. For $2r>20$, the efficacy decreases with an increase in the reduced dimension. The FOM energy error plots in Figure~\ref{fig:sg_2d_ener} compare energy error performance of the proposed approach against intrusive lifting and HOpInf with spDEIM. All three approaches demonstrate similar energy error behavior with the intrusive lifting ROM of dimension $4r=80$ performing marginally better than the other two nonintrusive approaches. For prediction beyond $t\approx12.5$, the energy error can become relatively large but remains bounded, reflecting the limited predictive capability of a reduced basis constructed from training data on $t\in[0,10]$. The FOM energy in this example is approximately $151$ in the time interval $[0,10]$, while the perturbation term $\Delta E_{\mathrm{lift}}$ discussed in~\eqref{eq:pertrubed_energy} remains several orders of magnitude smaller (below $10^{-11}$) over the entire simulation.
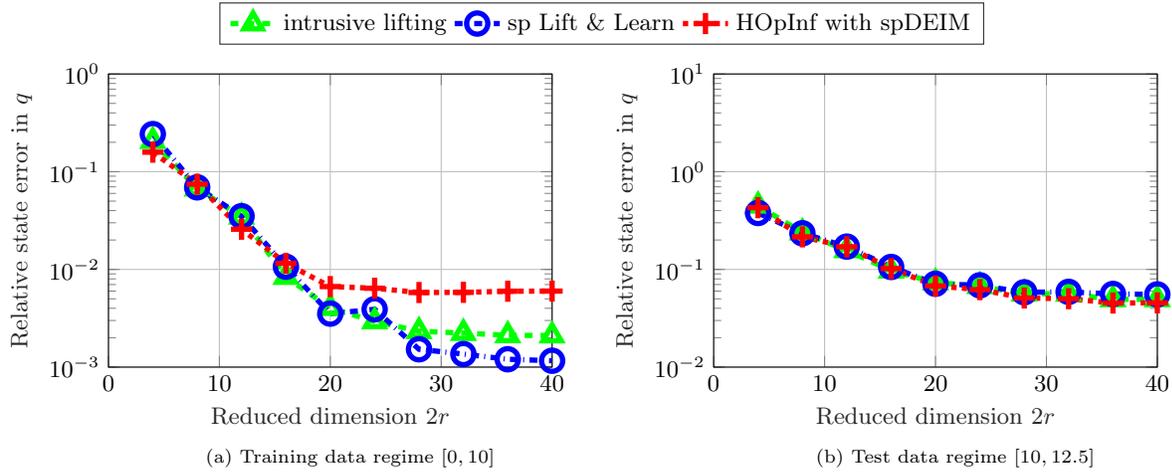
\begin{figure}[tbp]
\small
\captionsetup[subfigure]{oneside,margin={1.8cm,0 cm}}
\begin{subfigure}{.45\textwidth}
       \setlength\fheight{6 cm}
        \setlength\fwidth{\textwidth}
\input{figures/sg_2d/q_train_final.tex}
\caption{Training data regime $[0,10]$}
\label{fig:sg_state_train}
    \end{subfigure}
    \hspace{0.4cm}
    \begin{subfigure}{.45\textwidth}
           \setlength\fheight{6 cm}
           \setlength\fwidth{\textwidth}
\raisebox{-57mm}{\input{figures/sg_2d/q_test_final.tex}}
\caption{Test data regime $[10,12.5]$}
\label{fig:sg_state_test}
    \end{subfigure}
\caption{Two-dimensional sine-Gordon equation. The relative state error comparison in plot (a) shows that proposed structure-preserving Lift \& Learn approach achieves higher accuracy than the HOpInf with spDEIM approach in the training regime. Plot (b) shows that all three approaches achieve similar accuracy in the test data regime. }
 \label{fig:sg_2d_state}
\end{figure}
\begin{figure}[tbp]
\small
\captionsetup[subfigure]{oneside,margin={1.8cm,0 cm}}
\begin{subfigure}{.45\textwidth}
       \setlength\fheight{6 cm}
        \setlength\fwidth{\textwidth}
\input{figures/sg_2d/efficacy_final.tex}
\caption{Efficacy}
\label{fig:sg_2d_eff}
    \end{subfigure}
    \hspace{0.4cm}
    \begin{subfigure}{.45\textwidth}
           \setlength\fheight{6 cm}
           \setlength\fwidth{\textwidth}
\input{figures/sg_2d/fom_energy_final.tex}
\caption{FOM energy error}
\label{fig:sg_2d_ener}
    \end{subfigure}
\caption{Two-dimensional sine-Gordon wave equation. Plot (a) shows that the quadratic ROMs obtained via nonintrusive lifting approach achieve higher efficacy than the HOpInf ROMs with spDEIM. The energy error comparison in plot (b) shows that all approaches yield bounded energy error behavior in the training data regime with the intrusive lifting ROM achieving the lowest FOM energy error. The solid black line in plot (b) indicates the end of the training time interval.}
 \label{fig:sg_2d}
\end{figure}
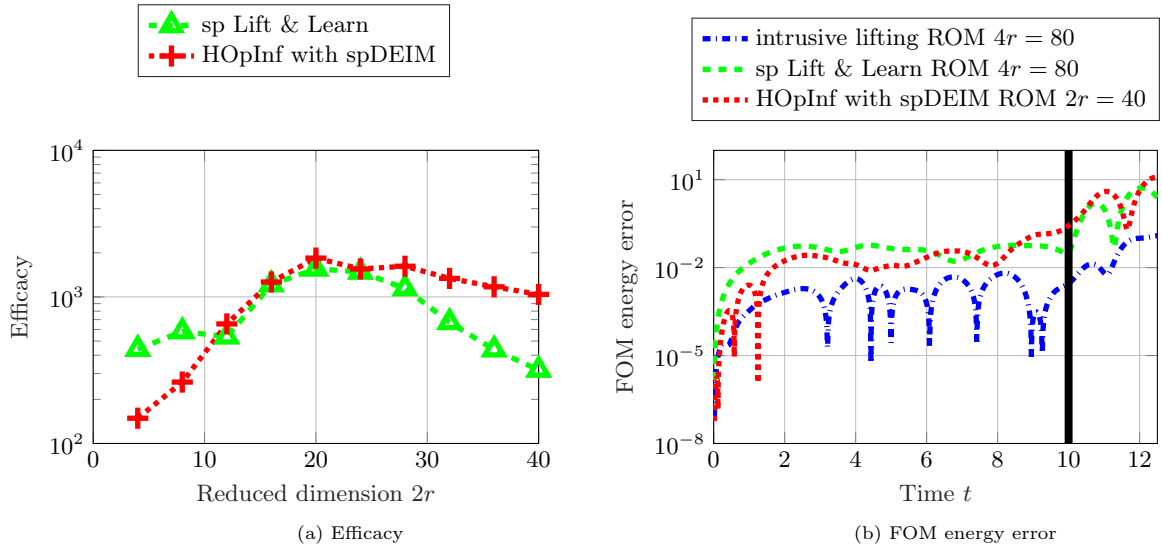
\subsection{Two-dimensional Klein-Gordon-Zakharov equations}
\label{sec:kgz}
The Klein-Gordon-Zakharov (KGZ) equations~\cite{boling1995global} are a system of nonlinear dispersive PDEs used to describe the mutual interaction between the Langmuir waves and ion acoustic waves in a plasma. These coupled equations play a crucial role in the study of the dynamics of strong Langmuir turbulence in plasma physics~\cite{berge1996perturbative}.  Unlike the previous two numerical examples, the nonlinear conservative FOM for the KGZ equations is not of the form~\eqref{eq:cons_fom} as it does not have a canonical Hamiltonian formulation.
\subsubsection{System of coupled PDEs and the corresponding nonlinear conservative FOM}
We consider the FOM setup from~\cite{guo2023energy}. Let $\Omega=(-20,20) \times (-20,20) \subset \real^2$ be the spatial domain and consider the coupled nonlinear equations
\begin{align*}
\frac{\partial ^2\psi}{\partial t^2}(x,y,t)&=\frac{\partial ^2\psi}{\partial x^2}(x,y,t)+\frac{\partial ^2\psi}{\partial y^2}(x,y,t)-\psi(x,y,t) -\psi(x,y,t)\phi(x,y,t)-|\psi(x,y,t)|^2\psi(x,y,t), \\
 \frac{\partial ^2\phi }{\partial t^2}(x,y,t)&=\frac{\partial ^2\phi}{\partial x^2}(x,y,t)+\frac{\partial ^2\phi}{\partial y^2}(x,y,t) + \left(\frac{\partial ^2}{\partial x^2}+  \frac{\partial ^2}{\partial y^2}\right) |\psi (x,y,t)|^2,
\end{align*}
where $t \in (0,T]$ is time, $\psi(x,y,t)$ is a complex-valued scalar field that describes the fast time scale component of the electric field raised by electrons, and $\phi(x,y,t)$ is a real-valued scalar field that describes the deviation of the ion density from its equilibrium. The boundary conditions are periodic and the initial conditions are
\begin{align*}
\psi(x,y,0)&=\sech(-(x-2)^2-y^2) + \sech(-x^2-(y-2)^2), \qquad \frac{\partial \psi}{\partial t}(x,y,0)=0,\\
 \phi (x,y,0)&=\sech(-(x-2)^2-y^2) + \sech(-x^2-(y-2)^2), \qquad \frac{\partial \phi}{\partial t}(x,y,0)=0.
\end{align*}

To rewrite the KGZ equations in first-order form, we first write the complex-valued function $\psi$ in terms of its real and imaginary parts as $\psi=q_1+iq_2$ and then define $p_1=\partial q_1/\partial t$ and $p_2=\partial q_2/\partial t$. 
We then discretize the two-dimensional spatial domain $\Omega$ with $n_x=n_y=400$ equally spaced grid points in both spatial directions to derive the nonlinear FOM of dimension $6n=6n_xn_y=960{,}000$ as
\begin{align*}
\dot{\q_1}&=\p_1,\\
\dot{\q_2}&=\p_2,\\
\dot{\p_1}&=\D\q_1-\q_1-\bphi\odot\q_1 - (\q_1^2+\q_2^2)\odot \q_1,\\
\dot{\p_2}&=\D\q_2-\q_2-\bphi\odot\q_2 - (\q_1^2+\q_2^2)\odot\q_2,\\
\dot{\bvarphi}&=\bphi + (\q_1^2+\q_2^2),\\
\dot{\bphi}&=\D\bvarphi,
\end{align*}
with the space-discretized energy 
\begin{equation*} 
 E(\q_1,\q_2,\p_1,\p_2,\bvarphi,\bphi)=\frac{1}{2}\y^\top \mathbf{M} \y +  \bphi^\top (\q_1^2+\q_2^2)+\sum_{i=1}^n \frac{1}{2} (q_{1,i}^2+q_{2,i}^2)^2,
 \end{equation*} 
 where $\y=[\q_1^\top,\q_2^\top,\p_1^\top,\p_2^\top,\bvarphi^\top,\bphi^\top]^\top \in \real^{6n}$ and $\mathbf{M}=\text{blkdiag}(\In-\D,\In-\D,\In,\In,-\D,\In) \in \real^{6n \times 6n}$.
\subsubsection{Model form for structure-preserving Lift \& Learn based on energy quadratization}
We follow the energy-quadratization strategy and introduce an auxiliary variable $w=q_1^2+q_2^2$ to lift the nonlinear FOM to a quadratic lifted FOM 
\begin{align*}
\dot{\q_1}&=\p_1,\\
\dot{\q_2}&=\p_2,\\
\dot{\p_1}&=\D\q_1-\q_1-\bphi\odot\q_1 - \w \odot \q_1,\\
\dot{\p_2}&=\D\q_2-\q_2-\bphi \odot\q_2 -\w \odot \q_2,\\
\dot{\bvarphi}&=\bphi + \w,\\
\dot{\bphi}&=\D\bvarphi,\\
\dot{\w}   &=2 \q_1\odot \p_1 + 2\q_2 \odot \p_2,
\end{align*}
with a quadratic FOM energy in the lifted state variables
\begin{equation*} 
 E_{\text{lift}}(\q_1,\q_2,\p_1,\p_2,\bvarphi,\bphi,\w)=\frac{1}{2}\y^\top\mathbf{M} \y +  \bphi^\top \w+ \frac{1}{2} \w^\top\w.
\end{equation*}
To approximate the lifted FOM state, we consider a block-diagonal projection matrix $\bar{\V} $ of the form
\begin{equation*}
\bar{\V} =\text{blkdiag}(\bPhi,\bPhi,\bPhi,\bPhi,\V,\V,\V) \in \real^{7n \times 7r},
\end{equation*}
where $\bPhi \in \real^{n\times r}$ is computed from the FOM snapshots of $\q_1,\q_2, \p_1,$ and $\p_2$, $\V\in \real^{n \times r}$ is computed from the FOM snapshots of $\bvarphi$, $\bphi$, and the lifted $\w$ snapshots. The model form for the nonintrusive quadratic ROM is 
\begin{align*}
\dot{\qhat}_1&=\phat_1,\\
\dot{\qhat}_2&=\phat_2,\\
\dot{\phat}_1&=\Dhat_{\q,1}\qhat_1-\qhat_1 + \Hhat_{\p} (\bphihat\otimes \qhat_1)+ \Hhat_{\p}( \what\otimes  \qhat_1),\\
\dot{\phat}_2&=\Dhat_{\q,2}\qhat_2-\qhat_2+ \Hhat_{\p} (\bphihat\otimes  \qhat_2)+ \Hhat_{\p}( \what\otimes  \qhat_2),\\
\dot{\bvarphihat}&=\bphihat + \what,\\
\dot{\bphihat}&=\Dhat_{\pmb{\varphi}}\bvarphihat,\\
\dot{\what}   &=\Hhat_{\w}\left( \qhat_1\otimes  \phat_1 +  \qhat_2 \otimes\phat_2\right),
\end{align*}
where the quadratic reduced operators $\Hhat_{\p} \in \real^{r \times r^2}$ and $\Hhat_{\w} \in \real^{r \times r^2}$ are computed analytically, and the linear reduced matrices $\Dhat_{\q,1}=\Dhat_{\q,1}^\top\in \real^{r \times r}$, $\Dhat_{\q,2}=\Dhat_{\q,2}^\top\in \real^{r \times r}$,  and $\Dhat_{\pmb \varphi}=\Dhat_{\pmb \varphi}^\top\in \real^{r \times r}$ are learned from projections of the high-dimensional data. We point out once more that all quadratic forms introduced by the lifting procedure do not have to be learned, but are known.
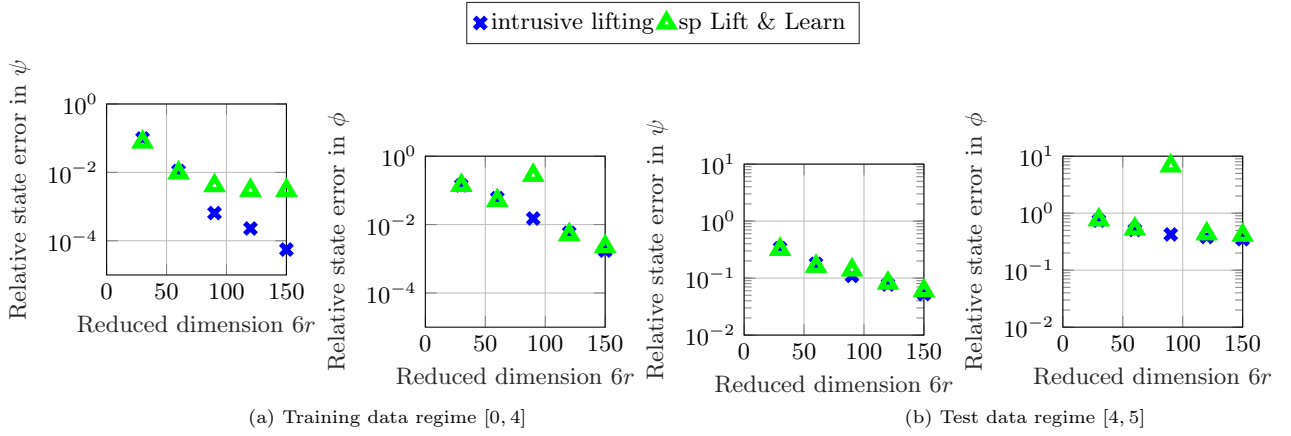
\begin{figure}[tbp]
\small
\captionsetup[subfigure]{oneside,margin={-2cm,0 cm}}
\begin{subfigure}{.23\textwidth}
       \setlength\fheight{3.5 cm}
        \setlength\fwidth{\textwidth}
\raisebox{0mm}{\input{figures/kgz/finest/q_train.tex}}
\label{fig:state_train_q}
    \end{subfigure}
    \hspace{0.25cm}
        \begin{subfigure}{.23\textwidth}
       \setlength\fheight{3.5 cm}
        \setlength\fwidth{\textwidth}
\raisebox{0mm}{\input{figures/kgz/finest/phi_train.tex}}
\caption{Training data regime $[0,4]$}
\label{fig:state_train_phi}
    \end{subfigure}
    \hspace{0.25cm}
    \begin{subfigure}{.23\textwidth}
           \setlength\fheight{3.5 cm}
           \setlength\fwidth{\textwidth}
\raisebox{0mm}{\input{figures/kgz/finest/q_test.tex}}
\label{fig:state_test_q}
    \end{subfigure}
\hspace{0.25cm}
    \begin{subfigure}{.23\textwidth}
           \setlength\fheight{3.5 cm}
           \setlength\fwidth{\textwidth}
\input{figures/kgz/finest/phi_test.tex}
\caption{Test data regime $[4,5]$}
\label{fig:state_test_phi}
    \end{subfigure}
\caption{Klein-Gordon-Zakharov equations. The relative state error comparison for $\psi$ in plot (a) shows that both approaches achieve similar accuracy up to $6r=60$. For $6r>60$, the relative state error for structure-preserving Lift \& Learn does not decrease as favorably with an increase in the reduced dimension. The relative state error comparison for $\psi$ in plot (b) shows that both
approaches achieve similar state error performance in the test data regime.  The relative state error comparison for $\phi$ in plots (a) and (b) show that, except $6r=90$ for structure-preserving Lift \& Learn, intrusive and nonintrusive approaches achieve similar accuracy in both training and test data regimes.}
 \label{fig:kgz_state}
\end{figure}
\begin{figure}[tbp]
\small
\captionsetup[subfigure]{oneside,margin={1.8cm,0 cm}}
    \begin{subfigure}{.45\textwidth}
       \setlength\fheight{6 cm}
        \setlength\fwidth{\textwidth}
\input{figures/kgz/finest/proj.tex}
\caption{Projection error in $\psi$}
\label{fig:proj_q}
    \end{subfigure}
          \hspace{0.4cm}
    \begin{subfigure}{.45 \textwidth}
           \setlength\fheight{6 cm}
           \setlength\fwidth{\textwidth}
\raisebox{-48mm}{\input{figures/kgz/finest/energy.tex}}
\caption{FOM energy error}
\label{fig:energy}
    \end{subfigure}
\caption{Klein-Gordon-Zakharov equations. The time-evolution of projection error  for $\psi$ and $\phi$ in plot  (a) shows that the basis matrix provides accurate approximations in the training data regime. However, the projection error for both field variables grows substantially outside the training data regime. Despite the projection error growth in the test data regime, the energy error comparison in plot (b) shows that both structure-preserving Lift \& Learn and intrusive lifting approaches yield ROMs with bounded energy error.}
 \label{fig:kgz_proj}
\end{figure}
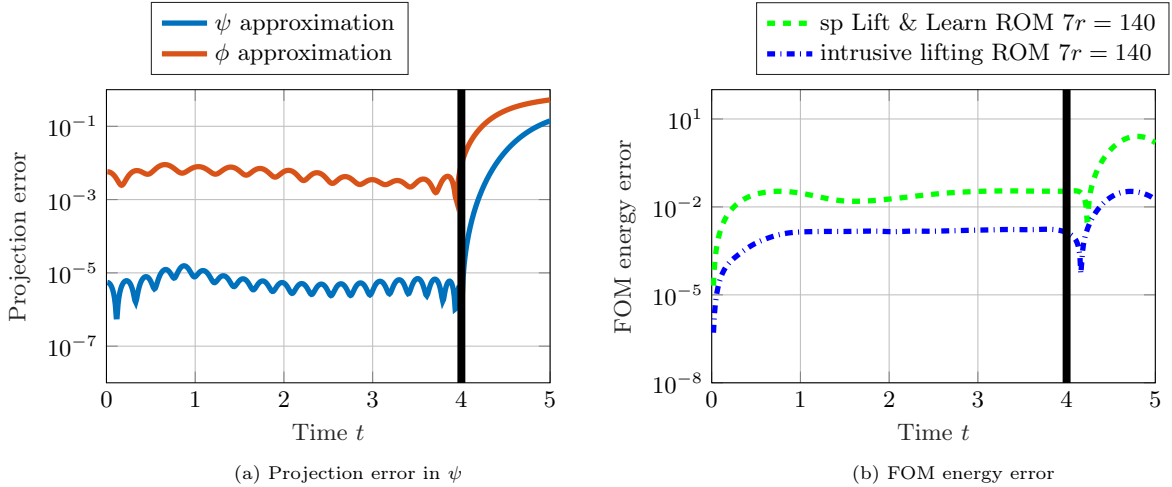
\subsubsection{Numerical results}
We build a training dataset consisting of FOM snapshots  from $t=0$ to $t=4$. 
For test dataset, we consider FOM snapshots from $t=4$ to $t=5$ (predicting 25\% outside training data) to study the generalization capability of the learned ROMs.  Due to the non-canonical nature of KGZ equations, the HOpInf with spDEIM approach is not applicable to this example. Therefore, we only consider structure-preserving Lift \& Learn ROMs and structure-preserving intrusive lifting ROMs for this numerical example.  

In Figure~\ref{fig:kgz_state}, we compare the relative state error of the structure-preserving Lift \& Learn approach against the intrusive lifting approach for the complex-valued scalar field $\psi(x,y,t)$ and the real-valued scalar field $\phi (x,y,t)$. For training data, the relative state error comparison for $\psi$ in Figure~\ref{fig:state_train_phi} shows that the quadratic ROMs obtained via structure-preserving Lift \& Learn and intrusive lifting achieve similar accuracy up to $6r=60$. For $6r>60$, we observe that the state error for $\psi$ levels off in the training data regime for the structure-preserving Lift \& Learn ROMs. In Figure~\ref{fig:state_test_phi}, both nonintrusive and intrusive ROMs exhibit similar accuracy in the test data regime for $\psi$.  The relative state error comparison for $\phi$ in Figure~\ref{fig:state_train_phi}  shows that structure-preserving Lift \& Learn ROMs and intrusive lifting ROMs achieve similar accuracy in the training data regime for all reduced dimensions except $6r=90$. For $6r=90$, we observe that the learned quadratic ROM encounters conditioning issues during numerical time integration.\footnote{We investigated adding a Tikhonov-type regularization term to the learning problem for $6r=90$ as is commonly done in unconstrained Operator Inference~\cite{mcquarrie2021data}. However, this did not lead to improved accuracy or stability in our numerical experiments. Since the proposed structure-preserving Lift \& Learn formulation enforces hard constraints derived from the underlying physics, it already provides physics-based regularization and additional regularization may therefore be ineffective.} In Figure~\ref{fig:state_test_phi}, we observe that both approaches fail to achieve relative state error below $10^{-1}$ for $\phi$ in the test data regime. 

In Figure~\ref{fig:proj_q}, we show the time evolution of the projection error in $\psi$ and $\phi$, which illustrates that the basis matrix provides accurate approximations with relative error below $10^{-2}$ for both field variables in the training data regime. However, the basis matrix yields substantially higher projection error in the test data regime with relative error higher than $10^{-1}$ for both $\psi$ and $\phi$ at $t=5$. This growth in the projection error is due to the transport-dominated nature of the problem. Notwithstanding the substantial projection error growth outside the training data regime, the energy error comparison in Figure~\ref{fig:energy} shows that both structure-preserving Lift \& Learn and intrusive lifting yield quadratic ROMs with bounded FOM energy error in the training data regime. 

We compare the FOM solution and the structure-preserving Lift \& Learn ROM solution, and also include the corresponding absolute pointwise error, for $\psi(x,y,t)$ and $\phi(x,y,t)$ in Figure~\ref{fig:kgz_psi} and Figure~\ref{fig:kgz_phi}, respectively. Figure~\ref{fig:kgz_psi} shows that the structure-preserving Lift \& Learn ROM of dimension $7r=140$ accurately captures the time-evolution of the complex-valued scalar field $\psi(x,y,t)$ and provides accurate prediction at $t=4.5$, which is 12.5\% outside the training time interval. Figure~\ref{fig:kgz_phi} shows that the structure-preserving Lift \& Learn ROM of dimension $7r=140$ provides accurate approximation of the scalar solution field $\phi(x,y,t)$ over the two-dimensional computational domain at $t=1.5$, $t=3$, and $t=4.5$. In terms of computational efficiency, numerical time integration of the structure-preserving quadratic ROM of size $7r=140$ requires approximately $1.48$~\si{\s} compared to the approximate FOM run time of $89$~\si{\min}, which is a factor of $3608\times$ speedup. Thus, the quadratic ROM learned via the proposed method provides accurate and stable predictions at a substantially lower computational cost.

\begin{figure}[tbp]
\centering
\newcolumntype{C}[1]{>{\centering\arraybackslash}m{#1}}

\setlength{\tabcolsep}{8pt}
\renewcommand{\arraystretch}{1.0}

\begin{tabular}{C{0.2\textwidth} C{0.23\textwidth} C{0.23\textwidth} C{0.23\textwidth}}

\makecell{\small Conservative\\ \small Nonlinear FOM} &

\begin{minipage}[b]{\linewidth}
  \centering
  \includegraphics[width=\linewidth]{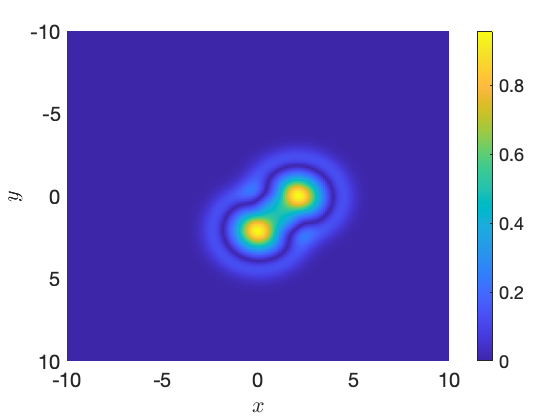}
\end{minipage} &

\begin{minipage}[b]{\linewidth}
  \centering
  \includegraphics[width=\linewidth]{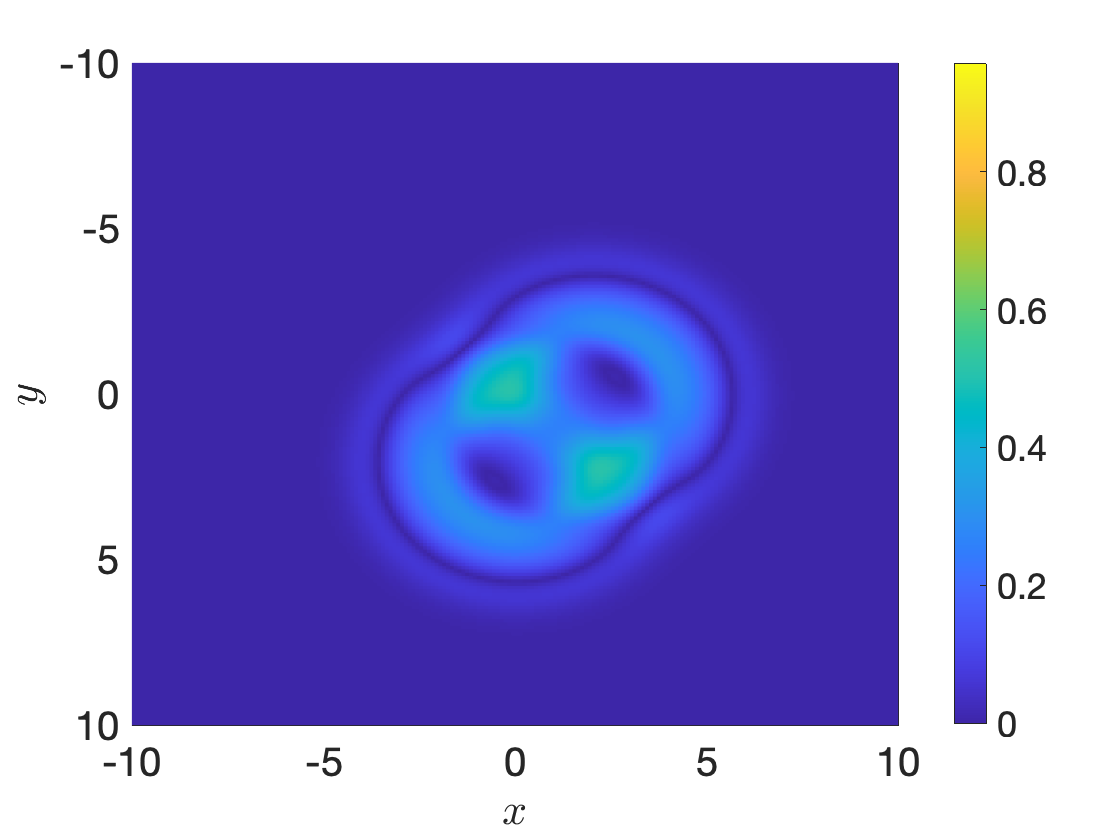}
\end{minipage} &

\begin{minipage}[b]{\linewidth}
  \centering
  \includegraphics[width=\linewidth]{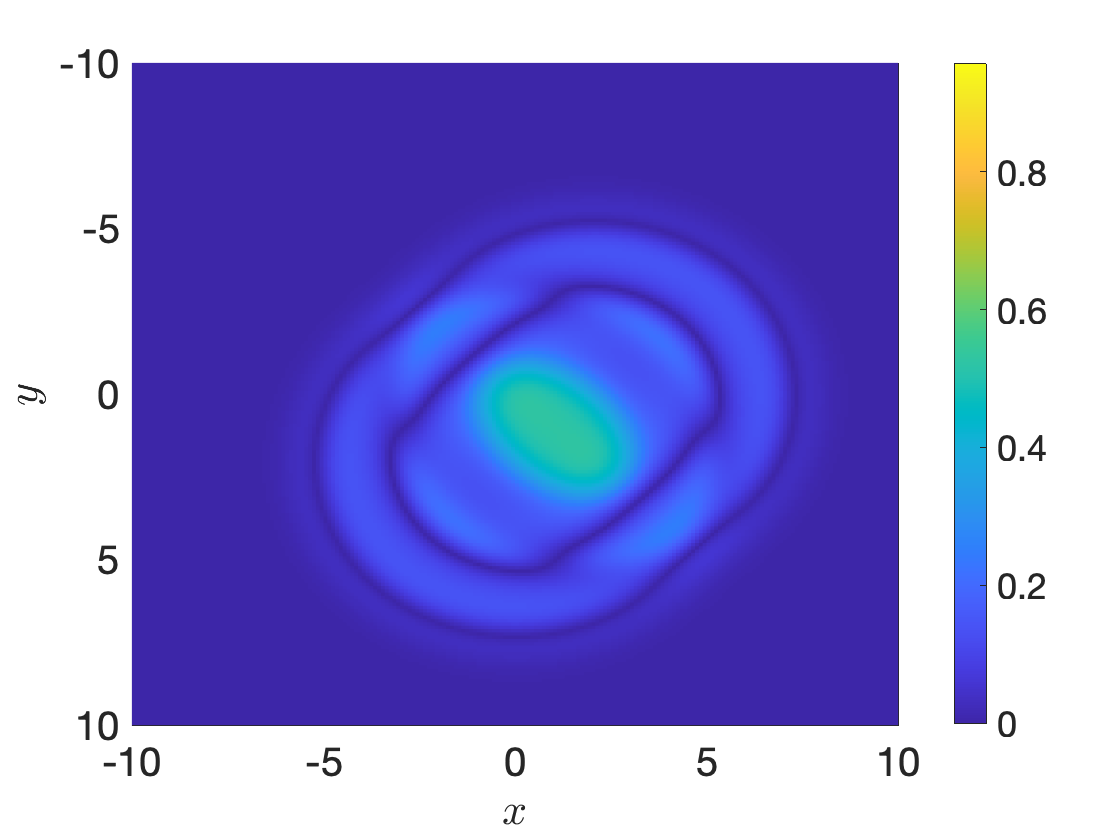}
\end{minipage}
\\[0.6em] 

\makecell{\small Structure-preserving\\ \small Lift \& Learn ROM} &

\begin{minipage}[b]{\linewidth}
  \centering
  \includegraphics[width=\linewidth]{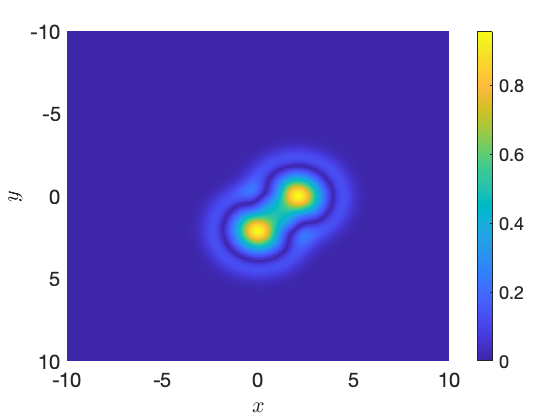}
\end{minipage} &

\begin{minipage}[b]{\linewidth}
  \centering
  \includegraphics[width=\linewidth]{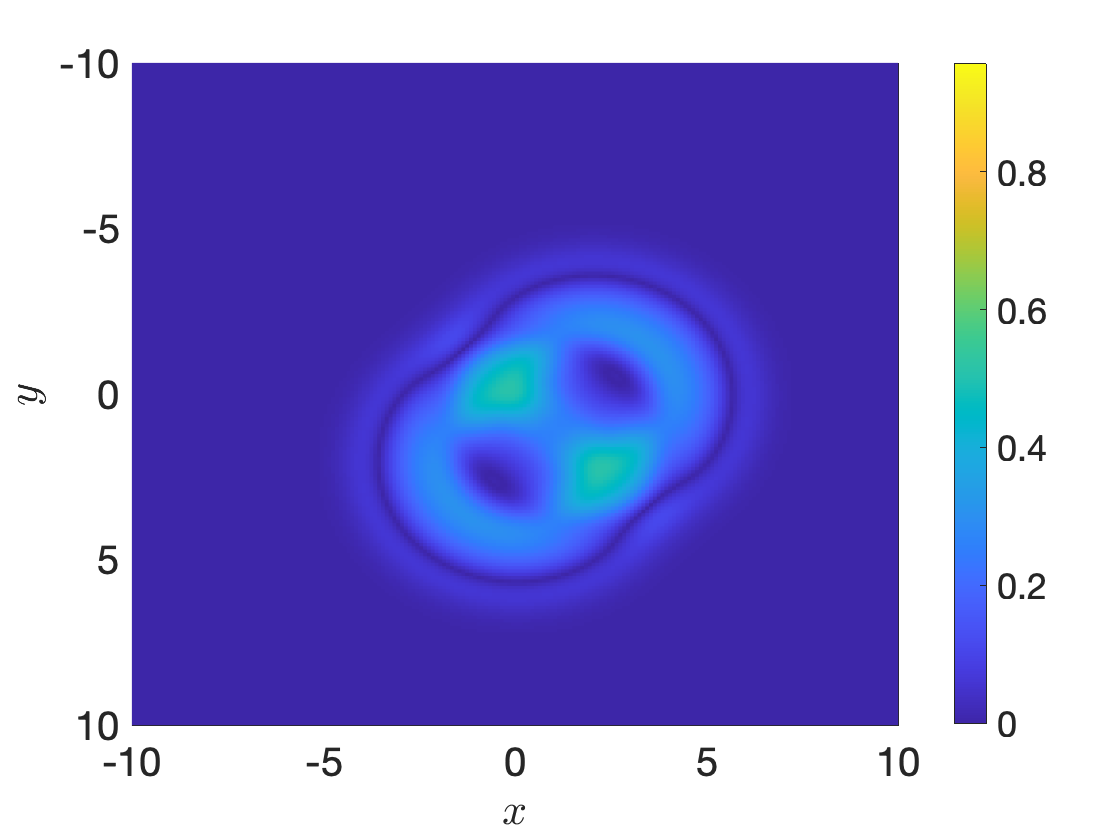}
\end{minipage} &

\begin{minipage}[b]{\linewidth}
  \centering
  \includegraphics[width=\linewidth]{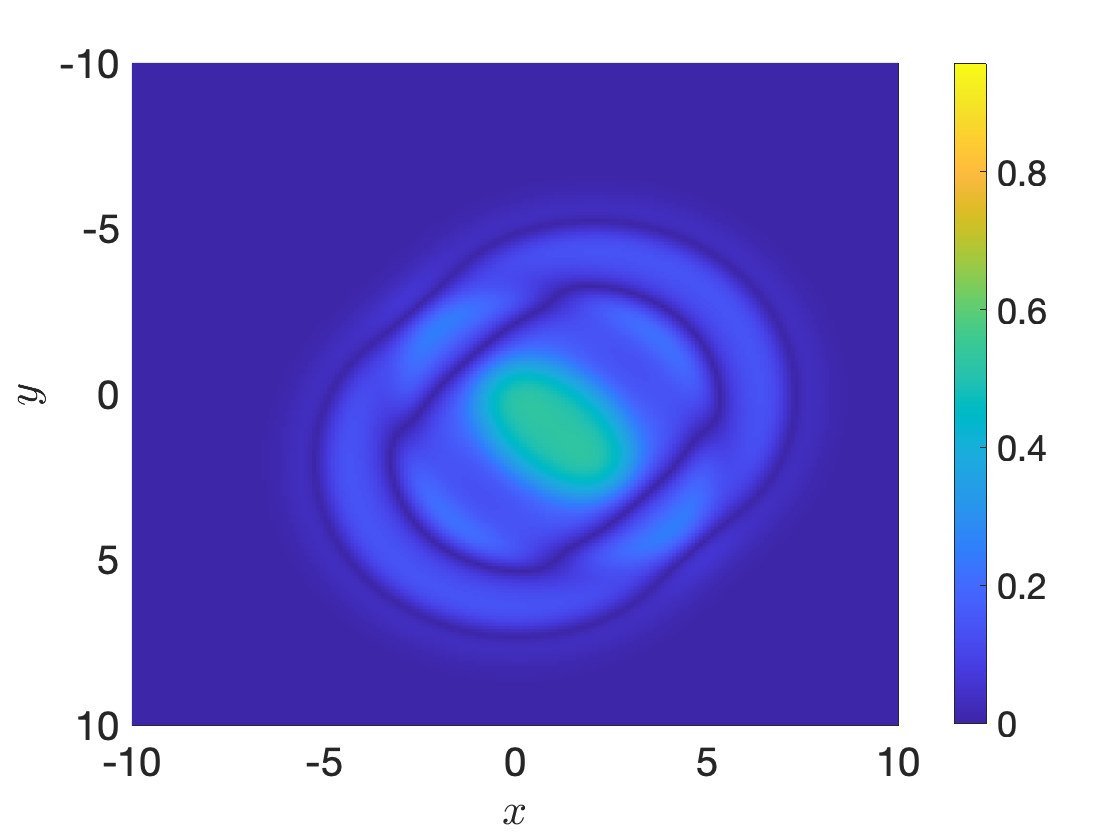}
\end{minipage}
\\[0.6em] 
\makecell{\small Absolute pointwise error \\ \small  for sp Lift \& Learn ROM} &

\begin{minipage}[b]{\linewidth}
  \centering
  \includegraphics[width=\linewidth]{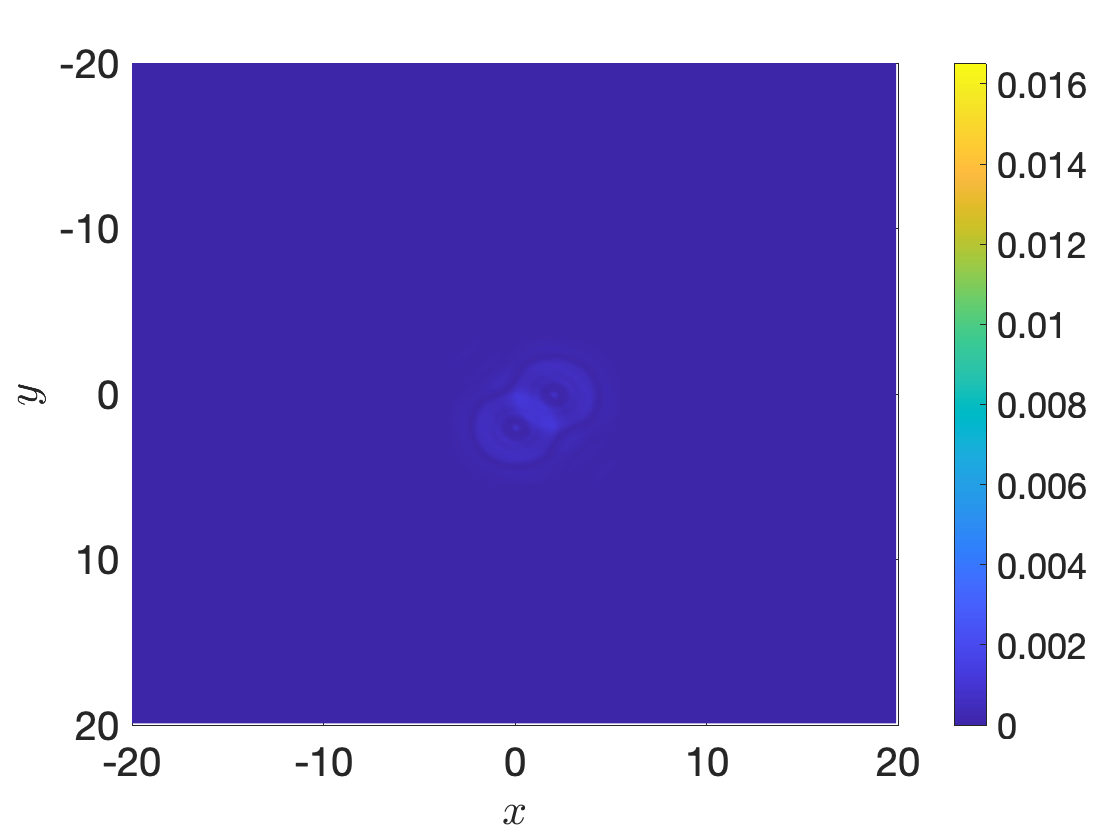}
  \subcaption{$t=1.5$}
\end{minipage} &

\begin{minipage}[b]{\linewidth}
  \centering
  \includegraphics[width=\linewidth]{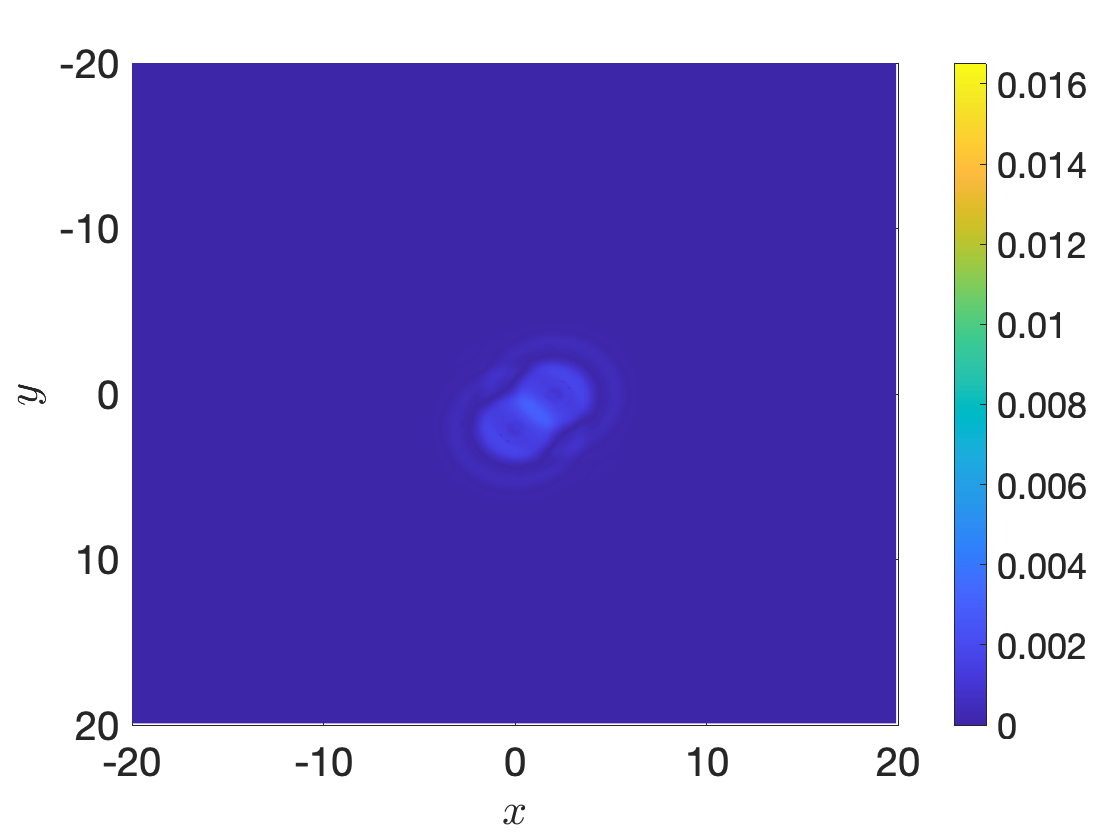}
  \subcaption{$t=3$}
\end{minipage} &

\begin{minipage}[b]{\linewidth}
  \centering
  \includegraphics[width=\linewidth]{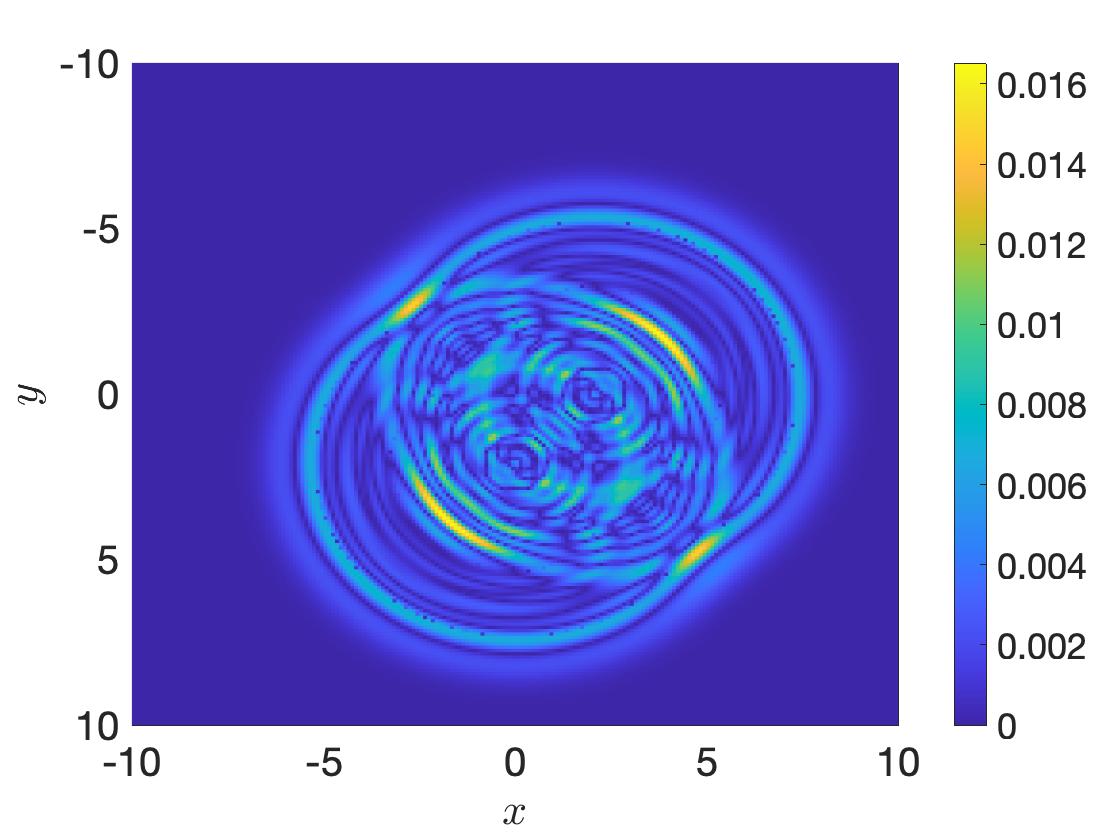}
  \subcaption{$t=4.5$}
\end{minipage}
\\
\end{tabular}
    \caption{Klein-Gordon-Zakharov equations. Plots compare the absolute value of the complex-valued field $\psi(x,y,t)$ at selected time instances $t\in \{1.5, 3, 4.5\}$ for the nonlinear conservative FOM (top row) and the structure-preserving quadratic ROM of dimension $7r=140$ (middle row).  The bottom row shows the absolute pointwise error between the structure-preserving Lift \& Learn ROM and the conservative FOM. The data-driven quadratic ROM learned via structure-preserving Lift \& Learn provides accurate approximate solutions for $|\psi(x,y,t)|$ at all three time instances. }
\label{fig:kgz_psi}
\end{figure}
\begin{figure}[tbp]
\centering
\newcolumntype{C}[1]{>{\centering\arraybackslash}m{#1}}

\setlength{\tabcolsep}{8pt}
\renewcommand{\arraystretch}{1.0}

\begin{tabular}{C{0.2\textwidth} C{0.23\textwidth} C{0.23\textwidth} C{0.23\textwidth}}

\makecell{\small Conservative\\ \small Nonlinear FOM} &

\begin{minipage}[b]{\linewidth}
  \centering
  \includegraphics[width=\linewidth]{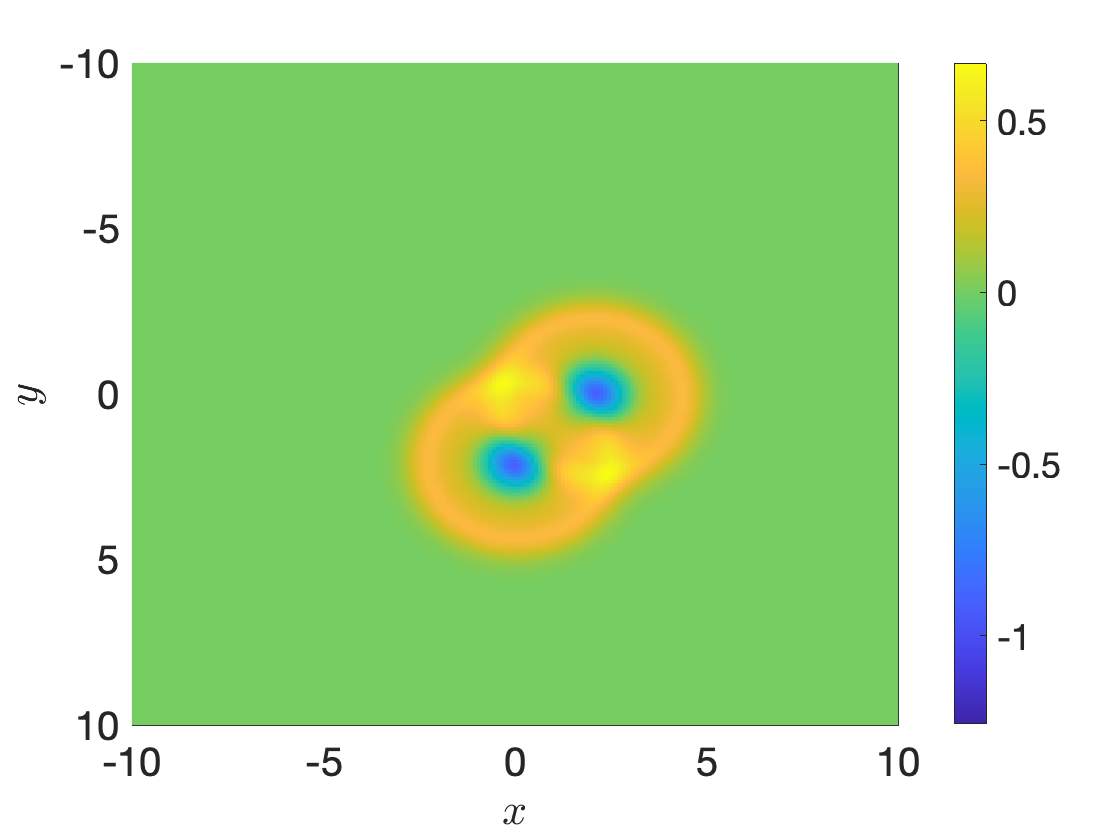}
\end{minipage} &

\begin{minipage}[b]{\linewidth}
  \centering
  \includegraphics[width=\linewidth]{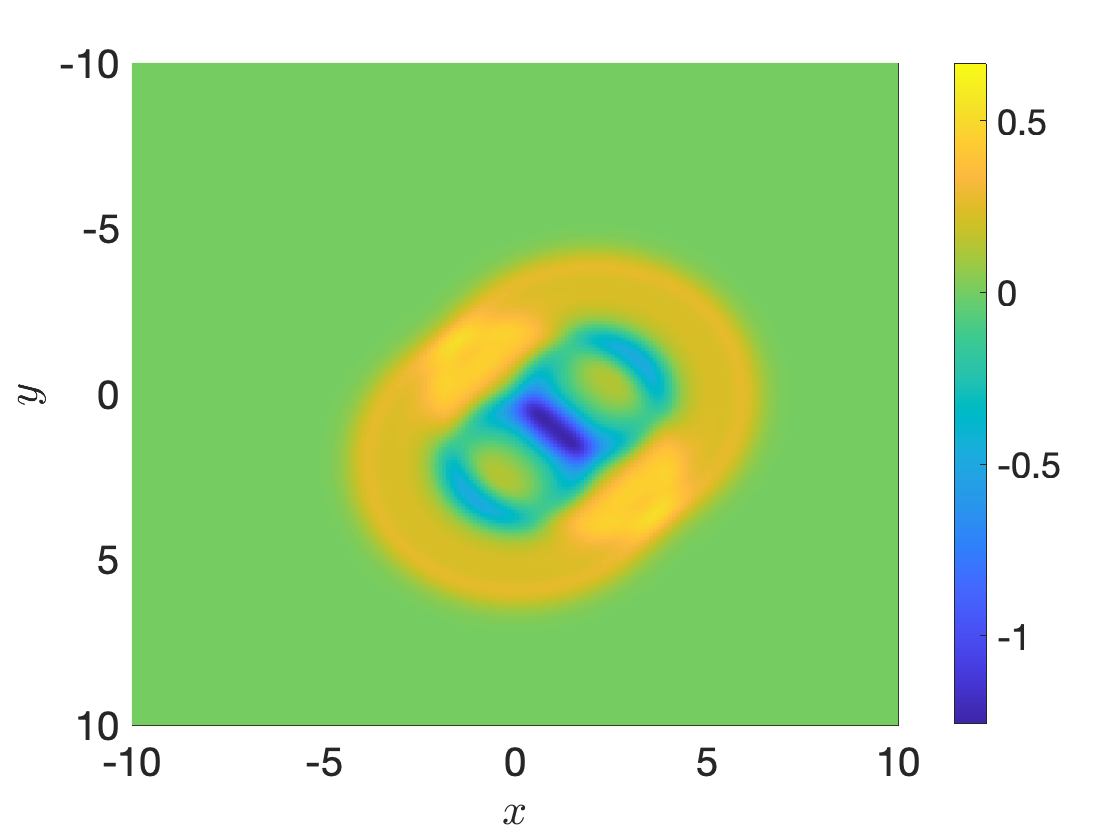}
\end{minipage} &

\begin{minipage}[b]{\linewidth}
  \centering
  \includegraphics[width=\linewidth]{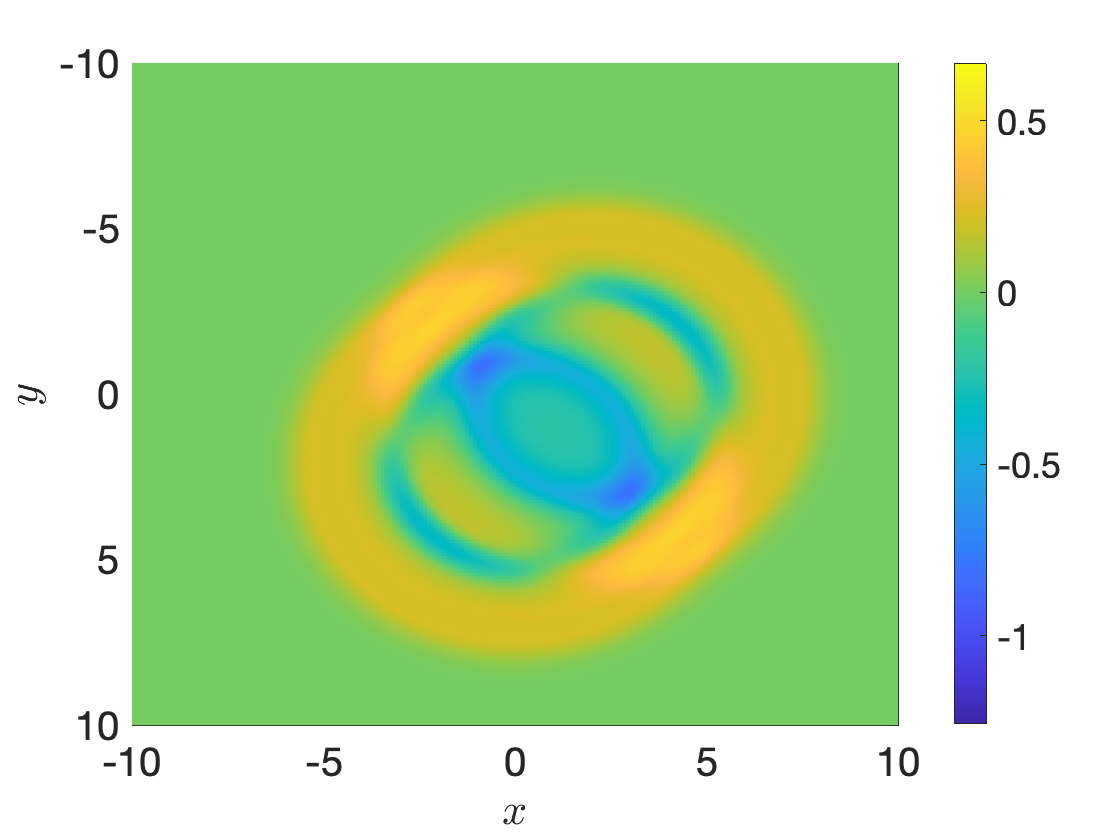}
\end{minipage}
\\[0.6em] 

\makecell{\small Structure-preserving\\ \small Lift \& Learn ROM} &

\begin{minipage}[b]{\linewidth}
  \centering
  \includegraphics[width=\linewidth]{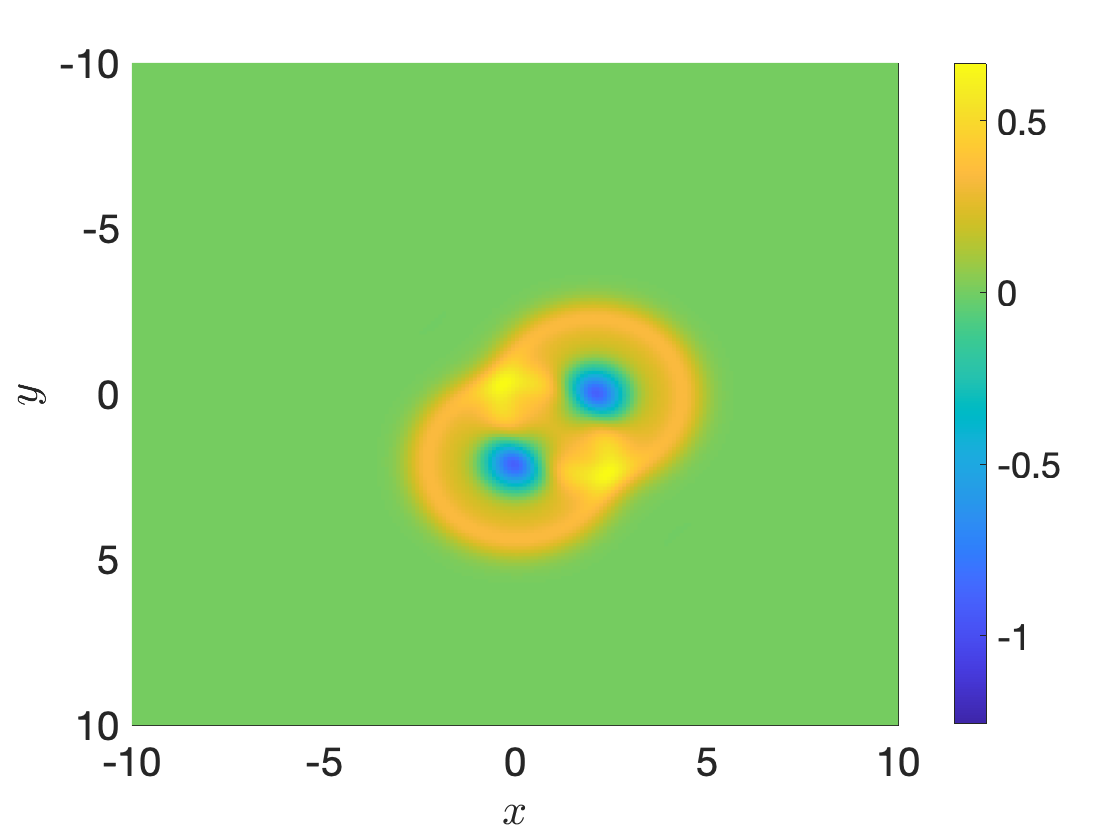}
\end{minipage} &

\begin{minipage}[b]{\linewidth}
  \centering
  \includegraphics[width=\linewidth]{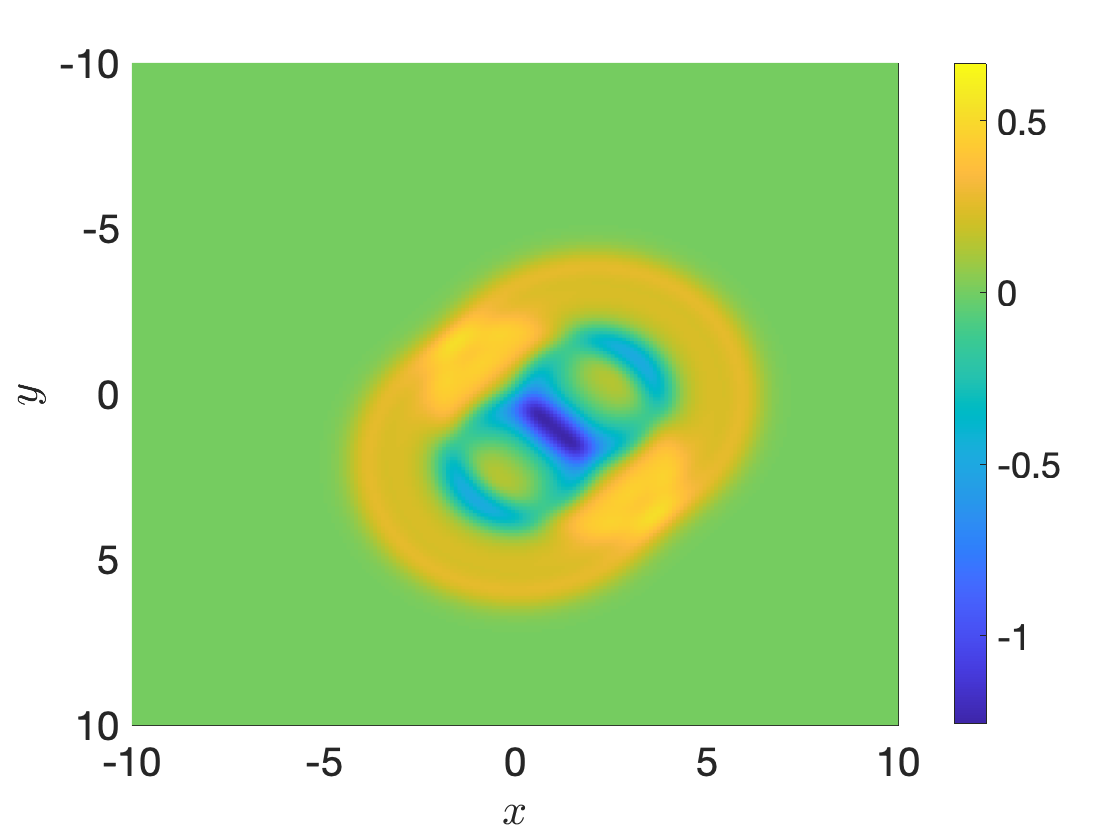}
\end{minipage} &

\begin{minipage}[b]{\linewidth}
  \centering
  \includegraphics[width=\linewidth]{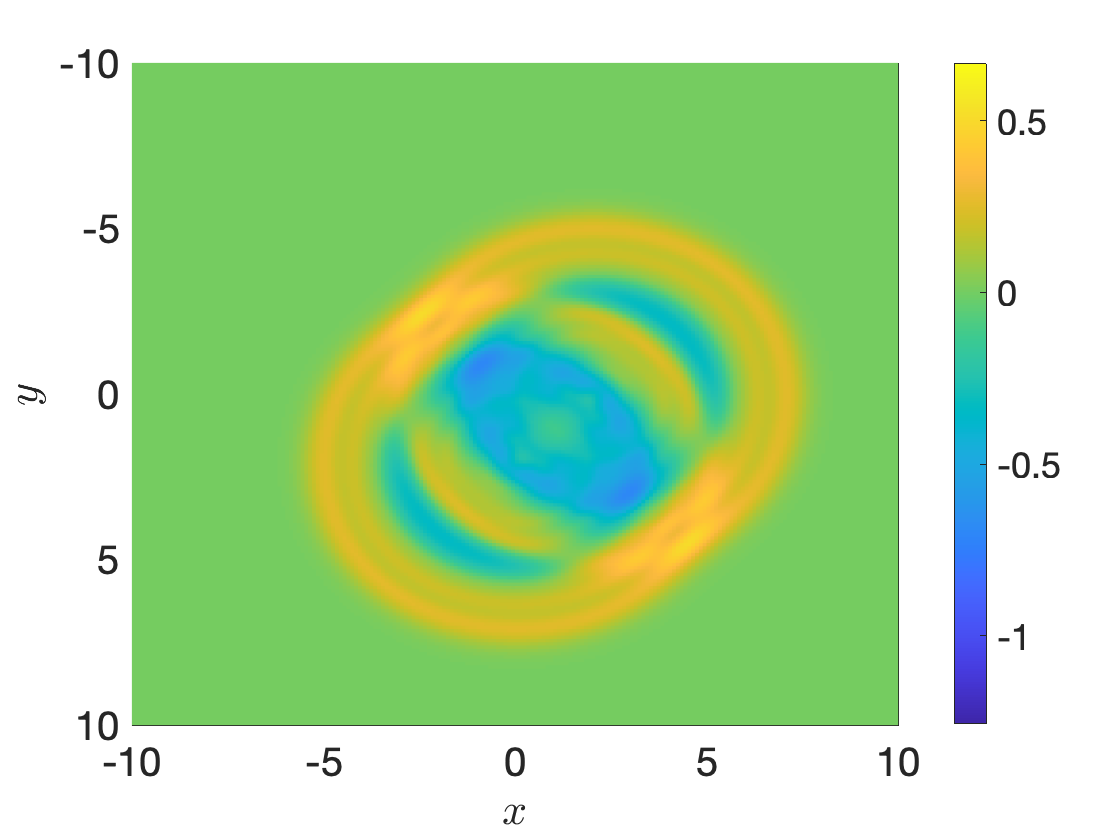}
\end{minipage}
\\[0.6em] 
\makecell{\small  Absolute pointwise error\\ \small  for sp Lift \& Learn ROM} &

\begin{minipage}[b]{\linewidth}
  \centering
  \includegraphics[width=\linewidth]{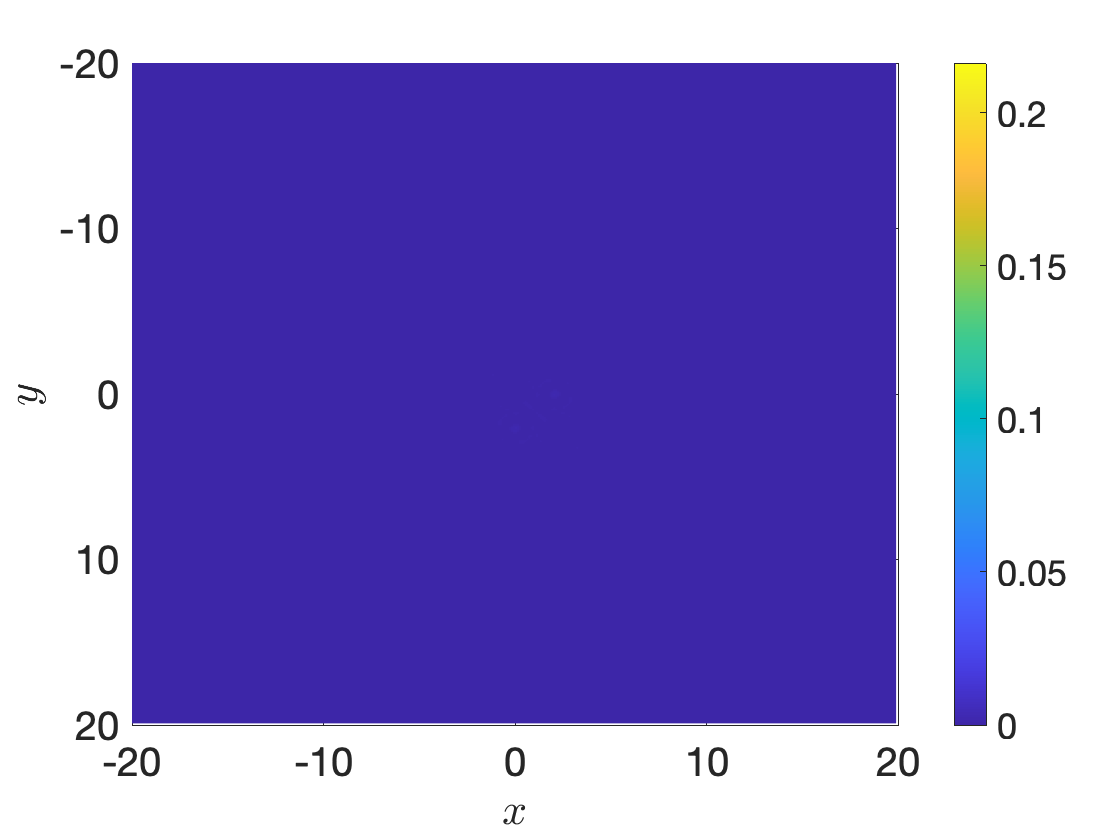}
  \subcaption{$t=1.5$}
\end{minipage} &

\begin{minipage}[b]{\linewidth}
  \centering
  \includegraphics[width=\linewidth]{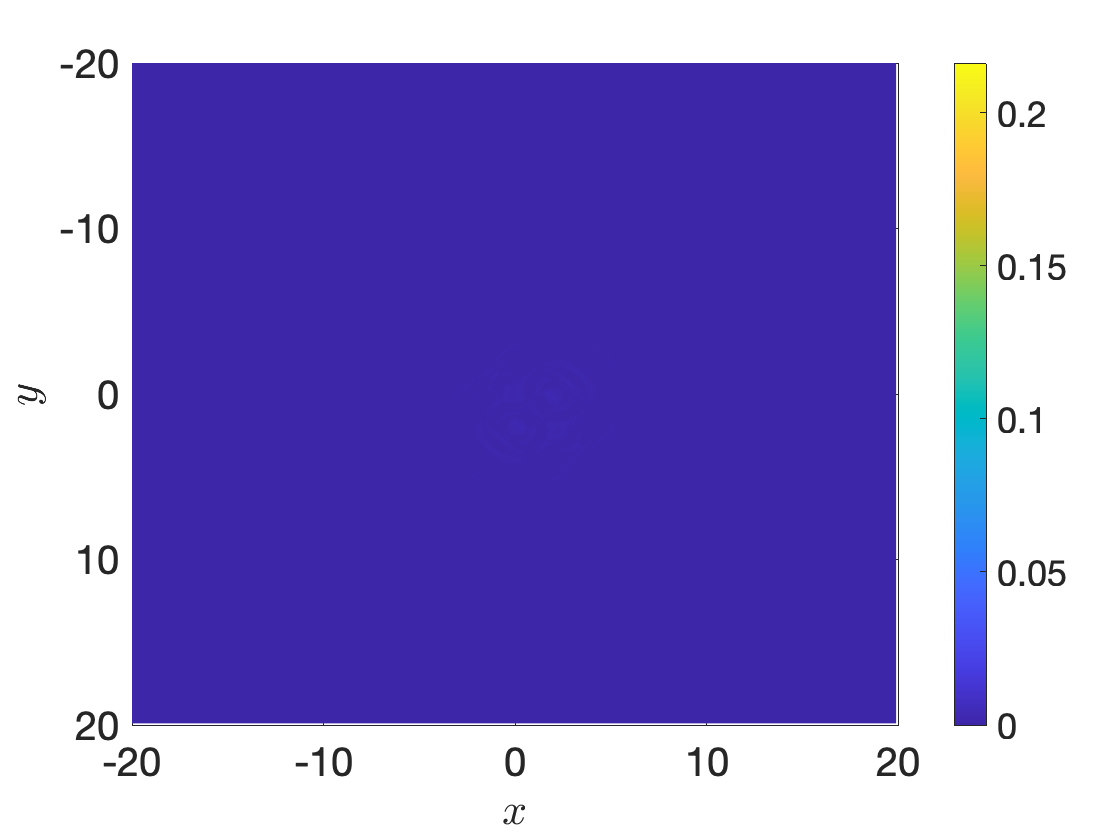}
  \subcaption{$t=3$}
\end{minipage} &

\begin{minipage}[b]{\linewidth}
  \centering
  \includegraphics[width=\linewidth]{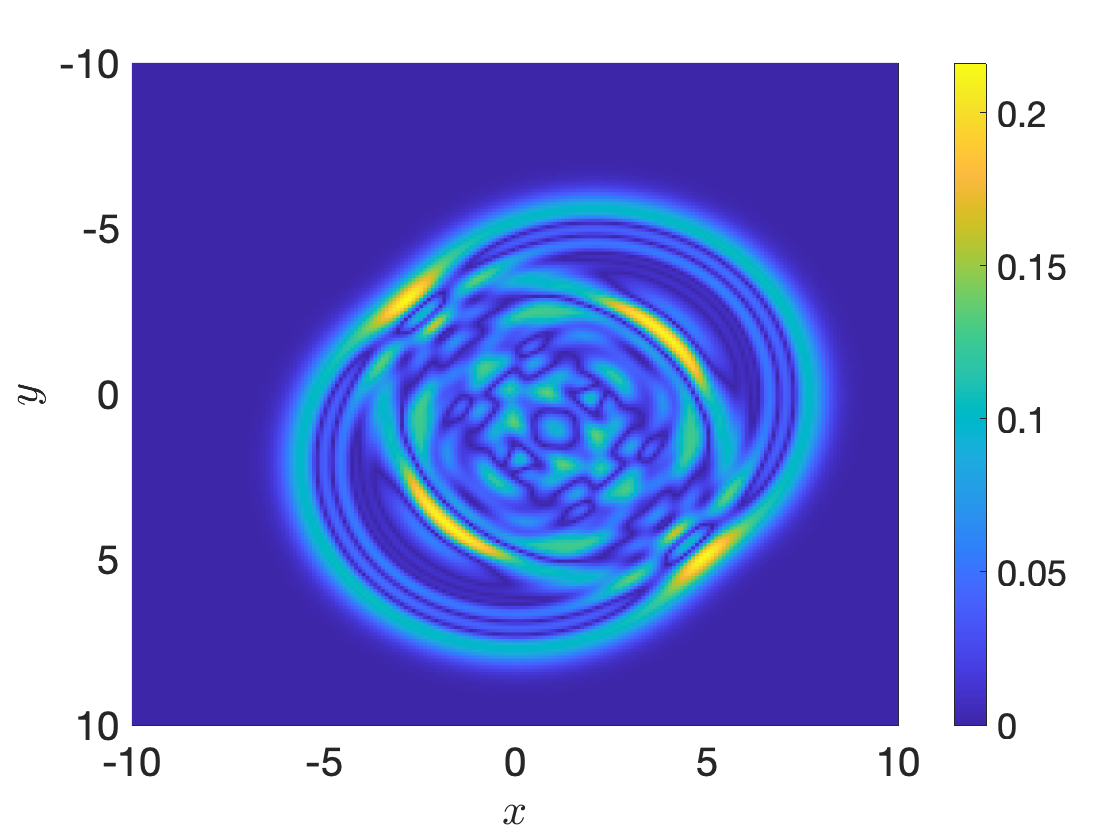}
  \subcaption{$t=4.5$}
\end{minipage}
\\
\end{tabular}
    \caption{Klein-Gordon-Zakharov equations. Plots compare the evolution of the scalar field $\phi(x,y,t)$ at selected time instances $t\in \{1.5, 3, 4.5\}$ for the nonlinear conservative FOM (top row) and the structure-preserving quadratic ROM of dimension $7r=140$ (middle row). The bottom row shows the absolute pointwise error between the structure-preserving Lift \& Learn ROM and the conservative FOM. The proposed structure-preserving Lift \& Learn method provides accurate approximation of $\phi(x,y,t)$ solution over the two-dimensional computational domain.}
\label{fig:kgz_phi}
\end{figure}
\section{Conclusions}
\label{sec:conclusion}
We have presented \textit{structure-preserving Lift \& Learn}, a nonintrusive model reduction method which uses lifting transformations to learn low-dimensional quadratic ROMs that respect the conservative nature of the high-dimensional problem. In contrast to the standard Lift \& Learn approach~\cite{qian2020lift} that was designed to match the algebraic structure of the lifted model (here: quadratic), the proposed structure-preserving Lift \& Learn approach enforces additional conservative properties in the learning process. It does this by leveraging prior knowledge of the conservative PDE through the energy-quadratization strategy, which enables the derivation of quadratic reduced terms analytically. The remaining linear reduced operators are then learned via a constrained optimization problem. We presented a theoretical result that shows the learned quadratic ROMs conserve a perturbed lifted FOM energy exactly.

Numerical experiments on two canonical Hamiltonian PDEs demonstrate the ability of the proposed approach to learn structure-preserving ROMs that achieve accuracy and computationally efficiency similar to nonintrusive Hamiltonian Operator Inference~\cite{sharma2022hamiltonian} ROMs with spDEIM~\cite{pagliantini2023gradient}. The numerical results also show that the proposed structure-preserving Lift \& Learn approach yields generalizable quadratic ROMs that provide accurate predictions with bounded energy error outside the training data regime. The final numerical example with the Klein-Gordon-Zakharov equations shows the proposed approach can be flexibly adapted to learn structure-preserving ROMs for a wider class of coupled conservative PDEs that do not have a canonical Hamiltonian formulation.

Future research directions motivated by this work are: automating the process of deriving quadratic reduced terms in structure-preserving Lift \& Learn via the optimal quadratization framework~\cite{bychkov2024exact}; studying the long-term boundedness of nonintrusive quadratic ROMs using the theoretical framework from~\cite{schlegel2015long}; deriving error bounds for the difference between nonlinear FOM energy approximation obtained via intrusive and nonintrusive quadratic ROMs; and combining the proposed method with the probabilistic learning framework in~\cite{galioto2024bayesian} to use it for learning structure-preserving ROMs from noisy data.
\section*{Acknowledgments}
 H. Sharma and B. Kramer were in part financially supported by the Applied and Computational Analysis Program of the Office of Naval Research under award N000142212624 and the National Science Foundation under award NSF-CMMI 2144023. We thank Iman Adibnazari for his help with running numerical experiments for the two-dimensional Klein-Gordon-Zakharov equations on the Triton Shared Computing Cluster (TSCC) allocation supported by ONR DURIP Award No. N000142412222.  
\bibliographystyle{vancouver}
\bibliography{main}
\end{document}

%% file: figures/motivation/q_train.tex
%
%
\begin{tikzpicture}

\begin{axis}[%
width=0.983\fheight,
height=0.65\fheight,
at={(0\fheight,0\fheight)},
scale only axis,
xmin=0,
xmax=20,
xlabel style={font=\color{white!15!black}},
xlabel={Reduced dimension $2r$},
ymode=log,
ymin=1e-05,
ymax=1,
yminorticks=true,
ylabel style={font=\color{white!15!black}},
ylabel={Relative state error in $q$},
axis background/.style={fill=white},
xmajorgrids,
ymajorgrids,
legend style={at={(0.25,1.15)}, anchor=south west, legend cell align=left, align=left, draw=white!15!black}
]
\addplot [color=magenta, dashed, line width=2.0pt, mark size=3.0pt, mark=x, mark options={solid, magenta}]
  table[row sep=crcr]{%
2	0.0973250003216301\\
4	0.0166129207910852\\
6	0.00339195010394367\\
8	0.000806900615702213\\
10	0.000442085995030091\\
12	0.000320273447389573\\
14	0.000207188663860293\\
16	0.000169415065049292\\
18	0.000127493779381865\\
20	6.73850108377361e-05\\
};
\addlegendentry{standard Lift \& Learn}


\end{axis}

\begin{axis}[%
width=1.273\fheight,
height=0.955\fheight,
at={(-0.169\fheight,-0.133\fheight)},
scale only axis,
xmin=0,
xmax=1,
ymin=0,
ymax=1,
axis line style={draw=none},
ticks=none,
axis x line*=bottom,
axis y line*=left
]
\end{axis}
\end{tikzpicture}%

%% file: figures/motivation/fom_energy.tex
%
%
\definecolor{mycolor1}{rgb}{0.92900,0.69400,0.12500}%
\definecolor{mycolor2}{rgb}{1.00000,0.00000,1.00000}%
\begin{tikzpicture}

\begin{axis}[%
width=0.983\fheight,
height=0.65\fheight,
at={(0\fheight,0\fheight)},
scale only axis,
xmin=-0,
xmax=30,
xlabel style={font=\color{white!15!black}},
xlabel={Time $t$},
ymode=log,
ymin=1e-8,
ymax=1e5,
yminorticks=true,
ylabel style={font=\color{white!15!black}},
ylabel={FOM energy error},
axis background/.style={fill=white},
xmajorgrids,
ymajorgrids,
yminorgrids,
legend style={at={(-0.05,1.05)}, anchor=south west, legend cell align=left, align=left, draw=white!15!black}
]
\addplot [color=mycolor1,dotted, line width=2.0pt]
  table[row sep=crcr]{%
0.0100000000000016	3.39772694673002e-05\\
0.0199999999999996	7.53799508572682e-05\\
0.0300000000000011	8.85093582780878e-05\\
0.0350000000000001	8.68650406857798e-05\\
0.0399999999999991	8.08826198550603e-05\\
0.0500000000000007	5.89070936740653e-05\\
0.0599999999999987	2.79663180862143e-05\\
0.0700000000000003	7.4937003233799e-06\\
0.0799999999999983	4.38786759815455e-05\\
0.0899999999999999	7.83614710115898e-05\\
0.105	0.000121999481606849\\
0.120000000000001	0.000151915753221488\\
0.129999999999999	0.000163026845356739\\
0.140000000000001	0.00016681606471991\\
0.145	0.000166000488363238\\
0.155000000000001	0.000159168236706364\\
0.164999999999999	0.000145853973847921\\
0.175000000000001	0.000126752413649456\\
0.184999999999999	0.000102688347169533\\
0.195	7.45735232925426e-05\\
0.204999999999998	4.33684655973822e-05\\
0.210000000000001	2.69122970877334e-05\\
0.219999999999999	7.10132897552283e-06\\
0.23	4.17977879436648e-05\\
0.239999999999998	7.62780704945951e-05\\
0.254999999999999	0.000125791549919541\\
0.27	0.000170495895034151\\
0.285	0.000208396251508701\\
0.300000000000001	0.000238006129734459\\
0.315000000000001	0.000258354167812057\\
0.329999999999998	0.000268966192280117\\
0.34	0.000270607471790641\\
0.350000000000001	0.000268005842045228\\
0.359999999999999	0.000261333113087403\\
0.375	0.000244222413115835\\
0.390000000000001	0.000219515963067351\\
0.405000000000001	0.000188444427877243\\
0.420000000000002	0.000152380832744825\\
0.434999999999999	0.000112775885071414\\
0.449999999999999	7.10986617349363e-05\\
0.460000000000001	4.28739843556513e-05\\
0.465	2.87840508121917e-05\\
0.469999999999999	1.47749041161659e-05\\
0.475000000000001	8.94160379516505e-07\\
0.48	1.28121376548051e-05\\
0.489999999999998	3.95257829381988e-05\\
0.504999999999999	7.72393797632272e-05\\
0.52	0.000111258907551814\\
0.535	0.000140770615445263\\
0.555	0.000172041886941087\\
0.57	0.000188894938020213\\
0.585000000000001	0.000199839620378839\\
0.600000000000001	0.00020483113964076\\
0.609999999999999	0.0002048988580583\\
0.620000000000001	0.000202431892432742\\
0.635000000000002	0.000194183095217681\\
0.649999999999999	0.000180835480449559\\
0.664999999999999	0.000162875304894782\\
0.68	0.000140859770894507\\
0.695	0.000115397391525107\\
0.710000000000001	8.71292070172558e-05\\
0.725000000000001	5.67112343077269e-05\\
0.734999999999999	3.55621890690827e-05\\
0.745000000000001	1.3939916811978e-05\\
0.75	3.00969703914688e-06\\
0.765000000000001	2.99862453090327e-05\\
0.780000000000001	6.28381327345549e-05\\
0.795000000000002	9.5002533456068e-05\\
0.815000000000001	0.00013597645941843\\
0.835000000000001	0.0001738516905192\\
0.855	0.000207812574686272\\
0.879999999999999	0.000243836556791167\\
0.905000000000001	0.000272013710286956\\
0.925000000000001	0.000288668611759134\\
0.945	0.000300098669072212\\
0.965	0.000306466865922061\\
0.984999999999999	0.000308060660913156\\
1.005	0.000305269815839892\\
1.03	0.000296336743986102\\
1.055	0.000282330115524587\\
1.085	0.000260402629052691\\
1.12	0.000230247574663167\\
1.16	0.000193398220957874\\
1.21	0.0001494230939727\\
1.305	9.02333373034024e-05\\
1.33	8.19515923069504e-05\\
1.35	7.76632892467432e-05\\
1.37	7.54101069304626e-05\\
1.385	7.50083962230974e-05\\
1.4	7.56596841711146e-05\\
1.42	7.80674448748183e-05\\
1.445	8.33283655765625e-05\\
1.48	9.42630384344054e-05\\
1.61	0.000154196855547184\\
1.655	0.000175310871512124\\
1.695	0.000191768070251896\\
1.735	0.000205126092382102\\
1.77	0.000213806564303809\\
1.805	0.000219456242987803\\
1.84	0.00022202800043658\\
1.88	0.000221328175328495\\
1.92	0.000217077702501456\\
1.96	0.000209780530170179\\
2.005	0.000198705634289808\\
2.06	0.000182443196877103\\
2.13	0.000160150388452962\\
2.27	0.000122603799056264\\
2.315	0.000114937170437202\\
2.355	0.000110383657030866\\
2.39	0.000108187520694969\\
2.425	0.000107628311710073\\
2.46	0.000108621916012908\\
2.5	0.000111493394101103\\
2.545	0.000116639892269177\\
2.61	0.000126835526074843\\
2.79	0.000161668407301363\\
2.86	0.000173852943120778\\
2.925	0.000183063825289765\\
2.99	0.000189809345940261\\
3.06	0.000194142723827895\\
3.13	0.000195524844414763\\
3.205	0.000194063272729749\\
3.29	0.000189348330437155\\
3.38	0.000181647058994372\\
3.48	0.000170805507792693\\
3.59	0.000157092494021071\\
3.705	0.00014157684618965\\
3.82	0.0001255110502143\\
3.945	0.000108235713256022\\
4.245	7.50852275075432e-05\\
4.31	7.09743888620553e-05\\
4.365	6.86953288067115e-05\\
4.415	6.76590774020269e-05\\
4.465	6.76599188551563e-05\\
4.515	6.87260570373383e-05\\
4.565	7.08633876769228e-05\\
4.62	7.44301327841867e-05\\
4.685	8.02096498802027e-05\\
4.77	9.00712241076458e-05\\
5.07	0.000137658655987138\\
5.165	0.000153471916704007\\
5.255	0.000167272890493563\\
5.345	0.000179261512172957\\
5.435	0.000188957551927161\\
5.525	0.000196063911104716\\
5.62	0.000200622378329739\\
5.715	0.000202194807485512\\
5.815	0.000200816631561906\\
5.92	0.000196360800728712\\
6.03	0.000188863938916484\\
6.145	0.000178542458840524\\
6.27	0.000165274342961652\\
6.415	0.00014856001295982\\
6.78	0.000112442997007634\\
6.87	0.000107320125053434\\
6.95	0.000104549059059877\\
7.025	0.000103547121653947\\
7.1	0.000104076846113798\\
7.18	0.000106235189278436\\
7.27	0.000110409987044591\\
7.385	0.000117906156319236\\
7.915	0.000163050124825758\\
8.075	0.000175046160038051\\
8.23	0.000184588942474306\\
8.375	0.000191017760286626\\
8.515	0.000194372624960692\\
8.655	0.000194727680579377\\
8.825	0.000191935418499156\\
9.095	0.000187280715280735\\
9.21	0.000188611697510054\\
9.315	0.000192759859189095\\
9.435	0.000200894732714118\\
9.7	0.00022184657673563\\
9.775	0.000223658228208023\\
9.835	0.000221771856203646\\
9.885	0.000217118709016973\\
9.93	0.000209949508061413\\
9.97	0.000200769222246323\\
10.01	0.000188543721534984\\
10.045	0.000175029886477773\\
10.075	0.00016112976159377\\
10.105	0.000144897091203689\\
10.135	0.000126142747447488\\
10.16	0.00010844790476483\\
10.185	8.87521295567236e-05\\
10.205	7.14760438000892e-05\\
10.225	5.27831138157354e-05\\
10.24	3.77962704760648e-05\\
10.25	2.7329764364481e-05\\
10.26	1.64737736838561e-05\\
10.265	1.08972216565917e-05\\
10.27	5.22029095861855e-06\\
10.275	5.58028901309628e-07\\
10.28	6.43875202399615e-06\\
10.29	1.85114844413192e-05\\
10.305	3.74141838648255e-05\\
10.325	6.4142657848265e-05\\
10.345	9.26753340309006e-05\\
10.37	0.000130982075901898\\
10.4	0.000181002214503677\\
10.435	0.000245247944718584\\
10.475	0.000326913626793156\\
10.52	0.000430011037209965\\
10.57	0.000559535072071072\\
10.63	0.000737483149507054\\
10.695	0.000960483623273943\\
10.77	0.00126075091445443\\
10.85	0.00163716290641957\\
10.94	0.00213782735156088\\
11.04	0.00280174043048475\\
11.15	0.0036801314021062\\
11.27	0.00484006029913074\\
11.4	0.00636897347723677\\
11.54	0.00838041487024557\\
11.69	0.0110211479838398\\
11.85	0.0144800159515058\\
12.025	0.0191463573346055\\
12.21	0.025250047271129\\
12.41	0.0334325370696858\\
12.62	0.044097514897814\\
12.845	0.0582857325769763\\
13.085	0.0771103716721057\\
13.34	0.102010518255198\\
13.61	0.134836885787522\\
13.895	0.177948372763822\\
14.195	0.234312064666867\\
14.51	0.307610739452414\\
14.84	0.402389234816767\\
15.19	0.526251806323641\\
15.565	0.690102306780857\\
15.97	0.909790637587289\\
16.405	1.20465506213261\\
16.87	1.60039560572326\\
17.355	2.11836950676523\\
17.86	2.79222159409236\\
18.39	3.67284385433278\\
18.95	4.82994241776845\\
19.535	6.3297763712948\\
20.145	8.26213741803796\\
20.785	10.7573126674433\\
21.455	13.9608634798366\\
22.155	18.0488862381045\\
22.89	23.2704002526944\\
23.655	29.8494600222998\\
24.455	38.1314996695381\\
25.29	48.4766480136861\\
26.16	61.2926127020959\\
27.065	77.0317391484458\\
28.01	96.29559721324\\
28.995	119.66003029306\\
29.935	145.252381652857\\
};
\addlegendentry{standard Lift \& Learn ROM $2r=16$}

\addplot [color=black, line width=3.0pt, forget plot]
  table[row sep=crcr]{%
10	1e-8\\
10	100000\\
};
\addplot [color=mycolor2,dashed, line width=2.0pt]
  table[row sep=crcr]{%
0.0100000000000016	2.83857992968706e-05\\
0.0199999999999996	6.17325571710127e-05\\
0.0249999999999986	6.86508508493945e-05\\
0.0300000000000011	7.02559048733065e-05\\
0.0350000000000001	6.735788635126e-05\\
0.0399999999999991	6.07039304441057e-05\\
0.0500000000000007	3.8816349970716e-05\\
0.0549999999999997	2.47840160767509e-05\\
0.0599999999999987	9.40327717557919e-06\\
0.0650000000000013	6.85686988788344e-06\\
0.0749999999999993	4.03781015592132e-05\\
0.0850000000000009	7.29485357311433e-05\\
0.100000000000001	0.000115424423157151\\
0.114999999999998	0.000146039295131573\\
0.125	0.000158512959416157\\
0.135000000000002	0.000164270640127029\\
0.140000000000001	0.00016463451601112\\
0.145	0.000163357768386163\\
0.155000000000001	0.000156076475946066\\
0.164999999999999	0.00014293280828781\\
0.175000000000001	0.000124588470868048\\
0.184999999999999	0.000101817185026221\\
0.195	7.54656635990612e-05\\
0.204999999999998	4.64191617197683e-05\\
0.210000000000001	3.11650089628302e-05\\
0.215	1.55715063385742e-05\\
0.219999999999999	2.52483204121744e-07\\
0.225000000000001	1.6200536180122e-05\\
0.234999999999999	4.805873750513e-05\\
0.25	9.43154480097518e-05\\
0.265000000000001	0.000136688845273625\\
0.280000000000001	0.000173269080761429\\
0.295000000000002	0.000202609692451006\\
0.309999999999999	0.000223738106726045\\
0.324999999999999	0.000236142324652633\\
0.335000000000001	0.00023950746563628\\
0.344999999999999	0.000239022045022352\\
0.355	0.000234826403115562\\
0.370000000000001	0.000222030457749157\\
0.385000000000002	0.000202241470248055\\
0.399999999999999	0.000176551307973227\\
0.414999999999999	0.00014619013995798\\
0.43	0.000112466368960895\\
0.445	7.6711411068686e-05\\
0.454999999999998	5.23933652232243e-05\\
0.465	2.8125674461421e-05\\
0.469999999999999	1.61219292351689e-05\\
0.475000000000001	4.26158335642412e-06\\
0.48	7.41472122456342e-06\\
0.489999999999998	3.00607645016271e-05\\
0.5	5.15247407690819e-05\\
0.515000000000001	8.09395999112894e-05\\
0.530000000000001	0.000106360781444437\\
0.545000000000002	0.000127201318719017\\
0.559999999999999	0.000143040658571181\\
0.574999999999999	0.000153621028118778\\
0.59	0.000158839369390762\\
0.600000000000001	0.000159353647433136\\
0.609999999999999	0.000157544210964033\\
0.620000000000001	0.000153478664188356\\
0.635000000000002	0.000143349295683492\\
0.649999999999999	0.000128723349092752\\
0.664999999999999	0.000110053245975905\\
0.68	8.78499861016732e-05\\
0.695	6.26655964026666e-05\\
0.705000000000002	4.45039478620401e-05\\
0.715	2.54469805156532e-05\\
0.719999999999999	1.56370267347938e-05\\
0.725000000000001	5.66820517633457e-06\\
0.73	4.43809446437628e-06\\
0.739999999999998	2.49786840953448e-05\\
0.75	4.58187052316817e-05\\
0.765000000000001	7.72907263382192e-05\\
0.785	0.000118861451824159\\
0.805	0.000158984432630404\\
0.830000000000002	0.000205675386212079\\
0.855	0.000247154950528738\\
0.879999999999999	0.000282361854218039\\
0.905000000000001	0.000310645960695411\\
0.93	0.000331738723502895\\
0.955000000000002	0.000345709205845424\\
0.98	0.000352910872891243\\
1.005	0.000353923749969454\\
1.03	0.000349495804002232\\
1.06	0.000338215420953246\\
1.095	0.0003188443911597\\
1.135	0.000291613389165946\\
1.19	0.000251398744353537\\
1.295	0.000188216994445156\\
1.33	0.000174977961805212\\
1.36	0.000167273674293256\\
1.39	0.000162896791465528\\
1.415	0.000161669741743253\\
1.445	0.00016286914626562\\
1.475	0.000166660021591269\\
1.515	0.000175065034014211\\
1.57	0.000191173128115452\\
1.7	0.000236860332023525\\
1.76	0.000256022736778051\\
1.815	0.000270495913582635\\
1.87	0.000281365881207308\\
1.93	0.00028889667794374\\
1.99	0.000292116582770063\\
2.06	0.000291162764270325\\
2.14	0.00028537781798832\\
2.25	0.000272908534554971\\
2.455	0.000250543103959443\\
2.56	0.000244190801168998\\
2.67	0.000241654029409234\\
2.825	0.000242343568800152\\
3.125	0.000243808471196871\\
3.34	0.000245542359778028\\
3.485	0.000250719234912821\\
3.92	0.000270547176825176\\
4.05	0.000269828077591684\\
4.2	0.0002645809556725\\
4.71	0.000244022756618278\\
4.92	0.000241948487001764\\
5.17	0.000243253940112708\\
5.575	0.00025001406780234\\
5.835	0.000252186393851161\\
6.11	0.000250276771112111\\
6.535	0.000247394313447557\\
7.015	0.00024964405528749\\
7.855	0.000252444118558515\\
8.66	0.000248467889605308\\
9.35	0.000256935385237966\\
9.495	0.000264159591807811\\
9.66	0.00027704368569914\\
9.885	0.000295543479040817\\
9.98	0.000298826555138022\\
10.055	0.000297159010486326\\
10.12	0.0002915051482546\\
10.18	0.000281960165040118\\
10.23	0.000270275206958104\\
10.28	0.000254734418319912\\
10.325	0.000237093581304747\\
10.365	0.000218243917061045\\
10.4	0.000199124430565689\\
10.435	0.000177400461211619\\
10.465	0.00015659525143974\\
10.495	0.000133674334335865\\
10.52	0.000112888521869082\\
10.545	9.05106307982348e-05\\
10.565	7.14231367453522e-05\\
10.585	5.12491869244511e-05\\
10.6	3.53869326374933e-05\\
10.61	2.44561169182588e-05\\
10.62	1.32357417612638e-05\\
10.625	7.51566815893057e-06\\
10.63	1.72163393585833e-06\\
10.645	1.61095925079735e-05\\
10.66	3.46259769230529e-05\\
10.675	5.38426392324709e-05\\
10.695	8.05812006205997e-05\\
10.72	0.000115852807397232\\
10.75	0.000160994467449882\\
10.785	0.000217727052472583\\
10.825	0.000288236319448743\\
10.87	0.000375285951520255\\
10.92	0.000482375431658964\\
10.975	0.000613960318361251\\
11.04	0.00079016862630965\\
11.115	0.0010250913921368\\
11.2	0.00133888117227343\\
11.3	0.00178515975889581\\
11.415	0.00242464197518188\\
11.55	0.00339268489526035\\
11.705	0.00487835475232145\\
11.88	0.00718881722163191\\
12.06	0.0104882620951514\\
12.245	0.0151483914416985\\
12.43	0.0214446523194454\\
12.615	0.0297760866365592\\
12.805	0.0409214507651258\\
13	0.0556576909411602\\
13.205	0.0754936542606119\\
13.425	0.102785606435489\\
13.66	0.140347934163274\\
13.915	0.193286321321722\\
14.185	0.26658061343268\\
14.47	0.367964319347139\\
14.765	0.505115752388038\\
15.07	0.689273800009699\\
15.39	0.939097850459631\\
15.725	1.27645374840104\\
16.075	1.72992513666399\\
16.435	2.32682540468665\\
16.81	3.11723497049396\\
17.195	4.14062454244269\\
17.59	5.45237325229409\\
18.005	7.16471013689055\\
18.445	9.41848546283074\\
18.91	12.3753072969042\\
19.395	16.1944170808838\\
19.9	21.0956284118075\\
20.43	27.4083417294294\\
20.985	35.4910567362475\\
21.565	45.7760960826584\\
22.175	58.8958492085579\\
22.815	75.529733435256\\
23.485	96.4845786282395\\
24.19	122.907516682792\\
24.925	155.755683621545\\
25.7	196.855599745653\\
26.51	247.584356117253\\
27.365	310.525928797425\\
28.265	388.088088937986\\
29.22	484.114355005883\\
29.935	565.948748236263\\
};
\addlegendentry{standard Lift \& Learn ROM $2r=20$}

\end{axis}

\begin{axis}[%
width=1.273\fheight,
height=0.955\fheight,
at={(-0.169\fheight,-0.133\fheight)},
scale only axis,
xmin=0,
xmax=1,
ymin=0,
ymax=1,
axis line style={draw=none},
ticks=none,
axis x line*=bottom,
axis y line*=left
]
\end{axis}
\end{tikzpicture}%

%% file: figures/motivation/state_train_final_hopinf.tex
%
%
\begin{tikzpicture}

\begin{axis}[%
width=0.983\fheight,
height=0.65\fheight,
at={(0\fheight,0\fheight)},
scale only axis,
xmin=0,
xmax=20,
xlabel style={font=\color{white!15!black}},
xlabel={Reduced dimension $2r$},
ymode=log,
ymin=0.001,
ymax=5,
yminorticks=true,
ylabel style={font=\color{white!15!black}},
ylabel={Relative state error},
axis background/.style={fill=white},
xmajorgrids,
ymajorgrids,
legend style={at={(0.25,1.15)}, anchor=south west, legend cell align=left, align=left, draw=white!15!black}
]
\addplot [
  color=orange!90!black,
  dashed,
  line width=2.0pt,
  mark=triangle*,
  mark size=4.0pt,
  mark options={solid, orange!90!black}
]
  table[row sep=crcr]{%
2	0.561723897305558\\
4	1.04790892463046\\
6	0.0189122741250397\\
8	0.00308752623679536\\
10	0.0012578942694591\\
12	0.00115035392853978\\
14	0.00113726753249264\\
16	0.00112221055513186\\
18	0.00108339797455637\\
20	0.00107750107054505\\
};
\addlegendentry{sp Lift \& Learn}

\addplot [
  color=purple!80!black,
  densely dotted,
  line width=2.0pt,
  mark=+,
  mark size=4.0pt,
  mark options={solid, purple!80!black}
]
  table[row sep=crcr]{%
2	0.407629517746982\\
4	0.133327078236933\\
6	0.00821573923591957\\
8	0.00234461206692522\\
10	0.00141931810918762\\
12	0.00241315308922548\\
14	0.00261707172830494\\
16	0.00256717474557063\\
18	0.00272067026875688\\
20	0.00277497503509196\\
};
\addlegendentry{HOpInf with spDEIM}

\end{axis}

\begin{axis}[%
width=1.269\fheight,
height=0.952\fheight,
at={(-0.166\fheight,-0.131\fheight)},
scale only axis,
xmin=0,
xmax=1,
ymin=0,
ymax=1,
axis line style={draw=none},
ticks=none,
axis x line*=bottom,
axis y line*=left
]
\end{axis}
\end{tikzpicture}%

%% file: figures/motivation/fom_energy_final_hopinf_sparse.tex
%
%
\begin{tikzpicture}

\begin{axis}[%
width=0.983\fheight,
height=0.65\fheight,
at={(0\fheight,0\fheight)},
scale only axis,
xmin=0,
xmax=30.005,
xlabel style={font=\color{white!15!black}},
xlabel={Time $t$},
ymode=log,
ymin=1e-05,
ymax=0.1,
yminorticks=true,
ylabel style={font=\color{white!15!black}},
ylabel={FOM energy error},
axis background/.style={fill=white},
xmajorgrids,
ymajorgrids,
legend style={at={(-0.05,1.15)}, anchor=south west, legend cell align=left, align=left, draw=white!15!black}
]
\addplot [color=orange!90!black, dashed, line width=2.0pt]
  table[row sep=crcr]{%
0.105	2.40264659510103e-05\\
0.204999999999998	9.1239577366764e-05\\
0.305	0.000190759865560608\\
0.405000000000001	0.000312327282577485\\
0.504999999999999	0.000448966231809321\\
0.605	0.000595254409674337\\
0.705000000000002	0.000745402609774717\\
0.805	0.000893499478078754\\
0.905000000000001	0.0010349145364927\\
1.005	0.00116714290230106\\
1.105	0.00128956180933244\\
1.205	0.00140258205736643\\
1.305	0.00150683039079808\\
1.405	0.00160268466175353\\
1.505	0.00169017177543367\\
1.605	0.00176909711749146\\
1.705	0.00183926571617689\\
1.805	0.00190069786982292\\
1.905	0.00195378438112641\\
2.005	0.00199935444318556\\
2.105	0.00203864704976411\\
2.305	0.00210462421973715\\
2.905	0.00228586835635554\\
3.205	0.00237143411688691\\
3.505	0.00244083227870818\\
4.005	0.00253787943591419\\
5.305	0.00280006243004608\\
5.805	0.00294509134747897\\
6.405	0.00313718973017955\\
6.805	0.00329027916505993\\
7.405	0.00355361738797627\\
8.005	0.0038551667112671\\
9.205	0.00458407815344516\\
9.905	0.00506966304278931\\
10.305	0.00536905037904772\\
10.605	0.00562942591817259\\
11.105	0.00613161821556219\\
11.905	0.00703507557964544\\
12.205	0.00736470310288555\\
13.305	0.00862847002881324\\
13.605	0.00908518404199581\\
13.905	0.00962303498979506\\
14.805	0.0115221681536042\\
15.105	0.0121306614267041\\
15.405	0.012687160542194\\
15.705	0.0131856821906723\\
16.105	0.0137853511624001\\
16.705	0.0146454934877084\\
17.205	0.0153375071889826\\
17.605	0.0158354872476337\\
17.905	0.016139954140678\\
18.205	0.0163644714811085\\
18.505	0.0165010669635563\\
18.805	0.0165413111141788\\
19.105	0.0164770158102357\\
19.405	0.0162991937132716\\
19.605	0.016111768836592\\
19.805	0.0158686313745762\\
20.105	0.0154138397629566\\
20.405	0.014884299357297\\
20.805	0.0141246933262749\\
21.505	0.0128695341165014\\
22.505	0.0113379965507221\\
22.805	0.0108327611402245\\
23.105	0.0102875223218405\\
23.405	0.00971548903540618\\
24.205	0.00830765602125612\\
24.605	0.00772542246790808\\
24.905	0.00735009575080309\\
25.405	0.00681108835042039\\
26.505	0.00575526824023599\\
27.705	0.00479469419606743\\
28.805	0.0040871872013753\\
29.505	0.00371522000325229\\
30.005	0.00347361430159765\\
};
\addlegendentry{sp Lift \& Learn ROM $4r=40$}

\addplot [color=purple!80!black, dashdotted, line width=2.0pt]
  table[row sep=crcr]{%
0.100000000000001	0.000158586019406215\\
0.199999999999999	0.000605503901898175\\
0.300000000000001	0.00126843666290099\\
0.399999999999999	0.00206150737371403\\
0.5	0.00290850713793489\\
0.600000000000001	0.00375080149258622\\
0.699999999999999	0.00454632798678745\\
0.800000000000001	0.00526751326805196\\
0.899999999999999	0.00590005371936898\\
1	0.00644114002096786\\
1.1	0.00689631778165634\\
1.2	0.00727566942969913\\
1.3	0.00759046030959711\\
1.4	0.00785094366696981\\
1.5	0.00806540129372379\\
1.6	0.00824012249049134\\
1.7	0.0083799231389321\\
1.8	0.00848885805439535\\
1.9	0.00857087625768218\\
2.1	0.00867177473180549\\
2.4	0.00873900074277501\\
3.1	0.00885181171805979\\
3.3	0.00882984987341826\\
3.6	0.008727570702181\\
4	0.00856843084764946\\
4.3	0.00849944343723049\\
5.4	0.00834749356562983\\
5.7	0.00839088181188075\\
6.1	0.00846385739382429\\
6.4	0.00847845382853763\\
7.2	0.00847469381374923\\
7.5	0.00854132057427748\\
7.8	0.00866239563724139\\
8.2	0.00887811864486213\\
8.6	0.00915199549547552\\
8.9	0.0094121802202558\\
9.7	0.0101844269460791\\
10	0.0104995710538454\\
10.2	0.0107683439225212\\
10.4	0.0110941070665333\\
10.8	0.0118650222810324\\
11.2	0.0127406601054289\\
11.5	0.0135216722113896\\
11.7	0.0141376333757819\\
11.9	0.0148436438500603\\
12.2	0.0160716151201434\\
12.8	0.0190209401887635\\
13.3	0.0218236892231403\\
13.9	0.0255417891011901\\
14.8	0.0323698962515237\\
15.2	0.0361787226612016\\
15.7	0.0416618859461123\\
15.9	0.0439148333956538\\
16.1	0.0460893650413112\\
16.3	0.0481356173087839\\
16.5	0.0500458921261105\\
16.8	0.052729013589687\\
17.2	0.0561995490264059\\
17.6	0.0596151911397589\\
17.9	0.0619378627481917\\
18.2	0.0638776084891111\\
18.5	0.0654573874113887\\
18.9	0.0672617650805767\\
19.2	0.0684110496755716\\
19.4	0.0689456989163764\\
19.6	0.0691579235899977\\
19.8	0.0689507411055244\\
20	0.0682921289232079\\
20.2	0.0672316866970382\\
20.5	0.0651431995878228\\
21.4	0.0584622651265646\\
21.7	0.056216103765852\\
22	0.0537143317913833\\
22.3	0.0509477706288626\\
22.6	0.0480325409592161\\
23	0.0441137329508231\\
23.3	0.0412043891320549\\
23.6	0.0383252902707907\\
24.1	0.033740411661995\\
24.7	0.0288200928152982\\
25	0.026468249241109\\
25.2	0.024897965719834\\
25.5	0.0225745501567651\\
25.9	0.0197656986907084\\
26.1	0.0185794349952956\\
26.3	0.0175546821746089\\
26.5	0.0166733454551213\\
26.7	0.0159060393284313\\
27	0.0149045225533406\\
27.5	0.0134682050934316\\
28.3	0.0115323545026413\\
28.6	0.0109080780134718\\
28.9	0.0103805321468186\\
29.1	0.0100934062885578\\
29.3	0.00985854905939118\\
29.5	0.00967414496779281\\
29.7	0.00953567466398654\\
30	0.00939091436827864\\
};
\addlegendentry{HOpInf with spDEIM ROM $2r=20$}

\addplot [color=black, line width=3.0pt, forget plot]
  table[row sep=crcr]{%
10	1e-05\\
10	0.1\\
};
\end{axis}
\end{tikzpicture}%

%% file: figures/exp_1d/state_train.tex
%
%
\begin{tikzpicture}

\begin{axis}[%
width=0.983\fheight,
height=0.65\fheight,
at={(0\fheight,0\fheight)},
scale only axis,
xmin=2,
xmax=10,
xlabel style={font=\color{white!15!black}},
xlabel={Reduced dimension $2r$},
ymode=log,
ymin=0.001,
ymax=1,
yminorticks=true,
ylabel style={font=\color{white!15!black}},
ylabel={Relative state error in $q$},
axis background/.style={fill=white},
xmajorgrids,
ymajorgrids,
legend style={draw=none, legend columns=-1},
legend style={at={(0.5,1.1)}, anchor=south west, legend cell align=left, align=left, draw=white!15!black},
legend style={font=\small}
]

\addplot [
  color=blue!70!black,
  dashdotted,
  line width=2.0pt,
  mark=o,
  mark size=4.0pt,
  mark options={solid, blue!70!black}
]
  table[row sep=crcr]{%
2	0.119672265262051\\
4	0.0107892836569169\\
6	0.00975451572147666\\
8	0.0096727992375151\\
10	0.0097000928488974\\
};
\addlegendentry{intrusive lifting}

\addplot [
  color=orange!90!black,
  dashed,
  line width=2.0pt,
  mark=triangle*,
  mark size=4.0pt,
  mark options={solid, orange!90!black}
]
  table[row sep=crcr]{%
2	0.0914758466266258\\
4	0.0108941983751602\\
6	0.00925365322445069\\
8	0.00925669604741247\\
10	0.00925942544683294\\
};
\addlegendentry{sp Lift \& Learn}

\addplot [
  color=purple!80!black,
  densely dotted,
  line width=2.0pt,
  mark=+,
  mark size=4.0pt,
  mark options={solid, purple!80!black}
]
  table[row sep=crcr]{%
2	0.0348654720072889\\
4	0.0111837540366834\\
6	0.00930809743491637\\
8	0.00945028807401968\\
10	0.00945174422474881\\
};
\addlegendentry{HOpInf with spDEIM}

\end{axis}
\end{tikzpicture}%

%% file: figures/exp_1d/state_test.tex
%
%
\begin{tikzpicture}

\begin{axis}[%
width=0.983\fheight,
height=0.65\fheight,
at={(0\fheight,0\fheight)},
scale only axis,
xmin=2,
xmax=10,
xlabel style={font=\color{white!15!black}},
xlabel={Reduced dimension $2r$},
ymode=log,
ymin=0.001,
ymax=1,
yminorticks=true,
ylabel style={font=\color{white!15!black}},
ylabel={Relative state error in $q$},
axis background/.style={fill=white},
xmajorgrids,
ymajorgrids,
legend style={at={(0.441,0.641)}, anchor=south west, legend cell align=left, align=left, draw=white!15!black}
]

\addplot [
  color=blue!70!black,
  dashdotted,
  line width=2.0pt,
  mark=o,
  mark size=4.0pt,
  mark options={solid, blue!70!black}
]
  table[row sep=crcr]{%
2	0.743897471868986\\
4	0.0182335027992255\\
6	0.0171739479273348\\
8	0.0167515636859981\\
10	0.0167739338285605\\
};

\addplot [
  color=orange!90!black,
  dashed,
  line width=2.0pt,
  mark=triangle*,
  mark size=4.0pt,
  mark options={solid, orange!90!black}
]
  table[row sep=crcr]{%
2	0.674395004619373\\
4	0.0234940796409283\\
6	0.0111198021435666\\
8	0.0113608910756335\\
10	0.0113788075448081\\
};

\addplot [
  color=purple!80!black,
  densely dotted,
  line width=2.0pt,
  mark=+,
  mark size=4.0pt,
  mark options={solid, purple!80!black}
]
  table[row sep=crcr]{%
2	0.101144055890839\\
4	0.0260106123877024\\
6	0.0138436460786237\\
8	0.0174409110135583\\
10	0.0172352926739779\\
};

\end{axis}
\end{tikzpicture}%

%% file: figures/exp_1d/efficacy.tex
%
%
\begin{tikzpicture}

\begin{axis}[%
width=0.987\fheight,
height=0.65\fheight,
at={(0\fheight,0\fheight)},
scale only axis,
xmin=2,
xmax=10,
xlabel style={font=\color{white!15!black}},
xlabel={Reduced dimension $2r$},
ymode=log,
ymin=10,
ymax=10000,
yminorticks=true,
ylabel style={font=\color{white!15!black}},
ylabel={Efficacy},
axis background/.style={fill=white},
xmajorgrids,
ymajorgrids,
legend style={at={(0.1,1.2)}, anchor=south west, legend cell align=left, align=left, draw=white!15!black},
legend style={font=\small}
]
\addplot [
  color=orange!90!black,
  dashed,
  line width=2.0pt,
  mark=triangle*,
  mark size=4.0pt,
  mark options={solid, orange!90!black}
]
  table[row sep=crcr]{%
2	1092.78299351485\\
4	6425.34584454716\\
6	5820.37840559622\\
8	4498.84366463997\\
10	3309.51753465284\\
};
\addlegendentry{sp Lift \& Learn}

\addplot [
  color=purple!80!black,
  densely dotted,
  line width=2.0pt,
  mark=+,
  mark size=4.0pt,
  mark options={solid, purple!80!black}
]
  table[row sep=crcr]{%
2	24.6051184078054\\
4	51.326078562473\\
6	55.6335416125225\\
8	50.6805923146553\\
10	47.3875424676396\\
};
\addlegendentry{HOpInf with spDEIM}

\end{axis}
\end{tikzpicture}%

%% file: figures/exp_1d/fom_energy_sparse.tex
%
%
\begin{tikzpicture}

\begin{axis}[%
width=0.983\fheight,
height=0.65\fheight,
at={(0\fheight,0\fheight)},
scale only axis,
xmin=0,
xmax=100,
xlabel style={font=\color{white!15!black}},
xlabel={Time $t$},
ymode=log,
ymin=1e-07,
ymax=0.01,
yminorticks=true,
ylabel style={font=\color{white!15!black}},
ylabel={FOM energy error},
axis background/.style={fill=white},
xmajorgrids,
ymajorgrids,
legend style={at={(-0.05,1.1)}, anchor=south west, legend cell align=left, align=left, draw=white!15!black},
legend style={font=\small}
]
\addplot [color=blue!70!black, dotted, line width=2.0pt]
  table[row sep=crcr]{%
0.549999999999997	6.48988753165119e-05\\
1.05	0.000163890891204086\\
1.55	0.000263944481823006\\
2.05	0.000331015459942664\\
2.55	0.000359629426244838\\
3.05	0.000366155578941096\\
3.55	0.000260866781335147\\
4.05	0.0001412509624325\\
4.55	6.19853120281197e-05\\
5.05	3.62593103007267e-05\\
5.55	6.48377157653989e-05\\
6.05	0.000121359792568469\\
6.55	0.000268803998599872\\
7.05	0.000355511498131783\\
7.55	0.000363590175775504\\
8.05	0.000336649822936048\\
8.55	0.000265520483369354\\
9.05	0.000194065793463904\\
9.55	6.25079349626414e-05\\
10.05	3.88275834586298e-06\\
10.55	4.59831878423809e-05\\
11.05	0.000145921836584229\\
11.55	0.00026925433470148\\
12.05	0.000318286263853583\\
12.55	0.000362608557029558\\
13.05	0.000356787762206692\\
13.55	0.000274417163292408\\
14.05	0.000163178045944253\\
14.55	6.53314027249794e-05\\
15.05	5.44309448615302e-05\\
15.55	4.69288777564841e-05\\
16.05	0.000114133229945092\\
16.55	0.000253129790702175\\
17.05	0.000349820362290093\\
17.55	0.000375877302776973\\
18.05	0.000325721597713745\\
18.55	0.000277274189477144\\
19.05	0.000191041583455915\\
19.55	7.42271883183702e-05\\
20.05	1.37913136296412e-05\\
20.55	2.79896619549829e-05\\
21.05	0.000148813716331467\\
21.55	0.000245544178870731\\
22.05	0.00031999851571102\\
22.55	0.000367577510130054\\
23.05	0.000362319578590654\\
23.55	0.000302210054200639\\
24.05	0.000159202581806469\\
24.55	7.8516092117981e-05\\
25.05	5.01918340797469e-05\\
25.55	3.97742565828837e-05\\
26.05	0.000110652941709261\\
26.55	0.000222071641911864\\
27.05	0.000351628533814838\\
27.55	0.000368787717648984\\
28.05	0.000329309055174762\\
28.55	0.000288914407961286\\
29.05	0.000192122085718531\\
29.55	0.000103945548047518\\
30.05	1.10628195170879e-05\\
30.55	2.66124945395391e-05\\
31.05	0.000133758745825267\\
31.55	0.000231810413553388\\
32.05	0.000331186102661779\\
32.55	0.000357440252536428\\
33.05	0.000364074124381672\\
33.55	0.000294146168358761\\
34.05	0.000165947687199772\\
34.55	9.4077942259628e-05\\
35.05	4.41190147179807e-05\\
35.55	5.05829163734656e-05\\
36.05	8.57138680087917e-05\\
36.55	0.000212078885297548\\
37.05	0.000351949502663937\\
37.55	0.000376776919803643\\
38.05	0.000350214924855122\\
38.55	0.000279521946083797\\
39.05	0.000209855655562987\\
39.55	0.000109577710043423\\
40.05	1.65708799802724e-05\\
40.55	3.00323771163755e-05\\
41.05	0.000103192343652998\\
41.55	0.000229834515971645\\
42.05	0.000309978645251233\\
42.55	0.000355991946573403\\
43.05	0.000371987860518314\\
43.55	0.000296673713324197\\
44.05	0.000193866091533015\\
44.55	8.97534448769544e-05\\
45.05	5.62800479827576e-05\\
45.55	5.01536858005401e-05\\
46.05	7.39527297860711e-05\\
46.55	0.000212895867037045\\
47.05	0.000331014412278039\\
47.55	0.000381184428764172\\
48.05	0.000344724572726428\\
48.55	0.000282198513499914\\
49.05	0.000225804245903648\\
49.55	0.000107138023587653\\
50.05	3.20096464659728e-05\\
50.55	1.47251183705532e-05\\
51.05	9.44707489001726e-05\\
51.55	0.000220634326378514\\
52.05	0.000301558354339991\\
52.55	0.000372069463697332\\
53.05	0.000370791104512822\\
53.55	0.00031392037674151\\
54.05	0.000197390218374365\\
54.55	9.38492075569104e-05\\
55.05	7.28989854228529e-05\\
55.55	3.956029120117e-05\\
56.05	7.38412790526297e-05\\
56.55	0.000178105866161279\\
57.05	0.000317504919847202\\
57.55	0.000387340301420261\\
58.05	0.000347407073628982\\
58.55	0.000301335358495513\\
59.05	0.000214120942345373\\
59.55	0.000124684434868791\\
60.05	4.10787108454254e-05\\
60.55	1.41779663794767e-05\\
61.05	0.000101387376716646\\
61.55	0.000192161539694092\\
62.05	0.000300826557981401\\
62.55	0.000360048237657865\\
63.05	0.00037309715155688\\
63.55	0.00033019434246616\\
64.05	0.000192624230792926\\
64.55	0.000109571146059738\\
65.05	6.27635739235304e-05\\
65.55	4.55242373788242e-05\\
66.05	6.87884715784515e-05\\
66.55	0.000155713247625534\\
67.05	0.000322905821026239\\
67.55	0.000381782236445626\\
68.05	0.000359612938134381\\
68.55	0.000303406757378041\\
69.05	0.000220063416716583\\
69.55	0.000148605007621842\\
70.05	4.04365390093367e-05\\
70.55	2.14098287124459e-05\\
71.05	7.67403997715722e-05\\
71.55	0.000177735343901649\\
72.05	0.0002952187054829\\
72.55	0.000351845926038401\\
73.05	0.000383786065182768\\
73.55	0.000324043381140543\\
74.05	0.000208578944367217\\
74.55	0.000114095613346514\\
75.05	6.38197098711905e-05\\
75.55	6.45978029426024e-05\\
76.05	5.45669103031615e-05\\
76.55	0.000152176127538941\\
77.05	0.000298871764167551\\
77.55	0.00038099813113002\\
78.05	0.000376745030017719\\
78.55	0.00030258603086381\\
79.05	0.000239079922181706\\
79.55	0.000136482470581566\\
80.05	4.76384626126478e-05\\
80.55	2.31041774835156e-05\\
81.05	6.22483514213733e-05\\
81.55	0.000183148889116283\\
82.05	0.000269211700909097\\
82.55	0.000353102896388676\\
83.05	0.000385047652250345\\
83.55	0.000337655090454122\\
84.05	0.000237086279461254\\
84.55	0.000108604517081028\\
85.05	7.77086509638349e-05\\
85.55	5.66517166008546e-05\\
86.05	5.35324102107417e-05\\
86.55	0.000142421940428078\\
87.05	0.00027111378760114\\
87.55	0.000383276921993665\\
88.05	0.000368764914266995\\
88.55	0.000312316899889991\\
89.05	0.00024324006312854\\
89.55	0.000139417116760197\\
90.05	6.83718397749574e-05\\
90.55	1.86230717319389e-05\\
91.05	6.84717386745413e-05\\
91.55	0.000161963588827777\\
92.05	0.000257444311210345\\
92.55	0.000357896941581236\\
93.05	0.000385134528612847\\
93.55	0.000356801923632291\\
94.05	0.000229938190268663\\
94.55	0.000114094223627021\\
95.05	7.98815891673971e-05\\
95.55	5.28845898255293e-05\\
96.05	6.48096194839167e-05\\
96.55	0.000113374580370186\\
97.05	0.000261690678580558\\
97.55	0.00037385794469787\\
98.05	0.000377531717457836\\
98.55	0.000338160052754344\\
99.05	0.00024180947186278\\
99.55	0.000162543518829282\\
};
\addlegendentry{intrusive lifting ROM $3r=15$}

\addplot [color=orange!90!black, dashed, line width=2.0pt]
  table[row sep=crcr]{%
0.549999999999997	0.000145658717752396\\
1.05	0.000232545758038346\\
1.55	0.000288811918864078\\
2.05	0.000315260984265783\\
2.55	0.000288515965811441\\
3.05	0.000402570510227625\\
3.55	0.000291390534886284\\
4.05	0.000146643523447614\\
4.55	9.0596114872632e-05\\
5.05	0.000108814584045448\\
5.55	0.000161783625166389\\
6.05	0.000162667823803955\\
6.55	0.000266002762289455\\
7.05	0.000323163388401831\\
7.55	0.000274936449562046\\
8.05	0.000384310641725459\\
8.55	0.000342724063467276\\
9.05	0.000244498446724736\\
9.55	6.3354280127922e-05\\
10.05	3.8767607067515e-05\\
10.55	0.000191457585424695\\
11.05	0.000283461412030602\\
11.55	0.000312356693971686\\
12.05	0.000232182254329097\\
12.55	0.000221916704815259\\
13.05	0.00041790008441637\\
13.55	0.00037505272536438\\
14.05	0.000231133192977916\\
14.55	3.21729733129415e-05\\
15.05	2.75343034392587e-05\\
15.55	0.000152416040338079\\
16.05	0.000216032930383328\\
16.55	0.000335454451944491\\
17.05	0.000282242666438613\\
17.55	0.000211382756084907\\
18.05	0.000354809913985657\\
18.55	0.000385012082901328\\
19.05	0.00033298330659307\\
19.55	5.53722668298507e-05\\
20.05	0.000100228375148953\\
20.55	0.000132108448271247\\
21.05	0.000302430813312661\\
21.55	0.000401591337887154\\
22.05	0.000228565507677531\\
22.55	0.000152428006482498\\
23.05	0.000341138865853509\\
23.55	0.000418527373462888\\
24.05	0.000360248445402685\\
24.55	5.74716430049735e-05\\
25.05	3.4679966659814e-05\\
25.55	6.31226582710131e-05\\
26.05	0.000208960391039035\\
26.55	0.000423650985141705\\
27.05	0.00033055594385197\\
27.55	0.000196151988207824\\
28.05	0.000252550249480714\\
28.55	0.000339283898033086\\
29.05	0.000411601127166307\\
29.55	0.000143826747555988\\
30.05	8.97216097116064e-05\\
30.55	2.02492388285938e-05\\
31.05	0.000234324679161661\\
31.55	0.000457258340913591\\
32.05	0.000296695027892392\\
32.55	0.000169158725301056\\
33.05	0.000253445967478905\\
33.55	0.000384748086711971\\
34.05	0.00041994720106602\\
34.55	0.00012920118477875\\
35.05	1.56709146314959e-07\\
35.55	1.08432379974219e-05\\
36.05	0.000131819729924436\\
36.55	0.000429445148563294\\
37.05	0.000406509488357234\\
37.55	0.000270341743971892\\
38.05	0.000181509849802201\\
38.55	0.000241114614827683\\
39.05	0.000400516791782903\\
39.55	0.000239268009764077\\
40.05	5.67362316221372e-06\\
40.55	3.8099763241104e-05\\
41.05	0.000134484536132344\\
41.55	0.000414803940654464\\
42.05	0.000344313202841673\\
42.55	0.000262176609323758\\
43.05	0.000244228207462826\\
43.55	0.000339010877580328\\
44.05	0.000371580183106115\\
44.55	0.000153431748575867\\
45.05	9.37331939503718e-05\\
45.55	4.22564424093687e-06\\
46.05	7.19919099752336e-05\\
46.55	0.000345419454039333\\
47.05	0.000415308619216982\\
47.55	0.000363912972574407\\
48.05	0.00019999407070446\\
48.55	0.000198412220925232\\
49.05	0.000324442567544595\\
49.55	0.000252193359967512\\
50.05	9.70736726554319e-05\\
50.55	1.19141004004608e-05\\
51.05	0.000100548486737351\\
51.55	0.000322520745418881\\
52.05	0.000314947803204904\\
52.55	0.000334177888466488\\
53.05	0.000300394742246773\\
53.55	0.000358893990142218\\
54.05	0.000295024925206715\\
54.55	0.000102307709788373\\
55.05	0.000140702147431526\\
55.55	7.24648830969992e-05\\
56.05	0.000100837060728152\\
56.55	0.000261495498836249\\
57.05	0.000346976539940176\\
57.55	0.000381013328067504\\
58.05	0.00025895333106753\\
58.55	0.000260810257072647\\
59.05	0.000279471488353014\\
59.55	0.000199988607148759\\
60.05	0.000100759267659187\\
60.55	2.45988355866678e-05\\
61.05	0.000151234725436037\\
61.55	0.000280027723289845\\
62.05	0.000259571609421753\\
62.55	0.000318159883752739\\
63.05	0.000323980555070903\\
63.55	0.000430316438392175\\
64.05	0.000290248769217785\\
64.55	5.23494694985655e-05\\
65.05	9.74109090651834e-05\\
65.55	9.72384461135509e-05\\
66.05	0.00018225284979756\\
66.55	0.000251174387232129\\
67.05	0.000286313819793046\\
67.55	0.000317760910478994\\
68.05	0.000274767947238343\\
68.55	0.000359130642126961\\
69.05	0.000312005249916999\\
69.55	0.000178911429703286\\
70.05	3.84588043832148e-05\\
70.55	6.89632661057748e-06\\
71.05	0.00020407996684576\\
71.55	0.000316539220991618\\
72.05	0.000270022125881281\\
72.55	0.000253531072299754\\
73.05	0.000255854824229067\\
73.55	0.000462812940879171\\
74.05	0.000360217757948533\\
74.55	8.99467915971403e-05\\
75.05	3.00706207049103e-05\\
75.55	4.32758667595374e-05\\
76.05	0.000216370809820756\\
76.55	0.00028408755199934\\
77.05	0.000305733930684682\\
77.55	0.000265263761976768\\
78.05	0.000238863501194136\\
78.55	0.000388515146382277\\
79.05	0.000363113474897129\\
79.55	0.00024033826201581\\
80.05	6.71442014147855e-06\\
80.55	8.68617244443881e-05\\
81.05	0.000178755251168363\\
81.55	0.000360775637827544\\
82.05	0.000359513669721321\\
82.55	0.000228948536336674\\
83.05	0.000151096108155816\\
83.55	0.000404210678486508\\
84.05	0.00041758459096365\\
84.55	0.000203126246227952\\
85.05	2.62325156422502e-05\\
85.55	2.19275067454568e-05\\
86.05	0.000161234517525552\\
86.55	0.000272919424503793\\
87.05	0.000377849198267905\\
87.55	0.000294978740932951\\
88.05	0.000223577108370141\\
88.55	0.000331051183438272\\
89.05	0.000347569217552128\\
89.55	0.000327810054150984\\
90.05	5.6157395930058e-05\\
90.55	0.000122783297088592\\
91.05	9.4811601822172e-05\\
91.55	0.000335839346946116\\
92.05	0.000441889610724657\\
92.55	0.000269317755138286\\
93.05	0.000112996346934888\\
93.55	0.000308169819187166\\
94.05	0.000400710682002072\\
94.55	0.000295220528201238\\
95.05	8.45380543256287e-05\\
95.55	1.73573753159598e-05\\
96.05	7.97789407988224e-05\\
96.55	0.000191907775624242\\
97.05	0.000415561546726988\\
97.55	0.000370587103756888\\
98.05	0.000279837154770649\\
98.55	0.000267803483229745\\
99.05	0.000265952146613256\\
99.55	0.000351276559999254\\
};
\addlegendentry{sp Lift \& Learn ROM $3r=15$}

\addplot [color=purple!80!black, dashdotted, line width=2.0pt]
  table[row sep=crcr]{%
0.549999999999997	0.000302225736019475\\
1.05	0.000647002545889198\\
1.55	0.000776556990970169\\
2.05	0.000644826068433304\\
2.55	0.000520189411554657\\
3.05	0.000678075027355095\\
3.55	0.000750226082591974\\
4.05	0.000580856702166278\\
4.55	0.000281988454177779\\
5.05	0.000209653004355678\\
5.55	0.000303829597487902\\
6.05	0.000483761051372598\\
6.55	0.000775629873518521\\
7.05	0.000703409363670741\\
7.55	0.000559043500593255\\
8.05	0.000653527707258461\\
8.55	0.000740667368088057\\
9.05	0.000671350724734566\\
9.55	0.000300828606076535\\
10.05	4.12737460921939e-05\\
10.55	0.000258665835575226\\
11.05	0.000633645407546946\\
11.55	0.000808302463244288\\
12.05	0.000648034231531358\\
12.55	0.000548657878987737\\
13.05	0.000653432661271727\\
13.55	0.000748271031324543\\
14.05	0.000659144689712255\\
14.55	0.000304098539433276\\
15.05	0.000167247575244431\\
15.55	0.000262116534861148\\
16.05	0.000447961023837209\\
16.55	0.000803154911157097\\
17.05	0.000693535535209356\\
17.55	0.0005580287607065\\
18.05	0.000612203707591301\\
18.55	0.000731780789671913\\
19.05	0.000716322517583117\\
19.55	0.000327578105473321\\
20.05	9.69567276382221e-05\\
20.55	0.000188917150579043\\
21.05	0.000592614972186935\\
21.55	0.000802679368742147\\
22.05	0.000614734199909849\\
22.55	0.000550848506095973\\
23.05	0.000643239133881106\\
23.55	0.000792691488347911\\
24.05	0.000656088556154886\\
24.55	0.000298928489774148\\
25.05	0.000179251083429803\\
25.55	0.000222923202499746\\
26.05	0.000485953356859673\\
26.55	0.000761337064431272\\
27.05	0.00072325825348211\\
27.55	0.000579976751094469\\
28.05	0.000561684441919593\\
28.55	0.000752527417387511\\
29.05	0.000716129472989796\\
29.55	0.000400377255219313\\
30.05	0.000139165928053879\\
30.55	0.000168378095770863\\
31.05	0.00054788735112698\\
31.55	0.000781218230345571\\
32.05	0.000666476245959917\\
32.55	0.000555555488862931\\
33.05	0.000604313861671202\\
33.55	0.000796386793471752\\
34.05	0.000645459899459471\\
34.55	0.000326482039704386\\
35.05	0.000178986268096054\\
35.55	0.000222604718012608\\
36.05	0.000511246224207331\\
36.55	0.000698895335631286\\
37.05	0.000748990124732113\\
37.55	0.00058023067949305\\
38.05	0.000544466042718434\\
38.55	0.000783959516559708\\
39.05	0.000716259428667705\\
39.55	0.000424886585021903\\
40.05	9.81952836982468e-05\\
40.55	0.000104706131452085\\
41.05	0.00053432519508509\\
41.55	0.000745521439663542\\
42.05	0.000715965016791308\\
42.55	0.000555782001517285\\
43.05	0.000621278601915378\\
43.55	0.000790624380872704\\
44.05	0.000617998087992465\\
44.55	0.000413245416276299\\
45.05	0.000176810371418812\\
45.55	0.0002449391480585\\
46.05	0.000465200564053796\\
46.55	0.00066281347995352\\
47.05	0.000769383829237679\\
47.55	0.000574436902154043\\
48.05	0.000571938977017056\\
48.55	0.00074472449099064\\
49.05	0.000746157431819463\\
49.55	0.00048721594765254\\
50.05	0.000109866201708471\\
50.55	0.000101505497423008\\
51.05	0.000475743744370019\\
51.55	0.000755485011595683\\
52.05	0.000764675756782471\\
52.55	0.000536693780785033\\
53.05	0.000593872309895365\\
53.55	0.000774904801636898\\
54.05	0.000667148017603061\\
54.55	0.000459354172285251\\
55.05	0.000198827603184709\\
55.55	0.000196176364734813\\
56.05	0.000401883438106979\\
56.55	0.000653963699105875\\
57.05	0.000777196874286374\\
57.55	0.000578770765986244\\
58.05	0.000571726538772161\\
58.55	0.000682435148201985\\
59.05	0.000757725611872989\\
59.55	0.000500716942052613\\
60.05	9.77963545142808e-05\\
60.55	0.000171776456709272\\
61.05	0.000420256611097038\\
61.55	0.000753878024157315\\
62.05	0.000757097240575136\\
62.55	0.000552128491851804\\
63.05	0.000582250335262662\\
63.55	0.000725960888708168\\
64.05	0.000731579670074841\\
64.55	0.00042789906353675\\
65.05	0.000235878916975119\\
65.55	0.00018598234042439\\
66.55	0.000679733091310305\\
67.05	0.000765462434348064\\
67.55	0.000634717508417645\\
68.05	0.000572860868251885\\
68.55	0.000660397011310844\\
69.05	0.00079442867048797\\
69.55	0.000540259320403172\\
70.05	0.000148427266017295\\
70.55	0.000103571320968388\\
71.05	0.000396118092381743\\
71.55	0.000755094454426524\\
72.05	0.000736356220273931\\
72.55	0.000571411478995538\\
73.05	0.000577100868881463\\
73.55	0.00073252388984377\\
74.05	0.00078535757395781\\
74.55	0.000419680570674079\\
75.05	0.000205318823707603\\
75.55	0.000191127173432688\\
76.05	0.000327941956030496\\
76.55	0.000697592567293443\\
77.05	0.000736734708082786\\
77.55	0.000620377582064636\\
78.05	0.000550940159269496\\
78.55	0.000662348715681725\\
79.05	0.00079181710021765\\
79.55	0.000565391185100226\\
80.05	0.000236832260633917\\
80.55	4.08903321571791e-05\\
81.05	0.000364747443670091\\
81.55	0.000715100857225076\\
82.05	0.000729708198092839\\
82.55	0.000610421435992718\\
83.05	0.000537620018800934\\
83.55	0.000726418770111578\\
84.05	0.000779989222326548\\
84.55	0.000459473686534495\\
85.05	0.000253131498236697\\
85.55	0.000196233233967038\\
86.05	0.000330314279178966\\
86.55	0.000619620352089057\\
87.05	0.000776932954176978\\
87.55	0.000662670050750105\\
88.05	0.000538804678303013\\
88.55	0.000663973190898522\\
89.05	0.000767051739687407\\
89.55	0.000617350274713357\\
90.05	0.00024101248491195\\
90.55	4.13553855307497e-05\\
91.05	0.000336983649970747\\
91.55	0.00069248774540043\\
92.05	0.000767774670054815\\
92.55	0.000625476892547291\\
93.05	0.000531990139922897\\
93.55	0.000711037075528583\\
94.05	0.000753825762491999\\
94.55	0.000497067670982437\\
95.05	0.000261331547242274\\
95.55	0.000196872875068246\\
96.05	0.000348564165867596\\
96.55	0.000539050297699666\\
97.05	0.00076337205499947\\
97.55	0.000669023931954488\\
98.05	0.000560144478581522\\
98.55	0.000672892989215798\\
99.05	0.000727017963881554\\
99.55	0.00063544774656757\\
};
\addlegendentry{HOpInf with spDEIM ROM $2r=10$}

\addplot [color=black, line width=3.0pt, forget plot]
  table[row sep=crcr]{%
10	1e-07\\
10	0.1\\
};
\end{axis}
\end{tikzpicture}%

%% file: figures/sg_2d/q_train_final.tex
%
%
\begin{tikzpicture}

\begin{axis}[%
width=0.983\fheight,
height=0.65\fheight,
at={(0\fheight,0\fheight)},
scale only axis,
xmin=0,
xmax=40,
xlabel style={font=\color{white!15!black}},
xlabel={Reduced dimension $2r$},
ymode=log,
ymin=0.001,
ymax=1,
yminorticks=true,
ylabel style={font=\color{white!15!black}},
ylabel={Relative state error in $q$},
axis background/.style={fill=white},
xmajorgrids,
ymajorgrids,
legend style={draw=none, legend columns=-1},
legend style={at={(0.25,1.1)}, anchor=south west, legend cell align=left, align=left, draw=white!15!black},
legend style={font=\small}
]
\addplot [
  color=blue!70!black,
  dashdotted,
  line width=2.0pt,
  mark=o,
  mark size=4.0pt,
  mark options={solid, blue!70!black}
]
  table[row sep=crcr]{%
4	0.205899014359303\\
8	0.06710953093224\\
12	0.0345802765123352\\
16	0.00836253324136981\\
20	0.00401692461776651\\
24	0.00293872021773688\\
28	0.00232797869508457\\
32	0.00223045968051165\\
36	0.00211385003166397\\
40	0.00210560227895953\\
};
\addlegendentry{intrusive lifting}

\addplot [
  color=orange!90!black,
  dashed,
  line width=2.0pt,
  mark=triangle*,
  mark size=4.0pt,
  mark options={solid, orange!90!black}
]
  table[row sep=crcr]{%
4	0.24266766129569\\
8	0.0691722107715322\\
12	0.0350473268638889\\
16	0.0106455848990292\\
20	0.00352828840194326\\
24	0.00390853353636164\\
28	0.00152206863655869\\
32	0.00136041476180307\\
36	0.0012034997637606\\
40	0.0011620193887574\\
};
\addlegendentry{sp Lift \& Learn}

\addplot [
  color=purple!80!black,
  densely dotted,
  line width=2.0pt,
  mark=+,
  mark size=4.0pt,
  mark options={solid, purple!80!black}
]
  table[row sep=crcr]{%
4	0.15838049867907\\
8	0.0744673208960969\\
12	0.0257796485179741\\
16	0.0115366832427231\\
20	0.00667352503460737\\
24	0.00641696452514929\\
28	0.00579934213802949\\
32	0.00582420708595129\\
36	0.00597926193656035\\
40	0.00599723525665712\\
};
\addlegendentry{HOpInf with spDEIM}

\end{axis}
\end{tikzpicture}%

%% file: figures/sg_2d/q_test_final.tex
%
%
\begin{tikzpicture}

\begin{axis}[%
width=0.983\fheight,
height=0.65\fheight,
at={(0\fheight,0\fheight)},
scale only axis,
xmin=0,
xmax=40,
xlabel style={font=\color{white!15!black}},
xlabel={Reduced dimension $2r$},
ymode=log,
ymin=0.01,
ymax=10,
yminorticks=true,
ylabel style={font=\color{white!15!black}},
ylabel={Relative state error in $q$},
axis background/.style={fill=white},
xmajorgrids,
ymajorgrids,
]
\addplot [
  color=blue!70!black,
  dashdotted,
  line width=2.0pt,
  mark=o,
  mark size=4.0pt,
  mark options={solid, blue!70!black}
]
  table[row sep=crcr]{%
4	0.448515831339607\\
8	0.238467136337287\\
12	0.156704830249936\\
16	0.0968759243782355\\
20	0.0736076030240395\\
24	0.0672613175553044\\
28	0.0564363401515626\\
32	0.0562527642080719\\
36	0.049322693439924\\
40	0.0494462863495051\\
};

\addplot [
  color=orange!90!black,
  dashed,
  line width=2.0pt,
  mark=triangle*,
  mark size=4.0pt,
  mark options={solid, orange!90!black}
]
  table[row sep=crcr]{%
4	0.378018218184115\\
8	0.234285714590927\\
12	0.171080932168652\\
16	0.105929805607251\\
20	0.0712833541784429\\
24	0.068711608526565\\
28	0.0582902504637456\\
32	0.0589014487242964\\
36	0.0560710647817749\\
40	0.0554768114995593\\
};

\addplot [
  color=purple!80!black,
  densely dotted,
  line width=2.0pt,
  mark=+,
  mark size=4.0pt,
  mark options={solid, purple!80!black}
]
  table[row sep=crcr]{%
4	0.427919721645593\\
8	0.215335824348628\\
12	0.170504269102923\\
16	0.10201405122822\\
20	0.067762748828086\\
24	0.0620642971970352\\
28	0.0509982792430625\\
32	0.0501449306746006\\
36	0.0452607599172993\\
40	0.0455072683909793\\
};

\end{axis}

\begin{axis}[%
width=1.273\fheight,
height=0.955\fheight,
at={(-0.169\fheight,-0.133\fheight)},
scale only axis,
xmin=0,
xmax=1,
ymin=0,
ymax=1,
axis line style={draw=none},
ticks=none,
axis x line*=bottom,
axis y line*=left
]
\end{axis}
\end{tikzpicture}%

%% file: figures/sg_2d/efficacy_final.tex
%
%
\begin{tikzpicture}

\begin{axis}[%
width=0.987\fheight,
height=0.65\fheight,
at={(0\fheight,0\fheight)},
scale only axis,
xmin=0,
xmax=40,
xlabel style={font=\color{white!15!black}},
xlabel={Reduced dimension $2r$},
ymode=log,
ymin=100,
ymax=10000,
yminorticks=true,
ylabel style={font=\color{white!15!black}},
ylabel={Efficacy},
axis background/.style={fill=white},
xmajorgrids,
ymajorgrids,
legend style={at={(0.1,1.25)}, anchor=south west, legend cell align=left, align=left, draw=white!15!black},
legend style={font=\small}
]
\addplot [
  color=orange!90!black,
  dashed,
  line width=2.0pt,
  mark=triangle*,
  mark size=4.0pt,
  mark options={solid, orange!90!black}
]
  table[row sep=crcr]{%
4	440.811903309089\\
8	580.86489336787\\
12	535.054673214124\\
16	1215.64743495144\\
20	1554.30451953931\\
24	1480.86638688627\\
28	1143.50665386102\\
32	671.286032684236\\
36	436.797040903393\\
40	317.502839115204\\
};
\addlegendentry{sp Lift \& Learn}

\addplot [
  color=purple!80!black,
  densely dotted,
  line width=2.0pt,
  mark=+,
  mark size=4.0pt,
  mark options={solid, purple!80!black}
]
  table[row sep=crcr]{%
4	148.702443393949\\
8	262.361457561601\\
12	654.216168758546\\
16	1263.6521739705\\
20	1840.39433780054\\
24	1551.55940840501\\
28	1618.8331205258\\
32	1338.0552707817\\
36	1172.35657805515\\
40	1040.24654012406\\
};
\addlegendentry{HOpInf with spDEIM}

\end{axis}

\begin{axis}[%
width=1.273\fheight,
height=0.955\fheight,
at={(-0.166\fheight,-0.133\fheight)},
scale only axis,
xmin=0,
xmax=1,
ymin=0,
ymax=1,
axis line style={draw=none},
ticks=none,
axis x line*=bottom,
axis y line*=left
]
\end{axis}
\end{tikzpicture}%

%% file: figures/sg_2d/fom_energy_final.tex
%
%
\begin{tikzpicture}

\begin{axis}[%
width=0.983\fheight,
height=0.65\fheight,
at={(0\fheight,0\fheight)},
scale only axis,
xmin=0,
xmax=12.5,
xlabel style={font=\color{white!15!black}},
xlabel={Time $t$},
ymode=log,
ymin=1e-08,
ymax=100,
yminorticks=true,
ylabel style={font=\color{white!15!black}},
ylabel={FOM energy error},
axis background/.style={fill=white},
xmajorgrids,
ymajorgrids,
legend style={at={(-0.05,1.1)}, anchor=south west, legend cell align=left, align=left, draw=white!15!black},
legend style={font=\small}
]
\addplot [  color=blue!70!black, dashdotted, line width=2.0pt]
  table[row sep=crcr]{%
0.0199999999999996	8.62809971295064e-08\\
0.0299999999999994	3.4376227085886e-07\\
0.0399999999999991	7.68381129091722e-07\\
0.0600000000000005	2.08976223657373e-06\\
0.0800000000000001	3.96782114912639e-06\\
0.0999999999999996	6.28603972836572e-06\\
0.119999999999999	8.90228141834085e-06\\
0.140000000000001	1.16589683955681e-05\\
0.16	1.43944388422824e-05\\
0.18	1.69547227586008e-05\\
0.199999999999999	1.92049354762275e-05\\
0.220000000000001	2.10395157559105e-05\\
0.24	2.23906094182894e-05\\
0.26	2.32340299447969e-05\\
0.279999999999999	2.35923732670926e-05\\
0.300000000000001	2.35350795344857e-05\\
0.32	2.31754081502004e-05\\
0.369999999999999	2.2015379934237e-05\\
0.390000000000001	2.19554053885076e-05\\
0.4	2.21148993925453e-05\\
0.41	2.24343493460764e-05\\
0.43	2.36477380913129e-05\\
0.449999999999999	2.57688403162319e-05\\
0.470000000000001	2.89386244327034e-05\\
0.5	3.58674899271137e-05\\
0.609999999999999	8.31493291818333e-05\\
0.65	0.000106465792895687\\
0.69	0.000131011361482888\\
0.73	0.000155683414564919\\
0.77	0.000179909246341049\\
0.82	0.000209758230376974\\
0.890000000000001	0.000253630368008817\\
1.14	0.000485874045701712\\
1.2	0.000550956243258769\\
1.25	0.000602460543207052\\
1.31	0.000660491642280976\\
1.4	0.000744914361457269\\
1.69	0.00107465401653353\\
1.78	0.00117907323722102\\
1.9	0.00131237840594185\\
2.16	0.00163099546998019\\
2.26	0.00174750706413688\\
2.35	0.00183386776072621\\
2.43	0.00188900715506498\\
2.5	0.00191402645684575\\
2.56	0.00191260320163092\\
2.62	0.00188509712501436\\
2.67	0.00183930814956257\\
2.72	0.00177090106732248\\
2.77	0.00167885201621102\\
2.81	0.00158788256719272\\
2.85	0.00148151278830701\\
2.89	0.00135997101414432\\
2.93	0.00122374115873827\\
2.96	0.00111240684777913\\
2.99	0.000993760046409808\\
3.02	0.000868449803843033\\
3.05	0.000737282030702045\\
3.08	0.000601209385010589\\
3.1	0.000508294265745279\\
3.12	0.000414009562253704\\
3.14	0.000318696430987984\\
3.16	0.000222690717061414\\
3.17	0.000174528871321308\\
3.18	0.000126311640328276\\
3.19	7.80743623035958e-05\\
3.2	2.98498317915801e-05\\
3.21	1.83320786636614e-05\\
3.22	6.64447761810153e-05\\
3.23	0.000114465292334716\\
3.25	0.000210157957034972\\
3.27	0.000305310436099535\\
3.29	0.000399898658237725\\
3.31	0.000493979448801475\\
3.34	0.000634478014921478\\
3.37	0.0007749826113959\\
3.4	0.000916673518882452\\
3.44	0.00110998772716739\\
3.49	0.00136406652647238\\
3.55	0.00169597777847266\\
3.61	0.00206329919158233\\
3.67	0.00246373885818357\\
3.73	0.00288551933279704\\
3.78	0.00323886594470533\\
3.83	0.00357613227815272\\
3.87	0.00381875204140487\\
3.91	0.00402015356142328\\
3.95	0.004161511610862\\
3.98	0.00421641212852445\\
4.01	0.00421950641673501\\
4.04	0.00416586523969672\\
4.07	0.00405360861531081\\
4.1	0.00388433852781345\\
4.13	0.0036629426260206\\
4.16	0.00339676027246933\\
4.19	0.00309428663505314\\
4.22	0.00276373255359577\\
4.25	0.00241180890565956\\
4.28	0.00204304754289296\\
4.31	0.00165982622132506\\
4.33	0.00139679816844619\\
4.35	0.00112792277032531\\
4.37	0.000853494193129881\\
4.39	0.000574123430958208\\
4.4	0.000432901578081509\\
4.41	0.00029089460314317\\
4.42	0.000148330585517397\\
4.43	5.47836043551796e-06\\
4.44	0.00013735213914515\\
4.45	0.000279810566938613\\
4.46	0.00042150787023238\\
4.48	0.000700886497509419\\
4.5	0.000971726855692033\\
4.52	0.00122991123512984\\
4.54	0.00147116071059536\\
4.56	0.00169125411302048\\
4.58	0.00188625134930298\\
4.6	0.00205270163346359\\
4.62	0.00218781903470735\\
4.64	0.00228961142352419\\
4.66	0.00235695400875629\\
4.68	0.00238960460240523\\
4.7	0.00238816380230091\\
4.72	0.00235398871807447\\
4.74	0.00228907291459327\\
4.76	0.00219590751880969\\
4.78	0.00207733856487575\\
4.8	0.00193643386434706\\
4.82	0.00177636924813014\\
4.84	0.00160033968481839\\
4.86	0.00141149618328163\\
4.88	0.00121290537115529\\
4.9	0.00100752572903548\\
4.92	0.000798193071587229\\
4.94	0.000587608011574047\\
4.96	0.000378319689035197\\
4.97	0.000274911414835879\\
4.98	0.000172702459119318\\
5	2.70749040475493e-05\\
5.01	0.000124160945001131\\
5.02	0.000219085300625374\\
5.03	0.000311650960577754\\
5.05	0.000489016812283808\\
5.07	0.000655087552795523\\
5.09	0.000809034283223445\\
5.11	0.000950371756944152\\
5.13	0.0010789429580127\\
5.16	0.00124823769406321\\
5.19	0.00139064881776139\\
5.22	0.00150847806086495\\
5.25	0.0016044144389427\\
5.28	0.00168115017387208\\
5.32	0.00175774975404165\\
5.36	0.00180923746613007\\
5.4	0.00183934158320807\\
5.45	0.00185179697963504\\
5.5	0.00184212846964219\\
5.57	0.00180326968475201\\
5.65	0.001735366291854\\
5.71	0.00166508086168687\\
5.75	0.00160224167868764\\
5.79	0.00152034258668136\\
5.82	0.00144303222962808\\
5.85	0.00134976016900946\\
5.88	0.00123887136237281\\
5.91	0.00110924666030587\\
5.94	0.000960337649199574\\
5.96	0.000850327840340833\\
5.98	0.000731851022955742\\
6	0.000605101951351752\\
6.02	0.000470338916606488\\
6.04	0.000327869554963036\\
6.05	0.000253851521213801\\
6.06	0.000178038428780929\\
6.07	0.00010047753530489\\
6.08	2.12173889285623e-05\\
6.1	0.000142200583063459\\
6.11	0.000226255488982133\\
6.13	0.000398790631899464\\
6.15	0.000576854029800233\\
6.17	0.000759973594500843\\
6.2	0.00104300768717148\\
6.23	0.00133433956694898\\
6.26	0.00163174416548391\\
6.29	0.00193264602567906\\
6.33	0.00233416203262568\\
6.37	0.0027287967917342\\
6.41	0.00310806204186774\\
6.45	0.00346339048896516\\
6.49	0.00378679437278478\\
6.53	0.00407146783803874\\
6.57	0.00431220581032466\\
6.61	0.00450556288803563\\
6.65	0.00464974517423251\\
6.69	0.00474429493455391\\
6.73	0.00478967344866869\\
6.77	0.00478686012018007\\
6.82	0.00471742735725638\\
6.86	0.00461077275525241\\
6.9	0.00446032421303171\\
6.94	0.00426794511994885\\
6.98	0.00403601590961953\\
7.02	0.00376754364916233\\
7.06	0.00346615535316486\\
7.1	0.0031359632464908\\
7.14	0.00278132820635073\\
7.18	0.00240657446297773\\
7.21	0.00211474008925834\\
7.24	0.00181537028426938\\
7.27	0.00150978467055611\\
7.3	0.00119910303786201\\
7.32	0.000989617376837852\\
7.34	0.00077851745314074\\
7.36	0.000566008932930708\\
7.37	0.000459284298255807\\
7.38	0.000352274534548997\\
7.39	0.000244999779701538\\
7.4	0.00013747941470865\\
7.41	2.97322233378509e-05\\
7.43	0.000186369528832802\\
7.44	0.000294688105251044\\
7.46	0.000511770959192291\\
7.48	0.000729326592060532\\
7.5	0.000947204271574082\\
7.52	0.00116524510273921\\
7.55	0.00149223875932985\\
7.58	0.00181859292099616\\
7.61	0.00214361691972481\\
7.65	0.00257358600207881\\
7.69	0.00299782247283001\\
7.73	0.00341409142296434\\
7.78	0.00391942756825411\\
7.83	0.00440298345570008\\
7.88	0.00485845622757915\\
7.93	0.00527838763081199\\
7.98	0.00565384041570396\\
8.03	0.00597422152128602\\
8.08	0.00622748830424461\\
8.13	0.00640093772410327\\
8.18	0.00648261033794659\\
8.22	0.0064754049661704\\
8.26	0.00640042346873486\\
8.3	0.00625690656295046\\
8.34	0.00604642417640244\\
8.38	0.00577277421363832\\
8.42	0.00544166467057912\\
8.46	0.00506028371516041\\
8.5	0.00463688233979227\\
8.54	0.00418046916681015\\
8.58	0.00370065564745783\\
8.62	0.00320761584461034\\
8.66	0.00271206727323224\\
8.7	0.00222516042633761\\
8.74	0.00175819153563352\\
8.77	0.00142766925541473\\
8.8	0.00111883552169674\\
8.83	0.000835507986473386\\
8.86	0.000580875841916349\\
8.88	0.000428299137962313\\
8.9	0.000290091797986315\\
8.92	0.000166598520664047\\
8.93	0.000110430215018278\\
8.94	5.79918679184629e-05\\
8.95	9.27889142658387e-06\\
8.96	3.5723244615292e-05\\
8.97	7.7038629234812e-05\\
8.98	0.000114700460737367\\
8.99	0.000148750552798447\\
9.01	0.00020622252959256\\
9.03	0.000249939650387023\\
9.05	0.000280486314955123\\
9.07	0.000298524913060828\\
9.08	0.000303076696798234\\
9.09	0.000304774278680269\\
9.1	0.000303712907941735\\
9.11	0.000299988832916643\\
9.12	0.000293698744935318\\
9.14	0.000273806463320944\\
9.16	0.000244800020646289\\
9.18	0.000207422416424742\\
9.2	0.00016238639190451\\
9.21	0.00013720942536238\\
9.22	0.000110368602664312\\
9.23	8.19424704743141e-05\\
9.24	5.20071880073373e-05\\
9.25	2.06366197784518e-05\\
9.26	1.20975380814343e-05\\
9.27	4.61255954833178e-05\\
9.28	8.13796822934819e-05\\
9.3	0.000155302252912588\\
9.32	0.000233350167902972\\
9.34	0.000315022558991404\\
9.36	0.000399825107568903\\
9.39	0.000531814434677291\\
9.42	0.000668034644596626\\
9.45	0.000806739496566297\\
9.48	0.000946140912904957\\
9.52	0.00113000070045491\\
9.56	0.00130786917938858\\
9.6	0.00147625518376844\\
9.64	0.00163263313632087\\
9.68	0.00177582441614031\\
9.73	0.00193715081980918\\
9.79	0.00211197133745445\\
9.94	0.00259329155722552\\
9.99	0.00283294490564658\\
10.04	0.00315220594131242\\
10.09	0.00357724438169352\\
10.15	0.00426032191695704\\
10.24	0.00569137636846026\\
10.34	0.00781561999113108\\
10.4	0.00924865461372172\\
10.45	0.0104349234401674\\
10.49	0.0113187017898464\\
10.53	0.0120941033674098\\
10.57	0.0127148285828214\\
10.61	0.0131383537381269\\
10.64	0.0133050645900785\\
10.67	0.0133299243509369\\
10.7	0.0132063011752441\\
10.73	0.0129326223861317\\
10.76	0.0125130410864291\\
10.79	0.0119579792950175\\
10.82	0.0112845143188006\\
10.86	0.0102447227325603\\
10.9	0.00911323266973996\\
10.96	0.00745789924736491\\
11	0.00655157730351601\\
11.03	0.00606639999539947\\
11.05	0.00586677262441043\\
11.07	0.00578551464695411\\
11.09	0.00583875077654774\\
11.11	0.00604194023083654\\
11.13	0.00640961354953432\\
11.15	0.00695511169843629\\
11.18	0.00813235226811776\\
11.22	0.0104206923325568\\
11.34	0.0224759979182592\\
11.39	0.0296442154781038\\
11.44	0.0377878235571664\\
11.49	0.0465904011947943\\
11.53	0.0538533493677826\\
11.57	0.0610855137585167\\
11.61	0.0680757548298099\\
11.65	0.0746252661341487\\
11.69	0.0805607187100227\\
11.73	0.0857464558975446\\
11.77	0.0900946845286583\\
11.81	0.0935726279615493\\
11.86	0.0967404771828645\\
11.91	0.0987655111865081\\
11.98	0.100248864553808\\
12.14	0.102896039587715\\
12.21	0.105712302684196\\
12.29	0.110557650717289\\
12.43	0.120003550552088\\
12.49	0.122400277296101\\
12.5	0.122607623632584\\
};
\addlegendentry{intrusive lifting ROM $4r=80$}

\addplot [color=orange!90!black, dashed, line width=2.0pt]
  table[row sep=crcr]{%
0.0199999999999996	1.48260032110556e-06\\
0.0299999999999994	5.93186671430886e-06\\
0.0399999999999991	1.33521666979504e-05\\
0.0600000000000005	3.71372151257673e-05\\
0.0800000000000001	7.29256686525331e-05\\
0.0999999999999996	0.00012083994562272\\
0.130000000000001	0.00021577094023778\\
0.16	0.000338843826802258\\
0.19	0.000490536549372337\\
0.220000000000001	0.000671178853648597\\
0.26	0.000957211302198992\\
0.300000000000001	0.00129438211885972\\
0.34	0.00168143662823859\\
0.380000000000001	0.00211662399912347\\
0.43	0.00272553966275919\\
0.48	0.00340413209495492\\
0.529999999999999	0.00415133624871414\\
0.59	0.0051392049717262\\
0.65	0.00622769787244798\\
0.720000000000001	0.00762177176723236\\
0.790000000000001	0.00913645380671115\\
0.859999999999999	0.0107501851167905\\
0.93	0.0124443171830677\\
1.01	0.0144760952743443\\
1.11	0.0172020237534776\\
1.23	0.0208396471043248\\
1.34	0.0245098596436241\\
1.44	0.0279920093438402\\
1.54	0.0314878475805967\\
1.65	0.0353009680763957\\
1.77	0.039419052965513\\
1.88	0.0430628746069698\\
1.98	0.0461045306730622\\
2.08	0.0487499221130681\\
2.18	0.0509253747847127\\
2.29	0.0527293184807918\\
2.39	0.0537812837844005\\
2.49	0.054209219777225\\
2.59	0.0539637069688936\\
2.69	0.0530260135934088\\
2.79	0.0514186283119358\\
2.89	0.0492246738807583\\
3	0.0463222238049562\\
3.15	0.0420339958138572\\
3.33	0.0374390644895527\\
3.41	0.0360340778173814\\
3.47	0.0354099890843216\\
3.53	0.0352364921128523\\
3.59	0.0355603476963828\\
3.65	0.0363937557088691\\
3.71	0.0377144388548084\\
3.78	0.0398031300213577\\
3.89	0.0439949294554186\\
4.04	0.050377553223741\\
4.12	0.0534289433803048\\
4.19	0.0556070353603583\\
4.26	0.0571818200043731\\
4.33	0.0580364601440299\\
4.4	0.0580994151018191\\
4.47	0.0574080177356093\\
4.55	0.0559218725798058\\
4.68	0.052810194945475\\
4.89	0.0475286759058\\
5.06	0.0437139674378607\\
5.15	0.0423906206764538\\
5.23	0.041759696848032\\
5.33	0.0415675054570212\\
5.69	0.0414819589084407\\
5.82	0.0408032735188936\\
5.91	0.0398400647121373\\
5.99	0.0384803401378006\\
6.06	0.0368587788673805\\
6.13	0.0348555256769348\\
6.2	0.0325335398211086\\
6.27	0.0299855845553026\\
6.35	0.0269477594975313\\
6.47	0.0225792163474045\\
6.57	0.0195686765662868\\
6.63	0.0181890119430293\\
6.68	0.0173253366407607\\
6.73	0.0167380302607627\\
6.77	0.0164732107622291\\
6.81	0.0163936077816619\\
6.85	0.01650068835663\\
6.89	0.016794594582208\\
6.94	0.0174217201111275\\
6.99	0.0183278256398764\\
7.05	0.0197555319420011\\
7.14	0.022477393701499\\
7.31	0.0287515449600716\\
7.41	0.0327148602009904\\
7.51	0.0367172003201489\\
7.61	0.0406266252210717\\
7.7	0.043888464079618\\
7.79	0.0467634770632998\\
7.89	0.0494957125547596\\
8.02	0.0525613803579868\\
8.15	0.0551480722740052\\
8.25	0.0565076638405833\\
8.35	0.057123386045277\\
8.5	0.0571460957666481\\
8.95	0.0564721208143598\\
9.2	0.0548959145463607\\
9.28	0.0537559354326361\\
9.36	0.0519653943072192\\
9.44	0.049589432259581\\
9.55	0.0458422768574525\\
9.65	0.0421999945030709\\
9.71	0.0396862993905634\\
9.76	0.0372360994295384\\
9.81	0.0344674946924596\\
9.9	0.0298542407123303\\
9.92	0.0293034550691613\\
9.94	0.0291259494130875\\
9.96	0.0294502717556545\\
9.98	0.0304254922291923\\
10	0.0322213149050138\\
10.02	0.035027640727218\\
10.04	0.039053522230306\\
10.07	0.047878681945098\\
10.11	0.066145529450175\\
10.25	0.216912053779344\\
10.3	0.313344078878685\\
10.35	0.434713190385109\\
10.4	0.579498724751257\\
10.45	0.743384131933426\\
10.49	0.883454955110222\\
10.53	1.0261358476415\\
10.57	1.16545616667148\\
10.61	1.29477498356638\\
10.65	1.40715925081822\\
10.68	1.47623875504615\\
10.71	1.52932680975117\\
10.74	1.56410739585252\\
10.77	1.57870564061555\\
10.8	1.57179417406686\\
10.83	1.54268479502612\\
10.86	1.49140055766874\\
10.89	1.41872406007532\\
10.92	1.3262186270334\\
10.95	1.21622019121663\\
10.98	1.09179893775789\\
11.01	0.956691125222959\\
11.04	0.815202862294423\\
11.07	0.672088932542769\\
11.1	0.532410949273073\\
11.13	0.401380123899949\\
11.16	0.284190689545366\\
11.19	0.185850494491741\\
11.22	0.111015441466382\\
11.26	0.0548789643472414\\
11.27	0.0495035396348667\\
11.28	0.0478098667812511\\
11.29	0.0498895369890578\\
11.3	0.0558235433322891\\
11.32	0.0795238007457615\\
11.37	0.209742596069951\\
11.4	0.336923273012396\\
11.43	0.499954197677992\\
11.46	0.696995402339036\\
11.49	0.925350529602641\\
11.53	1.27244224420523\\
11.57	1.6589553665008\\
11.61	2.07330924435116\\
11.65	2.50292473926453\\
11.69	2.93481673015692\\
11.73	3.35616634990614\\
11.77	3.75484189204494\\
11.81	4.11984418375693\\
11.85	4.44166007036292\\
11.89	4.7125156928655\\
11.93	4.92652875893579\\
11.97	5.0797654604691\\
12.01	5.17021271563915\\
12.05	5.19767983775589\\
12.09	5.16364555516855\\
12.13	5.07106666032226\\
12.17	4.92416368266705\\
12.21	4.72819714653293\\
12.25	4.48924550079193\\
12.3	4.14033275030326\\
12.35	3.74863422128185\\
12.4	3.32851922275058\\
12.45	2.89435336895257\\
12.5	2.46009183142991\\
};
\addlegendentry{sp Lift \& Learn ROM $4r=80$}

\addplot [color=black, line width=3.0pt, forget plot]
  table[row sep=crcr]{%
10	1e-08\\
10	100\\
};
\addplot [color=purple!80!black, dotted, line width=2.0pt]
  table[row sep=crcr]{%
0.0199999999999996	5.83254535740707e-08\\
0.0299999999999994	2.26868304889649e-07\\
0.0399999999999991	4.86456462749628e-07\\
0.0500000000000007	8.05553530881296e-07\\
0.0600000000000005	1.14087838956038e-06\\
0.0700000000000003	1.43826052735676e-06\\
0.0800000000000001	1.63371814560378e-06\\
0.0899999999999999	1.65473739471054e-06\\
0.0999999999999996	1.42174372231238e-06\\
0.109999999999999	8.49732640199361e-07\\
0.119999999999999	1.49957327812444e-07\\
0.130000000000001	1.66776176447456e-06\\
0.140000000000001	3.79392169852508e-06\\
0.16	1.02191489761026e-05\\
0.18	2.00665546464733e-05\\
0.199999999999999	3.38449766917619e-05\\
0.23	6.25120672360934e-05\\
0.26	0.00010068502869326\\
0.290000000000001	0.000146878856732655\\
0.32	0.000198017138297291\\
0.35	0.000249520217585814\\
0.369999999999999	0.000281228600747272\\
0.390000000000001	0.000308596579131699\\
0.41	0.000329599329580378\\
0.43	0.000342203726465959\\
0.44	0.000344739029211996\\
0.449999999999999	0.000344450063005933\\
0.460000000000001	0.000341116640032851\\
0.470000000000001	0.000334533064987045\\
0.48	0.000324510367451876\\
0.49	0.00031087837363084\\
0.5	0.000293487605272458\\
0.51	0.000272210985146376\\
0.52	0.00024694532498761\\
0.529999999999999	0.000217612587566328\\
0.540000000000001	0.000184160908657941\\
0.550000000000001	0.000146565368538723\\
0.56	0.000104828513399116\\
0.57	5.89806158008287e-05\\
0.58	9.07968631727272e-06\\
0.59	4.47887691865617e-05\\
0.6	0.000102512229717831\\
0.609999999999999	0.000163951877119689\\
0.630000000000001	0.000297298760206831\\
0.65	0.000443233258585679\\
0.67	0.00059981757600981\\
0.699999999999999	0.000849729692012626\\
0.73	0.00111013603107131\\
0.76	0.0013719565030362\\
0.790000000000001	0.00162590237308868\\
0.82	0.0018628873961381\\
0.85	0.00207435557097305\\
0.880000000000001	0.00225251068023463\\
0.91	0.00239044807769351\\
0.94	0.00248220266508433\\
0.960000000000001	0.00251519065145476\\
0.98	0.00252412366732678\\
1	0.0025078977222438\\
1.02	0.00246557176065326\\
1.04	0.00239636490551093\\
1.06	0.00229965695158243\\
1.08	0.00217499240572797\\
1.1	0.00202208791867306\\
1.12	0.00184084250693777\\
1.14	0.00163134956136855\\
1.16	0.00139390934332077\\
1.18	0.00112904044979768\\
1.2	0.000837488666870195\\
1.21	0.000682003132525643\\
1.22	0.000520231650453206\\
1.23	0.00035233251790196\\
1.24	0.000178478058549081\\
1.25	1.1457835251349e-06\\
1.26	0.000186339964769286\\
1.27	0.000376892990210763\\
1.28	0.00057258167607397\\
1.3	0.000978420272111905\\
1.32	0.00140187287410081\\
1.34	0.00184083726752724\\
1.37	0.00252359174235753\\
1.4	0.00322890455646266\\
1.43	0.00394955662953107\\
1.46	0.00467889785253828\\
1.5	0.00565504882023706\\
1.54	0.00662596022340657\\
1.58	0.0075850116220816\\
1.63	0.00876314894290662\\
1.68	0.00991957985291554\\
1.74	0.0112861808777849\\
1.81	0.012864035544308\\
1.88	0.0144318154033208\\
1.96	0.0162084396159651\\
2.05	0.0181749666008458\\
2.14	0.0200915714680381\\
2.23	0.0219247207535562\\
2.31	0.023418168476499\\
2.39	0.0246853567626033\\
2.46	0.0255276811341147\\
2.53	0.0260714983632045\\
2.6	0.0263090703287416\\
2.68	0.0262484510906044\\
2.77	0.0258610249601177\\
2.88	0.0250908568892949\\
3	0.0239907947217172\\
3.11	0.0227438894658611\\
3.21	0.021378484909064\\
3.3	0.0199477185269895\\
3.39	0.0183667206953283\\
3.6	0.0150044325727565\\
3.66	0.0144104238471046\\
3.72	0.0140213264965194\\
3.8	0.0137167615250328\\
3.91	0.013304183599646\\
3.97	0.0129140500046684\\
4.02	0.0124589072798936\\
4.07	0.0118843057828595\\
4.13	0.0110632040846089\\
4.2	0.0100140752711493\\
4.29	0.0087963241710793\\
4.33	0.00840471477684621\\
4.37	0.00814872481480159\\
4.4	0.00805546834447668\\
4.43	0.00804794325640723\\
4.46	0.00812139829779553\\
4.5	0.00832790354880963\\
4.56	0.00879857324173245\\
4.67	0.00976891739260317\\
4.74	0.0102743838614337\\
4.89	0.0112158586926318\\
4.98	0.0117356216219301\\
5.03	0.0118958195203562\\
5.08	0.011904390707512\\
5.13	0.0117613245993132\\
5.21	0.0113576023298534\\
5.27	0.0111034651840779\\
5.32	0.0110373635685191\\
5.36	0.011112393544754\\
5.4	0.0113057972684208\\
5.45	0.0117000333504362\\
5.51	0.012357208716969\\
5.59	0.0134721745139712\\
5.71	0.0155717558917633\\
6.01	0.0225563944723294\\
6.24	0.0293080107455298\\
6.32	0.0317368710035109\\
6.39	0.0335856942640719\\
6.45	0.0348342655719917\\
6.52	0.0358455356831844\\
6.6	0.0364954227335079\\
6.72	0.0369084752592921\\
6.85	0.0369742242343283\\
6.93	0.0365902175160878\\
6.99	0.0359080149532735\\
7.05	0.0348118439456473\\
7.11	0.0332997674425925\\
7.17	0.0314297335804782\\
7.23	0.0292961816016977\\
7.3	0.026607286424117\\
7.37	0.0238271586111448\\
7.45	0.0206641867700655\\
7.54	0.0173076093647644\\
7.65	0.0139454781124092\\
7.7	0.0128594832343442\\
7.74	0.012241714559996\\
7.77	0.0119355647826047\\
7.8	0.0117677560354379\\
7.83	0.0117389091023651\\
7.86	0.0118481355762819\\
7.89	0.0120941987703769\\
7.92	0.0124767839307322\\
7.96	0.0132027769374239\\
8	0.0141877316839367\\
8.05	0.0158194501936574\\
8.1	0.0179555345081015\\
8.17	0.0219139368905213\\
8.36	0.038328867951599\\
8.42	0.0446749082226844\\
8.48	0.0512425439241622\\
8.55	0.0591304960604556\\
8.65	0.071258781077265\\
8.82	0.096253765570744\\
8.89	0.107536792707061\\
8.95	0.116653940330883\\
9	0.123311903642876\\
9.05	0.128768528581788\\
9.1	0.132871575594648\\
9.16	0.136098448050467\\
9.23	0.138134783083051\\
9.42	0.142824150077916\\
9.5	0.147162979526465\\
9.58	0.153421360427626\\
9.65	0.16088370074819\\
9.71	0.169535200448582\\
9.76	0.179093977753813\\
9.81	0.191499651826827\\
9.86	0.207361355002768\\
9.92	0.231549481872469\\
10	0.272954539790939\\
10.15	0.379461140471376\\
10.24	0.469036183375751\\
10.3	0.548865948793685\\
10.36	0.654386208985363\\
10.42	0.795777196202267\\
10.5	1.05744468222663\\
10.65	1.81396825067166\\
10.71	2.20235377396738\\
10.77	2.6155280260759\\
10.82	2.9591177584625\\
10.87	3.2817073954931\\
10.91	3.51071912462569\\
10.95	3.7023319835387\\
10.99	3.84650371727429\\
11.03	3.93442622851874\\
11.06	3.9591641298822\\
11.09	3.94617004883304\\
11.12	3.89417965881434\\
11.15	3.80285841562643\\
11.18	3.67286206100548\\
11.21	3.50587085793485\\
11.24	3.3045958688627\\
11.27	3.07275644292107\\
11.3	2.81502895965585\\
11.33	2.53696775700857\\
11.36	2.24490002301093\\
11.39	1.94579722759664\\
11.42	1.64712638857138\\
11.45	1.35668508196871\\
11.48	1.08242460261326\\
11.51	0.832266040289163\\
11.54	0.613914249138491\\
11.58	0.3848029954507\\
11.62	0.240810816981897\\
11.64	0.204963295695229\\
11.65	0.196708583289104\\
11.66	0.19510736180519\\
11.67	0.200298624986134\\
11.68	0.212407618317519\\
11.7	0.257808985180303\\
11.73	0.380099039846382\\
11.78	0.731113379632568\\
11.82	1.14343559267958\\
11.86	1.66748237992825\\
11.9	2.29449613115405\\
11.94	3.01229830467865\\
11.98	3.80572857636141\\
12.02	4.6571815929543\\
12.06	5.5472131994319\\
12.1	6.45518638858107\\
12.14	7.35992745717177\\
12.18	8.24036482566951\\
12.22	9.07612644227179\\
12.26	9.84807630072651\\
12.3	10.5387759417375\\
12.34	11.1328624388568\\
12.38	11.6173398636006\\
12.42	11.9817862007028\\
12.46	12.2184818364036\\
12.5	12.3224688523482\\
};
\addlegendentry{HOpInf with spDEIM ROM $2r=40$}

\end{axis}

\begin{axis}[%
width=1.273\fheight,
height=0.955\fheight,
at={(-0.169\fheight,-0.133\fheight)},
scale only axis,
xmin=0,
xmax=1,
ymin=0,
ymax=1,
axis line style={draw=none},
ticks=none,
axis x line*=bottom,
axis y line*=left
]
\end{axis}
\end{tikzpicture}%

%% file: figures/kgz/finest/q_train.tex
%
%
\definecolor{mycolor1}{rgb}{1.00000,0.00000,1.00000}%
\begin{tikzpicture}

\begin{axis}[%
width=0.683\fheight,
height=0.65\fheight,
at={(0\fheight,0\fheight)},
scale only axis,
xmin=0,
xmax=150,
xlabel style={font=\color{white!15!black}},
xlabel={Reduced dimension $6r$},
ymode=log,
ymin=1e-05,
ymax=1,
yminorticks=true,
ylabel style={font=\color{white!15!black}},
ylabel={Relative state error in $\psi$},
axis background/.style={fill=white},
xmajorgrids,
xmajorgrids,
ymajorgrids,
legend style={draw=none, legend columns=-1},
legend style={at={(2,1.35)}, anchor=south west, legend cell align=left, align=left, draw=white!15!black}
]
\addplot [
  color=blue!70!black, 
  only marks,
  line width=2.0pt,
  mark=o,
  mark size=3.0pt,
  mark options={solid, blue!70!black}
]
  table[row sep=crcr]{%
30	0.097632676223942\\
60	0.0109450050074902\\
90	0.000645931701140772\\
120	0.00022907624093048\\
150	5.53977391437984e-05\\
165	4.21991557749112e-05\\
};
\addlegendentry{intrusive lifting}

\addplot [
  color=orange!90!black, 
  only marks,
  line width=2.0pt,
  mark=triangle*,
  mark size=3.0pt,
  mark options={solid, orange!90!black}
]
  table[row sep=crcr]{%
30	0.0771209864224302\\
60	0.00961195040712299\\
90	0.00410149976032425\\
120	0.00297044456162761\\
150	0.0029716988039546\\
};
\addlegendentry{sp Lift \& Learn}

\end{axis}

\begin{axis}[%
width=1.227\fheight,
height=0.723\fheight,
at={(-0.16\fheight,-0.08\fheight)},
scale only axis,
xmin=0,
xmax=1,
ymin=0,
ymax=1,
axis line style={draw=none},
ticks=none,
axis x line*=bottom,
axis y line*=left
]
\end{axis}
\end{tikzpicture}%

%% file: figures/kgz/finest/phi_train.tex
%
%
\definecolor{mycolor1}{rgb}{1.00000,0.00000,1.00000}%
\begin{tikzpicture}

\begin{axis}[%
width=0.683\fheight,
height=0.65\fheight,
at={(0\fheight,0\fheight)},
scale only axis,
xmin=0,
xmax=150,
xlabel style={font=\color{white!15!black}},
xlabel={Reduced dimension $6r$},
ymode=log,
ymin=1e-5,
ymax=1,
yminorticks=true,
ylabel style={font=\color{white!15!black}},
ylabel={Relative state error in $\phi$},
axis background/.style={fill=white},
xmajorgrids,
ymajorgrids,
legend style={draw=none, legend columns=-1},
legend style={at={(0.75,1.1)}, anchor=south west, legend cell align=left, align=left, draw=white!15!black}
]
\addplot [
  color=blue!70!black, 
  only marks,
  line width=2.0pt,
  mark=o,
  mark size=3.0pt,
  mark options={solid, blue!70!black}
]
  table[row sep=crcr]{%
30	0.140215223699327\\
60	0.058967236471198\\
90	0.0150084692853283\\
120	0.00571783789007633\\
150	0.00173415616775359\\
165	0.000867673061806839\\
};

\addplot [
  color=orange!90!black, 
  only marks,
  line width=2.0pt,
  mark=triangle*,
  mark size=3.0pt,
  mark options={solid, orange!90!black}
]
  table[row sep=crcr]{%
30	0.139879479493397\\
60	0.0504610894153922\\
90	0.275504732545312\\
120	0.00508609536310761\\
150	0.00227249775099045\\
};

\end{axis}
\end{tikzpicture}%

%% file: figures/kgz/finest/q_test.tex
%
%
\definecolor{mycolor1}{rgb}{1.00000,0.00000,1.00000}%
\begin{tikzpicture}

\begin{axis}[%
width=0.683\fheight,
height=0.65\fheight,
at={(0\fheight,0\fheight)},
scale only axis,
xmin=0,
xmax=150,
xlabel style={font=\color{white!15!black}},
xlabel={Reduced dimension $6r$},
ymode=log,
ymin=0.01,
ymax=10,
yminorticks=true,
ylabel style={font=\color{white!15!black}},
ylabel={Relative state error in $\psi$},
axis background/.style={fill=white},
xmajorgrids,
ymajorgrids,
legend style={at={(0.243,0.242)}, anchor=south west, legend cell align=left, align=left, draw=white!15!black}
]
\addplot [
  color=blue!70!black, 
  only marks,
  line width=2.0pt,
  mark=o,
  mark size=3.0pt,
  mark options={solid, blue!70!black}
]
  table[row sep=crcr]{%
30	0.333969451510639\\
60	0.177859571523406\\
90	0.109319283302549\\
120	0.0772185752599888\\
150	0.0524104004671025\\
165	0.0431307700631699\\
};


\addplot [
  color=orange!90!black, 
  only marks,
  line width=2.0pt,
  mark=triangle*,
  mark size=3.0pt,
  mark options={solid, orange!90!black}
]
  table[row sep=crcr]{%
30	0.318688330578915\\
60	0.160933982692348\\
90	0.138869724997589\\
120	0.0815959917154573\\
150	0.0601573553333776\\
};

\end{axis}
\end{tikzpicture}%

%% file: figures/kgz/finest/phi_test.tex
%
%
\definecolor{mycolor1}{rgb}{1.00000,0.00000,1.00000}%
\begin{tikzpicture}

\begin{axis}[%
width=0.683\fheight,
height=0.65\fheight,
at={(0\fheight,0\fheight)},
scale only axis,
xmin=0,
xmax=150,
xlabel style={font=\color{white!15!black}},
xlabel={Reduced dimension $6r$},
ymode=log,
ymin=0.01,
ymax=10,
yminorticks=true,
ylabel style={font=\color{white!15!black}},
ylabel={Relative state error in $\phi$},
axis background/.style={fill=white},
xmajorgrids,
ymajorgrids,
legend style={at={(0.22,0.222)}, anchor=south west, legend cell align=left, align=left, draw=white!15!black}
]
\addplot [
  color=blue!70!black, 
  only marks,
  line width=2.0pt,
  mark=o,
  mark size=3.0pt,
  mark options={solid, blue!70!black}
]
  table[row sep=crcr]{%
30	0.74202777726921\\
60	0.509024955419476\\
90	0.423784664574739\\
120	0.382584946028884\\
150	0.349676817326955\\
165	0.325878506991065\\
};

\addplot [
  color=orange!90!black, 
  only marks,
  line width=2.0pt,
  mark=triangle*,
  mark size=3.0pt,
  mark options={solid, orange!90!black}
]
  table[row sep=crcr]{%
30	0.764012944359469\\
60	0.528563670243657\\
90	6.75220606064465\\
120	0.43465061428603\\
150	0.410102258073255\\
};

\end{axis}

\begin{axis}[%
width=1.322\fheight,
height=0.991\fheight,
at={(-0.212\fheight,-0.167\fheight)},
scale only axis,
xmin=0,
xmax=1,
ymin=0,
ymax=1,
axis line style={draw=none},
ticks=none,
axis x line*=bottom,
axis y line*=left
]
\end{axis}
\end{tikzpicture}%

%% file: figures/kgz/finest/proj.tex
%
%
\definecolor{mycolor1}{rgb}{0.00000,0.44700,0.74100}%
\definecolor{mycolor2}{rgb}{0.85000,0.32500,0.09800}%
\begin{tikzpicture}

\begin{axis}[%
width=0.983\fheight,
height=0.65\fheight,
at={(0\fheight,0\fheight)},
scale only axis,
xmin=0,
xmax=5,
xlabel style={font=\color{white!15!black}},
xlabel={Time $t$},
ymode=log,
ymin=1e-08,
ymax=1,
yminorticks=true,
ylabel style={font=\color{white!15!black}},
ylabel={Projection error},
axis background/.style={fill=white},
xmajorgrids,
ymajorgrids,
yminorgrids,
legend style={at={(0.1,1.05)}, anchor=south west, legend cell align=left, align=left, draw=white!15!black}
]
\addplot [color=mycolor1, dotted, line width=2.0pt]
  table[row sep=crcr]{%
0.00999999999999979	5.4827215968769e-06\\
0.0199999999999996	5.42304394504227e-06\\
0.0300000000000002	5.24525709300256e-06\\
0.04	4.95308333389512e-06\\
0.0499999999999998	4.55267790862779e-06\\
0.0599999999999996	4.05257943522734e-06\\
0.0700000000000003	3.46370942647087e-06\\
0.0800000000000001	2.79959923439731e-06\\
0.0899999999999999	2.07756093610738e-06\\
0.0999999999999996	1.32488854544221e-06\\
0.11	6.33435343575519e-07\\
0.12	6.34545726418305e-07\\
0.13	1.34053191815896e-06\\
0.14	2.12274840462133e-06\\
0.15	2.88799246847012e-06\\
0.16	3.60751554480878e-06\\
0.17	4.26217631149375e-06\\
0.18	4.83626458721865e-06\\
0.19	5.31639057274894e-06\\
0.2	5.69133309544199e-06\\
0.21	5.95212899270213e-06\\
0.22	6.09222519379423e-06\\
0.23	6.10764052130004e-06\\
0.24	5.99712520245496e-06\\
0.25	5.76232945157454e-06\\
0.26	5.40802138184408e-06\\
0.27	4.94245272839173e-06\\
0.28	4.37811486887688e-06\\
0.29	3.73353157914694e-06\\
0.3	3.03800315561572e-06\\
0.32	1.77615558379902e-06\\
0.33	1.57801534620179e-06\\
0.34	1.9187790747668e-06\\
0.35	2.59955532730884e-06\\
0.36	3.40479607104469e-06\\
0.37	4.23224169382193e-06\\
0.38	5.03071669833036e-06\\
0.39	5.7679841093783e-06\\
0.4	6.42013656384494e-06\\
0.41	6.96804910500106e-06\\
0.42	7.39618347186592e-06\\
0.43	7.69226565970159e-06\\
0.44	7.84733203059629e-06\\
0.45	7.85596636983137e-06\\
0.46	7.71668603934828e-06\\
0.47	7.4325225367474e-06\\
0.48	7.01194548540495e-06\\
0.49	6.47046323425061e-06\\
0.5	5.83360468044491e-06\\
0.51	5.14270568504235e-06\\
0.53	3.91429807017212e-06\\
0.54	3.64549698342645e-06\\
0.55	3.79446489147366e-06\\
0.56	4.35692259650837e-06\\
0.58	6.21404744686599e-06\\
0.59	7.27784759684581e-06\\
0.6	8.33206165613267e-06\\
0.61	9.32742618867047e-06\\
0.62	1.0225383080891e-05\\
0.63	1.09944954697439e-05\\
0.64	1.16089228301316e-05\\
0.65	1.20479744686999e-05\\
0.66	1.2296258907568e-05\\
0.67	1.23442145333988e-05\\
0.68	1.21889581465293e-05\\
0.69	1.18354995219821e-05\\
0.7	1.12984902931872e-05\\
0.71	1.06048382091854e-05\\
0.72	9.79771958915744e-06\\
0.75	7.5028755578771e-06\\
0.76	7.19360447974416e-06\\
0.77	7.31568846259247e-06\\
0.78	7.87570152014049e-06\\
0.8	9.88489766383791e-06\\
0.81	1.10699458703521e-05\\
0.82	1.22314997877507e-05\\
0.83	1.32925100678503e-05\\
0.84	1.41954613494937e-05\\
0.85	1.48986390419675e-05\\
0.86	1.53738471139729e-05\\
0.87	1.56050741865887e-05\\
0.88	1.55876459543446e-05\\
0.89	1.53275894574653e-05\\
0.9	1.48410974800709e-05\\
0.91	1.41541032379447e-05\\
0.92	1.33020664483499e-05\\
0.93	1.23301312129852e-05\\
0.95	1.02600420637504e-05\\
0.97	8.52528930372356e-06\\
0.98	7.99740644259519e-06\\
0.99	7.78137799159668e-06\\
1	7.88000927415296e-06\\
1.01	8.23783530539846e-06\\
1.05	1.05369587500459e-05\\
1.06	1.09951717600417e-05\\
1.07	1.13300055883007e-05\\
1.08	1.15246805704264e-05\\
1.09	1.15708045685013e-05\\
1.1	1.14668767161415e-05\\
1.11	1.12172776233492e-05\\
1.12	1.08316200119727e-05\\
1.13	1.03243953339608e-05\\
1.14	9.7149083532255e-06\\
1.15	9.02754005794395e-06\\
1.17	7.54645704808115e-06\\
1.19	6.21074330453465e-06\\
1.2	5.73314007844134e-06\\
1.21	5.45436175081172e-06\\
1.22	5.40139330388757e-06\\
1.23	5.56083649451319e-06\\
1.24	5.88370906661597e-06\\
1.27	7.21666127369286e-06\\
1.28	7.61876382888042e-06\\
1.29	7.94633089916791e-06\\
1.3	8.18103166382454e-06\\
1.31	8.31123507435702e-06\\
1.32	8.33089297689657e-06\\
1.33	8.23882568872201e-06\\
1.34	8.03830230664388e-06\\
1.35	7.73685226477813e-06\\
1.36	7.34629829730911e-06\\
1.37	6.88303932639079e-06\\
1.38	6.36865198484308e-06\\
1.4	5.30486483777764e-06\\
1.41	4.83399057061925e-06\\
1.42	4.4679258089643e-06\\
1.43	4.25413902106987e-06\\
1.44	4.22142611925973e-06\\
1.45	4.36494796255697e-06\\
1.46	4.64774447384991e-06\\
1.49	5.82358429502499e-06\\
1.5	6.18556181997788e-06\\
1.51	6.48600189594967e-06\\
1.52	6.70832004610982e-06\\
1.53	6.84155658931242e-06\\
1.54	6.87950652249899e-06\\
1.55	6.82017678277868e-06\\
1.56	6.66548155377778e-06\\
1.57	6.42112724164024e-06\\
1.58	6.09668449177138e-06\\
1.59	5.70589116806572e-06\\
1.6	5.26727472223443e-06\\
1.62	4.35129369871693e-06\\
1.63	3.94566259097321e-06\\
1.64	3.63548355623492e-06\\
1.65	3.46654640460906e-06\\
1.66	3.46580684871721e-06\\
1.67	3.62604164933618e-06\\
1.68	3.90935826793659e-06\\
1.7	4.64912439114422e-06\\
1.71	5.02239054877755e-06\\
1.72	5.35811461552331e-06\\
1.73	5.6363977963034e-06\\
1.74	5.8433078059187e-06\\
1.75	5.96972150565206e-06\\
1.76	6.01059180836685e-06\\
1.77	5.96451421111464e-06\\
1.78	5.83350231409226e-06\\
1.79	5.62292920312085e-06\\
1.8	5.34163614238989e-06\\
1.81	5.00225222417087e-06\\
1.82	4.62180537667567e-06\\
1.86	3.23971545220187e-06\\
1.87	3.11539826513859e-06\\
1.88	3.13977311472426e-06\\
1.89	3.30287358415922e-06\\
1.9	3.56966301998237e-06\\
1.92	4.24335987495296e-06\\
1.93	4.57873896866936e-06\\
1.94	4.87879791127539e-06\\
1.95	5.12618004256565e-06\\
1.96	5.30860699445497e-06\\
1.97	5.41793806983844e-06\\
1.98	5.44957807269628e-06\\
1.99	5.40213276979793e-06\\
2	5.27724072338341e-06\\
2.01	5.079550882573e-06\\
2.02	4.81685554885873e-06\\
2.03	4.5004302502065e-06\\
2.04	4.14567517332316e-06\\
2.06	3.41013180277308e-06\\
2.07	3.0916098068118e-06\\
2.08	2.85902739887523e-06\\
2.09	2.75150692107788e-06\\
2.1	2.78910205942411e-06\\
2.11	2.96056141906266e-06\\
2.13	3.55449414260781e-06\\
2.14	3.89617360337396e-06\\
2.15	4.22556221268675e-06\\
2.16	4.52075785050234e-06\\
2.17	4.76578435351546e-06\\
2.18	4.94925920160364e-06\\
2.19	5.06351816897108e-06\\
2.2	5.10409087257975e-06\\
2.21	5.06941270254947e-06\\
2.22	4.96070027836093e-06\\
2.23	4.78196092538631e-06\\
2.24	4.54014680491896e-06\\
2.25	4.2455063849409e-06\\
2.26	3.91223231991611e-06\\
2.28	3.21317701962617e-06\\
2.29	2.90692281594704e-06\\
2.3	2.6815974374948e-06\\
2.31	2.57695921176295e-06\\
2.32	2.61465188451268e-06\\
2.33	2.78482122666926e-06\\
2.35	3.37475154579992e-06\\
2.36	3.71481046146194e-06\\
2.37	4.04353810760009e-06\\
2.38	4.3393784152105e-06\\
2.39	4.58648807568249e-06\\
2.4	4.7734136529836e-06\\
2.41	4.89224487402951e-06\\
2.42	4.93812109963226e-06\\
2.43	4.90897141378183e-06\\
2.44	4.8054152227605e-06\\
2.45	4.63079433899807e-06\\
2.46	4.39135168990911e-06\\
2.47	4.09662057750263e-06\\
2.48	3.76015231780647e-06\\
2.5	3.04460186298864e-06\\
2.51	2.72730703695631e-06\\
2.52	2.49432494161287e-06\\
2.53	2.39211597299659e-06\\
2.54	2.44651159937978e-06\\
2.55	2.64501036691356e-06\\
2.59	4.03260920972727e-06\\
2.6	4.35536617174771e-06\\
2.61	4.62675183349205e-06\\
2.62	4.83517790274402e-06\\
2.63	4.97249912885885e-06\\
2.64	5.03354751848393e-06\\
2.65	5.01590966520538e-06\\
2.66	4.91986325708805e-06\\
2.67	4.74844286048802e-06\\
2.68	4.50765226636648e-06\\
2.69	4.20689448658399e-06\\
2.7	3.85976468052636e-06\\
2.71	3.48544637330694e-06\\
2.73	2.77420313591002e-06\\
2.74	2.52473394315011e-06\\
2.75	2.41520801742802e-06\\
2.76	2.47634054474051e-06\\
2.77	2.69507529193258e-06\\
2.8	3.82031803210498e-06\\
2.81	4.21112529533074e-06\\
2.82	4.56375951305333e-06\\
2.83	4.86119147013299e-06\\
2.84	5.09086580121083e-06\\
2.85	5.24379932838108e-06\\
2.86	5.3141148852876e-06\\
2.87	5.29883648692235e-06\\
2.88	5.19785412164368e-06\\
2.89	5.01402841679546e-06\\
2.9	4.75346010779655e-06\\
2.91	4.42601523285744e-06\\
2.92	4.04629582506597e-06\\
2.93	3.63539020085522e-06\\
2.95	2.85575106004823e-06\\
2.96	2.59069805679618e-06\\
2.97	2.49222060800828e-06\\
2.98	2.59395369836604e-06\\
2.99	2.87355398990057e-06\\
3.02	4.19991045388681e-06\\
3.03	4.6483337547737e-06\\
3.04	5.05064284286316e-06\\
3.05	5.38858604481123e-06\\
3.06	5.64854780015642e-06\\
3.07	5.82065573041859e-06\\
3.08	5.89835393347978e-06\\
3.09	5.87824533552955e-06\\
3.1	5.76010668898944e-06\\
3.11	5.54705013020264e-06\\
3.12	5.24586796765921e-06\\
3.13	4.86767789709521e-06\\
3.14	4.42911558163509e-06\\
3.15	3.95452957414075e-06\\
3.17	3.05813819214477e-06\\
3.18	2.76269093409513e-06\\
3.19	2.6710619506692e-06\\
3.2	2.81859396211301e-06\\
3.21	3.16851026931037e-06\\
3.23	4.17770197009144e-06\\
3.24	4.71372543605895e-06\\
3.25	5.21600203725277e-06\\
3.26	5.65757067462593e-06\\
3.27	6.01839652950671e-06\\
3.28	6.28356528241996e-06\\
3.29	6.44239112967469e-06\\
3.3	6.48802957969415e-06\\
3.31	6.41737567753781e-06\\
3.32	6.23115149730261e-06\\
3.33	5.93417461967682e-06\\
3.34	5.53588993625351e-06\\
3.35	5.05138669593445e-06\\
3.36	4.50338401320407e-06\\
3.38	3.37277303123488e-06\\
3.39	2.9255100493022e-06\\
3.4	2.69438926458276e-06\\
3.41	2.76515064661292e-06\\
3.42	3.12030612508304e-06\\
3.44	4.28209492625437e-06\\
3.45	4.9145056755812e-06\\
3.46	5.50815056075945e-06\\
3.47	6.02881762191909e-06\\
3.48	6.45151549856266e-06\\
3.49	6.75775109966782e-06\\
3.5	6.9342863524996e-06\\
3.51	6.97263602784508e-06\\
3.52	6.86894457475759e-06\\
3.53	6.62408689342472e-06\\
3.54	6.24396873536779e-06\\
3.55	5.7401332710158e-06\\
3.56	5.13100267781777e-06\\
3.57	4.44458286932889e-06\\
3.58	3.72465957865877e-06\\
3.6	2.53396036877854e-06\\
3.61	2.37169461500307e-06\\
3.62	2.64209304862064e-06\\
3.64	3.92737517150939e-06\\
3.65	4.65114382705438e-06\\
3.66	5.3211405165847e-06\\
3.67	5.89262383247997e-06\\
3.68	6.33413511925156e-06\\
3.69	6.62292817768467e-06\\
3.7	6.74324175440847e-06\\
3.71	6.68578046955381e-06\\
3.72	6.44774247141625e-06\\
3.73	6.03316135922563e-06\\
3.74	5.45359794061058e-06\\
3.75	4.72959240548777e-06\\
3.76	3.89429505571407e-06\\
3.77	3.00427384723751e-06\\
3.78	2.17635967593034e-06\\
3.79	1.69252198971672e-06\\
3.8	1.91359770942743e-06\\
3.81	2.6469499629447e-06\\
3.82	3.52002327829752e-06\\
3.83	4.35491166628947e-06\\
3.84	5.06369964338866e-06\\
3.85	5.58942059867387e-06\\
3.86	5.89007668870231e-06\\
3.87	5.93504115812015e-06\\
3.88	5.70527038342104e-06\\
3.89	5.19501948935435e-06\\
3.9	4.41457036058341e-06\\
3.91	3.39482028646691e-06\\
3.92	2.2002899904265e-06\\
3.93	1.02629760765044e-06\\
3.94	1.0720839493283e-06\\
3.95	2.2334502711806e-06\\
3.96	3.31005382715102e-06\\
3.97	4.00454323168249e-06\\
3.98	4.07961240340958e-06\\
3.99	3.27268710943681e-06\\
4	1.28405876860291e-06\\
4.01	2.22815824785799e-06\\
4.02	7.6486268599745e-06\\
4.03	1.54096640472006e-05\\
4.04	2.59911277740526e-05\\
4.05	3.992249619024e-05\\
4.06	5.77840575178314e-05\\
4.08	0.000107878512460451\\
4.1	0.000181963120534355\\
4.12	0.000286573103352702\\
4.14	0.000429080890408415\\
4.16	0.000617667840258369\\
4.18	0.000861277092132432\\
4.2	0.00116954804079704\\
4.23	0.00177576652646031\\
4.26	0.00258731574872259\\
4.29	0.00364255438297507\\
4.32	0.00498053759393085\\
4.35	0.0066399209666022\\
4.38	0.00865782192797587\\
4.41	0.0110686835358664\\
4.44	0.0139031842086823\\
4.47	0.0171872337655733\\
4.5	0.0209410913592647\\
4.53	0.0251786348677037\\
4.56	0.0299068044320978\\
4.59	0.0351252354495508\\
4.62	0.0408260887915472\\
4.65	0.0469940786410681\\
4.68	0.0536066914041849\\
4.71	0.060634582902773\\
4.74	0.0680421356995982\\
4.77	0.0757881541079832\\
4.8	0.083826671310454\\
4.83	0.0921078411170457\\
4.87	0.103435340151131\\
4.91	0.114970117382161\\
4.95	0.126580052856359\\
4.99	0.138134915014809\\
5	0.141000244167107\\
};
\addlegendentry{$\psi$ approximation}

\addplot [color=mycolor2, dashed, line width=2.0pt]
  table[row sep=crcr]{%
0.00999999999999979	0.00575713821192212\\
0.0199999999999996	0.00573489642773741\\
0.0300000000000002	0.00566848585276517\\
0.04	0.00555885779460615\\
0.0499999999999998	0.00540762235468999\\
0.0599999999999996	0.00521709029101316\\
0.0700000000000003	0.00499034260861387\\
0.0800000000000001	0.00473134149748845\\
0.0899999999999999	0.0044451043070266\\
0.0999999999999996	0.00413797357940114\\
0.12	0.00349569965139576\\
0.14	0.00290235111301442\\
0.15	0.00267108490532383\\
0.16	0.00251553615057332\\
0.17	0.00245745763876321\\
0.18	0.0025070206899986\\
0.19	0.00265771406336496\\
0.2	0.00288994070560423\\
0.23	0.0038446935937602\\
0.25	0.00452198499006037\\
0.26	0.004838810228547\\
0.27	0.00513147027330744\\
0.28	0.00539462597787791\\
0.29	0.0056240933972018\\
0.3	0.00581668516626017\\
0.31	0.00597012119729425\\
0.32	0.00608298628018974\\
0.33	0.00615472063102116\\
0.34	0.0061856351862858\\
0.35	0.00617694724303192\\
0.36	0.00613083441449829\\
0.37	0.00605050595486251\\
0.38	0.00594029000963505\\
0.39	0.00580573240827036\\
0.41	0.00549243364020061\\
0.43	0.0051820867670139\\
0.44	0.00505567761891433\\
0.45	0.00496427226122811\\
0.46	0.00491879515421688\\
0.47	0.00492783485825749\\
0.48	0.00499637875556675\\
0.49	0.00512506957701384\\
0.5	0.00531026361715975\\
0.52	0.00581935986128235\\
0.56	0.00710835270362752\\
0.58	0.00773523570947498\\
0.59	0.0080175905051121\\
0.6	0.00827163491623678\\
0.61	0.00849281033591034\\
0.62	0.00867741039363186\\
0.63	0.00882254390470808\\
0.64	0.0089261040365972\\
0.65	0.00898674281268376\\
0.66	0.00900384986499685\\
0.67	0.00897753454664827\\
0.68	0.00890861089448836\\
0.69	0.0087985853728127\\
0.7	0.00864964775981586\\
0.71	0.00846466590223555\\
0.72	0.00824718528459655\\
0.74	0.00773233566001101\\
0.76	0.00714735692642057\\
0.81	0.0057578585139718\\
0.82	0.00556118921054532\\
0.83	0.00541125679151092\\
0.84	0.00531395420848589\\
0.85	0.00527251035064817\\
0.86	0.00528699223328343\\
0.87	0.0053542842087611\\
0.88	0.00546856833182958\\
0.89	0.0056221461210323\\
0.91	0.00601234898585711\\
0.94	0.00668044311618164\\
0.96	0.00710153528956889\\
0.98	0.00745745211968004\\
0.99	0.00760240938151376\\
1	0.00772200021056804\\
1.01	0.00781433970583959\\
1.02	0.00787799854871332\\
1.03	0.0079119708593077\\
1.04	0.0079156535851612\\
1.05	0.00788883669994493\\
1.06	0.00783170364994826\\
1.07	0.00774484172696191\\
1.08	0.00762926231278154\\
1.09	0.00748643119676715\\
1.1	0.00731830936194798\\
1.12	0.00691683455939207\\
1.14	0.00645265974791522\\
1.19	0.00531494299584533\\
1.2	0.00515293529672801\\
1.21	0.00503276612953941\\
1.22	0.00496148248598687\\
1.23	0.00494375754173156\\
1.24	0.00498112312841246\\
1.25	0.00507167374059213\\
1.26	0.00521035051469389\\
1.27	0.00538968776880806\\
1.29	0.00583405653176251\\
1.32	0.00657616835087454\\
1.34	0.00702749523095855\\
1.35	0.00722107619782493\\
1.36	0.00738679814224585\\
1.37	0.00752081824023273\\
1.38	0.00762006741471411\\
1.39	0.00768223095954459\\
1.4	0.00770573790645721\\
1.41	0.00768975885519078\\
1.42	0.00763421196652456\\
1.43	0.00753977717149667\\
1.44	0.00740791918608077\\
1.45	0.00724092049644927\\
1.46	0.00704192592634689\\
1.47	0.00681500042422328\\
1.49	0.00629865888086944\\
1.52	0.00547765394485857\\
1.54	0.00501201392012813\\
1.55	0.0048375001269804\\
1.56	0.00471592266798674\\
1.57	0.00465479610384878\\
1.58	0.00465754322968122\\
1.59	0.00472276033508279\\
1.6	0.00484445793042049\\
1.61	0.00501317131568088\\
1.63	0.00544535352602668\\
1.65	0.00592668573023845\\
1.67	0.00637831827717809\\
1.68	0.00657394619936694\\
1.69	0.00674175463629937\\
1.7	0.00687735658066783\\
1.71	0.00697736837417692\\
1.72	0.00703934147663276\\
1.73	0.00706171100117668\\
1.74	0.00704375988972841\\
1.75	0.0069855973808502\\
1.76	0.00688815093664949\\
1.77	0.00675317162736956\\
1.78	0.00658325388218391\\
1.79	0.00638187128814508\\
1.8	0.00615343041449933\\
1.82	0.00563811974502627\\
1.87	0.00439739811841258\\
1.88	0.00424260907710988\\
1.89	0.00414398847538879\\
1.9	0.00410766843866745\\
1.91	0.00413491221205782\\
1.92	0.00422175660108301\\
1.93	0.00435983068510501\\
1.95	0.00474383775690979\\
1.98	0.00541275022472118\\
1.99	0.00562053742553924\\
2	0.00580808454277801\\
2.01	0.00596975152699844\\
2.02	0.00610103386404726\\
2.03	0.00619845543323068\\
2.04	0.00625948158678082\\
2.05	0.00628245248877919\\
2.06	0.00626653460482126\\
2.07	0.00621168807326229\\
2.08	0.00611864843880771\\
2.09	0.0059889223831957\\
2.1	0.00582479842699469\\
2.11	0.00562937498572387\\
2.12	0.00540660947608857\\
2.13	0.00516139295984509\\
2.15	0.00462849098842703\\
2.19	0.00363896727498608\\
2.2	0.00347321578411299\\
2.21	0.00336364774672604\\
2.22	0.00331783968351544\\
2.23	0.00333784023834463\\
2.24	0.00341948844188845\\
2.25	0.00355344729573071\\
2.27	0.00392802492675251\\
2.29	0.00436121003390844\\
2.31	0.00477195945276107\\
2.32	0.0049506092223403\\
2.33	0.00510444415839549\\
2.34	0.00522963730131359\\
2.35	0.00532333068086955\\
2.36	0.00538353188918794\\
2.37	0.00540903685849058\\
2.38	0.00539937341146069\\
2.39	0.00535476123396659\\
2.4	0.0052760853685525\\
2.41	0.00516488180914424\\
2.42	0.00502333510514758\\
2.43	0.00485428924153954\\
2.44	0.004661274173217\\
2.45	0.00444855097753736\\
2.47	0.00398509857154839\\
2.51	0.00310785345483171\\
2.52	0.00295174098928715\\
2.53	0.00283996988614432\\
2.54	0.00277911743631499\\
2.55	0.00277159216214456\\
2.56	0.002814856629528\\
2.57	0.00290197116996294\\
2.58	0.00302317624423855\\
2.62	0.00364357765579198\\
2.63	0.00379032238835728\\
2.64	0.0039215747031072\\
2.65	0.00403339441904516\\
2.66	0.00412282588021613\\
2.67	0.00418777546509946\\
2.68	0.00422691583859002\\
2.69	0.00423961356299341\\
2.7	0.00422587581611323\\
2.71	0.00418631254607013\\
2.72	0.00412211138865303\\
2.73	0.00403502362764432\\
2.74	0.00392736015925304\\
2.75	0.00380199661553355\\
2.77	0.00351257652965524\\
2.8	0.00305146080392534\\
2.82	0.0027927067174854\\
2.83	0.00269633595137402\\
2.84	0.00262885478908879\\
2.85	0.00259339173247887\\
2.86	0.00259071878464128\\
2.87	0.00261906451268607\\
2.88	0.00267441969419456\\
2.89	0.00275121620184633\\
2.94	0.003241565532107\\
2.95	0.00332440646126347\\
2.96	0.00339345613602994\\
2.97	0.00344647621039909\\
2.98	0.00348189940367445\\
2.99	0.00349878678004718\\
3	0.0034967940775\\
3.01	0.00347614744404551\\
3.02	0.00343762790596578\\
3.03	0.00338256335950473\\
3.04	0.0033128263116606\\
3.06	0.00313954996772422\\
3.11	0.00268135900327021\\
3.12	0.00262082471862507\\
3.13	0.00258037316472111\\
3.14	0.0025624551685163\\
3.15	0.00256795933811504\\
3.16	0.00259606518079095\\
3.17	0.00264435987511043\\
3.18	0.00270917845155718\\
3.23	0.00311969760765302\\
3.24	0.00318838184487078\\
3.25	0.0032443458421225\\
3.26	0.00328513909015257\\
3.27	0.00330887378246126\\
3.28	0.00331422219242491\\
3.29	0.00330041743723147\\
3.3	0.00326726105484577\\
3.31	0.00321513987967364\\
3.32	0.00314505440652957\\
3.33	0.00305866072121672\\
3.34	0.00295832754346815\\
3.36	0.00272932020058386\\
3.39	0.0023893394973818\\
3.4	0.00230290198123474\\
3.41	0.00224197492183018\\
3.42	0.00221272456587671\\
3.43	0.00221907295261902\\
3.44	0.00226182248498924\\
3.45	0.00233847804912959\\
3.46	0.00244384268249699\\
3.51	0.00315255387173136\\
3.52	0.00328386576360034\\
3.53	0.00339912442785977\\
3.54	0.00349432311366033\\
3.55	0.00356607555956058\\
3.56	0.00361158285541144\\
3.57	0.00362861501587234\\
3.58	0.00361550758501264\\
3.59	0.00357117357702254\\
3.6	0.0034951324534768\\
3.61	0.00338756053521714\\
3.62	0.00324937175867097\\
3.63	0.00308234527170126\\
3.64	0.0028893293950929\\
3.65	0.00267457389306452\\
3.66	0.00244427921571834\\
3.68	0.00197756520361462\\
3.69	0.00177391352343739\\
3.7	0.00162313442405328\\
3.71	0.00155527477163669\\
3.72	0.00159144667514263\\
3.73	0.00173048471701566\\
3.75	0.00222512001419858\\
3.77	0.00284174754447421\\
3.78	0.00315023334971372\\
3.79	0.00344179518941916\\
3.8	0.00370606565641685\\
3.81	0.00393372619893951\\
3.82	0.00411611716061792\\
3.83	0.00424501987341856\\
3.84	0.00431254599697318\\
3.85	0.00431109899390631\\
3.86	0.00423339432693841\\
3.87	0.00407254620934327\\
3.88	0.00382226300620245\\
3.89	0.00347727542563313\\
3.9	0.00303437075760547\\
3.91	0.00249533084656026\\
3.92	0.00187741782310813\\
3.93	0.0012636798300355\\
3.94	0.00103531127326368\\
3.97	0.00405662242704603\\
3.98	0.00556000186584134\\
3.99	0.00723327781788164\\
4	0.00907357373138648\\
4.02	0.0132569872126413\\
4.04	0.018119464233907\\
4.06	0.0236699840433932\\
4.08	0.0299122508564023\\
4.1	0.0368434262670977\\
4.12	0.0444538594566439\\
4.14	0.0527272091226128\\
4.16	0.0616407744165661\\
4.18	0.0711659615418483\\
4.21	0.0865250221457224\\
4.24	0.103049209963669\\
4.27	0.120586405179202\\
4.3	0.13897141237828\\
4.33	0.158031411795932\\
4.36	0.177591207421509\\
4.39	0.197478115911193\\
4.42	0.217526380803278\\
4.45	0.237581029269798\\
4.48	0.257501112614988\\
4.51	0.277162287916392\\
4.54	0.296458709302703\\
4.57	0.315304207181796\\
4.6	0.333632746278442\\
4.64	0.357189622667557\\
4.68	0.37968990196923\\
4.72	0.401128730076146\\
4.76	0.421544728589287\\
4.81	0.44573718639971\\
4.86	0.468628528422412\\
4.92	0.494673477901045\\
4.98	0.519503404387598\\
5	0.527568458625063\\
};
\addlegendentry{$\phi$ approximation}

\addplot [color=black, line width=3.0pt, forget plot]
  table[row sep=crcr]{%
4	1e-10\\
4	1\\
};
\end{axis}

\begin{axis}[%
width=1.269\fheight,
height=0.952\fheight,
at={(-0.166\fheight,-0.131\fheight)},
scale only axis,
xmin=0,
xmax=1,
ymin=0,
ymax=1,
axis line style={draw=none},
ticks=none,
axis x line*=bottom,
axis y line*=left
]
\end{axis}
\end{tikzpicture}%

%% file: figures/kgz/finest/energy.tex
%
%
\begin{tikzpicture}

\begin{axis}[%
width=0.983\fheight,
height=0.65\fheight,
at={(0\fheight,0\fheight)},
scale only axis,
xmin=0,
xmax=5,
xlabel style={font=\color{white!15!black}},
xlabel={Time $t$},
ymode=log,
ymin=1e-08,
ymax=100,
yminorticks=true,
ylabel style={font=\color{white!15!black}},
ylabel={FOM energy error},
axis background/.style={fill=white},
xmajorgrids,
ymajorgrids,
yminorgrids,
legend style={at={(0.1,1.05)}, anchor=south west, legend cell align=left, align=left, draw=white!15!black}
]
\addplot [color=orange!90!black, dashed, line width=2.0pt]
  table[row sep=crcr]{%
0.0199999999999996	2.05947238373483e-05\\
0.0300000000000002	8.23001988874239e-05\\
0.04	0.000184881093605327\\
0.0499999999999998	0.000327947700384357\\
0.0599999999999996	0.000510959652438032\\
0.0700000000000003	0.000733231027456895\\
0.0800000000000001	0.000993936740910612\\
0.0899999999999999	0.00129212010781884\\
0.0999999999999996	0.00162670143098239\\
0.12	0.00240018151833737\\
0.14	0.00330365432187591\\
0.16	0.00432508987202709\\
0.18	0.00545150281515362\\
0.2	0.00666929047438315\\
0.22	0.00796453792148895\\
0.24	0.00932327628437178\\
0.26	0.0107316872265665\\
0.28	0.0121762537455424\\
0.3	0.0136438641733685\\
0.32	0.0151218815535731\\
0.34	0.0165981935938453\\
0.36	0.0180612586477264\\
0.38	0.0195001605761081\\
0.4	0.0209046802809234\\
0.42	0.0222653850414827\\
0.44	0.023573729711984\\
0.46	0.024822157694548\\
0.48	0.0260041856342468\\
0.51	0.0276413151186398\\
0.54	0.0291038741960665\\
0.57	0.03038363743231\\
0.6	0.0314764571902561\\
0.63	0.0323808619918827\\
0.66	0.0330967192751222\\
0.69	0.0336243606383323\\
0.72	0.0339644255446592\\
0.75	0.0341184335383423\\
0.78	0.0340898357273591\\
0.81	0.033885123089699\\
0.84	0.0335145500034287\\
0.87	0.0329921742011584\\
0.9	0.0323351631614741\\
0.94	0.0312827543824824\\
0.98	0.0300710167011539\\
1.02	0.0287435199899892\\
1.06	0.0273399417566861\\
1.11	0.0255346529091412\\
1.18	0.0230384483515354\\
1.28	0.0198668607957052\\
1.33	0.0185705465169894\\
1.37	0.0177011404062569\\
1.41	0.0169872767724064\\
1.44	0.0165536334039553\\
1.47	0.0162044094196472\\
1.5	0.0159351323986175\\
1.53	0.0157398484444275\\
1.56	0.0156116280313563\\
1.6	0.0155325261023154\\
1.64	0.0155441051685966\\
1.68	0.0156342508613898\\
1.72	0.0157952288152137\\
1.77	0.0160893579533876\\
1.82	0.0164797581987295\\
1.88	0.0170586725551539\\
1.96	0.0179722551265923\\
2.06	0.0192911911351666\\
2.14	0.0205169275405524\\
2.25	0.0224747074933475\\
2.35	0.0243819008346736\\
2.43	0.0258843942481826\\
2.52	0.0275295583670732\\
2.6	0.0289370314327289\\
2.67	0.0300728206170379\\
2.73	0.0309366037511063\\
2.8	0.0318089109256471\\
2.89	0.0327707621737636\\
2.99	0.0337043449434395\\
3.07	0.0343231950065455\\
3.15	0.0347706101908943\\
3.23	0.0350333588555486\\
3.32	0.0351432842846952\\
3.42	0.0350826333885243\\
3.53	0.0348355197636102\\
3.71	0.0343830296038277\\
3.97	0.034178554874602\\
4	0.0344324899817502\\
4.04	0.0350050204084527\\
4.07	0.0354251602177964\\
4.09	0.0355053178440039\\
4.1	0.0354191456178069\\
4.11	0.0352087862542475\\
4.12	0.0348377746252799\\
4.13	0.0342640533070289\\
4.14	0.0334396540566604\\
4.15	0.0323104251379164\\
4.16	0.0308158135146868\\
4.17	0.0288887108026484\\
4.18	0.0264553715770535\\
4.19	0.0234354121786646\\
4.2	0.019741897536428\\
4.21	0.0152815227386145\\
4.22	0.00995489513667507\\
4.23	0.00365692166940789\\
4.24	0.00372269514405616\\
4.25	0.0122988504579371\\
4.26	0.0221903189899649\\
4.27	0.0335189265034887\\
4.28	0.0464086134714944\\
4.29	0.0609843951947596\\
4.31	0.0956927636411002\\
4.33	0.138620256621034\\
4.35	0.190679037716805\\
4.37	0.252670843137444\\
4.39	0.325238875175715\\
4.41	0.408820083129063\\
4.43	0.503600142894115\\
4.45	0.609473514947513\\
4.47	0.726010903007114\\
4.49	0.852436248636451\\
4.51	0.98761508467213\\
4.53	1.13005564487333\\
4.55	1.27792360835976\\
4.57	1.42907077137933\\
4.59	1.58107731672673\\
4.61	1.73130672657759\\
4.63	1.87697179207988\\
4.65	2.01520964551706\\
4.67	2.1431633070401\\
4.69	2.25806692079576\\
4.71	2.35733167031515\\
4.73	2.43862931754495\\
4.75	2.49997040262425\\
4.77	2.53977436299303\\
4.79	2.55692916407617\\
4.81	2.55083845738291\\
4.83	2.52145476939367\\
4.85	2.46929774833619\\
4.87	2.39545702845481\\
4.89	2.3015797873867\\
4.91	2.18984355018669\\
4.93	2.06291521656343\\
4.95	1.92389764450577\\
4.97	1.77626540760914\\
4.99	1.62379155398808\\
5	1.54698006961588\\
};
\addlegendentry{sp Lift \& Learn ROM $7r=140$}

\addplot [color=blue!70!black, dashdotted, line width=2.0pt]
  table[row sep=crcr]{%
0.0199999999999996	5.11169905621501e-07\\
0.0300000000000002	2.0367763181639e-06\\
0.04	4.55338080428191e-06\\
0.0499999999999998	8.02281434971519e-06\\
0.0599999999999996	1.239348468971e-05\\
0.0700000000000003	1.76021381685132e-05\\
0.0800000000000001	2.35760102896165e-05\\
0.0899999999999999	3.02352851076648e-05\\
0.0999999999999996	3.74957728354275e-05\\
0.11	4.52717062989905e-05\\
0.12	5.34785539048243e-05\\
0.14	7.08692168836932e-05\\
0.16	8.91144769366291e-05\\
0.18	0.000107828506370424\\
0.2	0.000126832298369663\\
0.22	0.000146161802745155\\
0.24	0.000166044643906389\\
0.27	0.000197735992028356\\
0.31	0.000245776718193156\\
0.36	0.000318715021808203\\
0.4	0.000387671967146162\\
0.43	0.000444217033123095\\
0.46	0.000503426460254559\\
0.49	0.000564112424899576\\
0.52	0.000625535162484993\\
0.55	0.000687474792125612\\
0.58	0.000750074822990427\\
0.62	0.000834889004713658\\
0.66	0.000920884503871094\\
0.69	0.000985136507415518\\
0.72	0.00104785497507692\\
0.75	0.00110767545927047\\
0.78	0.00116335828176034\\
0.81	0.00121398152185975\\
0.84	0.00125900369072497\\
0.87	0.00129820617662517\\
0.9	0.00133157900980677\\
0.93	0.00135923001274932\\
0.96	0.00138137107222064\\
1	0.00140301228278986\\
1.04	0.00141718352364478\\
1.1	0.00142989315660998\\
1.28	0.00145680033662302\\
1.34	0.00145705322564936\\
1.49	0.00145213072621118\\
1.56	0.00146094606040606\\
1.79	0.00149842793858625\\
1.83	0.00149570676090661\\
1.87	0.0014867333943721\\
1.92	0.00146796005574288\\
2	0.00143675298163089\\
2.04	0.00143052710661322\\
2.07	0.0014324730051294\\
2.11	0.00144337512317179\\
2.24	0.00149288040193369\\
2.28	0.0014963077487073\\
2.33	0.00149166817662262\\
2.43	0.00147932219896574\\
2.48	0.00148206032280996\\
2.54	0.00149329362095614\\
2.67	0.00152824643814256\\
2.9	0.00159161887519531\\
3.15	0.00164819507753691\\
3.28	0.00168998202948842\\
3.33	0.00169650943310444\\
3.38	0.00169491553096577\\
3.46	0.00168211376131239\\
3.53	0.00167360387952158\\
3.58	0.00167537932086816\\
3.63	0.00168538457831801\\
3.7	0.00170944704720569\\
3.75	0.00172493405683781\\
3.78	0.0017280634974486\\
3.81	0.00172266276678783\\
3.83	0.00171270876975541\\
3.85	0.00169665919105228\\
3.87	0.00167382089271541\\
3.89	0.0016437055307506\\
3.91	0.00160604142734428\\
3.93	0.00156074124386123\\
3.95	0.00150781520937926\\
3.97	0.00144722419062873\\
3.99	0.00137867310732872\\
4.01	0.00130135259700182\\
4.03	0.00121364468739102\\
4.04	0.00116509295912238\\
4.05	0.00111281566496473\\
4.06	0.00105624896707013\\
4.07	0.000994725400996685\\
4.08	0.000927468040081293\\
4.09	0.000853586231801273\\
4.1	0.000772073148923482\\
4.11	0.0006818053743973\\
4.12	0.000581544708520596\\
4.13	0.000469942350464408\\
4.14	0.000345545563613996\\
4.15	0.000206806887513267\\
4.16	5.20959083314665e-05\\
4.17	0.000120286453725384\\
4.18	0.000312091236464767\\
4.19	0.000525102458650508\\
4.2	0.000761116078906525\\
4.21	0.00102191714757282\\
4.22	0.00130925545538503\\
4.24	0.00197021283379399\\
4.26	0.00275629604357619\\
4.28	0.00367757292144233\\
4.3	0.00474106035397653\\
4.32	0.00595000309005172\\
4.34	0.00730329497276216\\
4.36	0.00879508713682298\\
4.38	0.0104146158815456\\
4.4	0.0121462683657455\\
4.42	0.0139698882273888\\
4.44	0.0158613070740694\\
4.46	0.0177930728648925\\
4.48	0.01973533370734\\
4.5	0.0216568264592843\\
4.52	0.0235259143572421\\
4.54	0.0253116169143323\\
4.56	0.0269845784010659\\
4.58	0.0285179278600117\\
4.6	0.0298879930731027\\
4.62	0.0310748422793404\\
4.64	0.0320626397393016\\
4.66	0.0328398134818826\\
4.68	0.0333990448736358\\
4.7	0.0337370993063677\\
4.72	0.033854524796493\\
4.74	0.0337552503412007\\
4.76	0.0334461184069914\\
4.78	0.032936386051083\\
4.8	0.0322372271622362\\
4.82	0.0313612645281774\\
4.84	0.0303221553257072\\
4.86	0.0291342476390855\\
4.88	0.0278123191765144\\
4.9	0.0263714028657387\\
4.92	0.0248266978003812\\
4.94	0.0231935583491577\\
4.96	0.021487549335975\\
4.98	0.019724551203285\\
5	0.0179208960915048\\
};
\addlegendentry{intrusive lifting ROM $7r=140$}

\addplot [color=black, line width=3.0pt]
  table[row sep=crcr]{%
4	5.11169905621501e-010\\
4	100\\
};

\end{axis}
\end{tikzpicture}%